\pdfoutput=1
\documentclass[12pt]{report}
\usepackage[table,dvipsnames]{xcolor}
\usepackage{setspace}
\usepackage[a-1b]{pdfx} % Need this to create a PDF/A file

\usepackage{pdfpages}

\newdimen \jot \jot=5mm
\usepackage[T1]{fontenc}
\usepackage{lmodern} % load a font with all the characters
\usepackage{tocbibind} % need this to contents adding for TOC
\usepackage[hypcap=true]{caption}
\hypersetup{
  linktocpage,
  colorlinks,
  linkcolor={blue!50!black},
  citecolor={blue!50!black},
  urlcolor={blue!80!black}
}

\usepackage[
backend=bibtex,
style=chicago-authordate,
natbib=true,
sorting=none,
maxbibnames=99,
]{biblatex}
\let\cite\parencite
\bibliography{ms.bib}
\AtEveryBibitem{\clearfield{note}
\clearfield{month}}
\usepackage{enumerate}
\usepackage{amssymb,amsthm,amsmath,mathtools,graphicx,url,xspace,booktabs,microtype,makeidx,enumitem,thmtools,latexsym,tabularx,bm,changepage,placeins,array,dcolumn}
\usepackage[table,dvipsnames]{xcolor}
\usepackage[normalem]{ulem}

\usepackage{subcaption,caption,multirow}

\usepackage{xpatch}
\usepackage{algorithm}
\usepackage{algpseudocode}

\usepackage{multirow}
\usepackage{pifont}

\makeglossary % enable the glossary
\newcommand{\indexed}[1]{\textbf{#1}\index{#1}}

\newcommand{\xline}{\noindent\underline{\makebox[0.1cm][l]{}}}

\newcommand{\mb}[1]{\textbf{#1}}
\newcommand{\ma}[1]{#1^\dagger}
\newcommand{\mi}[1]{\textbf{#1}}
\newcommand{\minimize}[3]{
\begin{aligned}
& \underset{#1}{\textrm{minimize:}}
& & #2 \\
& \textrm{subject to:}
& &  #3
\end{aligned}
}
\let\oldmc\multicolumn
\makeatletter
\newcolumntype{B}[3]{>{\boldmath\DC@{#1}{#2}{#3}}c<{\DC@end}}
\newcommand{\mcinherit}{\renewcommand{\multicolumn}[3]{\oldmc{##1}{##2}{\ifodd\rownum \@oddrowcolor\else\@evenrowcolor\fi ##3}}}
\makeatother
\renewcommand{\m}[1]{\multicolumn{1}{c}{#1}}
\newcommand{\mm}[1]{\multicolumn{1}{c|}{#1}}
\newcommand{\mmm}[1]{\multicolumn{1}{c||}{#1}}
\newcommand{\my}[1]{\multicolumn{1}{B{.}{.}{-1}}{#1}}
\newcommand{\myy}[1]{\multicolumn{1}{B{.}{.}{-1}|}{#1}}

\newcolumntype{H}{>{\setbox0=\hbox\bgroup}c<{\egroup}@{}}
\newcolumntype{d}[1]{D{.}{.}{#1}}
\newcolumntype{S}{l}
\newcolumntype{N}{>{\columncolor{green}}c}
\newcommand{\R}{\mathbb{R}}
\newcommand{\remove}[1]{}

\makeatletter
\newcommand*{\@rowstyle}{}
\newcommand*{\rowstyle}[1]{% sets the style of the next row
  \gdef\@rowstyle{#1}%
  \@rowstyle\ignorespaces%
}
\newcolumntype{=}{% resets the row style
  >{\gdef\@rowstyle{}}%
}
\newcolumntype{+}{% adds the current row style to the next column
  >{\@rowstyle}%
}
\makeatother

\definecolor{lightgray}{gray}{0.96}
\definecolor{darkgray}{gray}{0.7}
\definecolor{darkergray}{gray}{0.0}

\newcommand{\newcite}[1]{\cite{#1}}    % {\citealt{#1}}
\newcommand{\mycite}[1]{\cite{#1}}     % {\citep{#1}}
\newcommand{\mynewcite}[1]{\cite{#1}}  % {\citet{#1}}
\newcommand{\myciteauthor}[1]{\cite{#1}}%{\citeauthor{#1}}

\newcommand{\union}{\cup}

\newcommand{\figref}[1]{Figure~\ref{#1}}
\newcommand{\tabref}[1]{Table~\ref{#1}}
\newcommand{\thref}[1]{Theorem~\ref{#1}}
\newcommand{\lemref}[1]{Lemma~\ref{#1}}
\newcommand{\secref}[1]{\S~\ref{#1}}

\newcommand{\charef}[1]{Chapter~\ref{#1}}
\newcommand{\algoref}[1]{Algorithm~\ref{#1}}
\newcommand{\eqnref}[1]{Eq.~\ref{#1}}
\newcommand{\appref}[1]{(Appendix~\ref{#1})}
\newcommand{\chapref}[1]{Chapter~\ref{#1}}
\newcommand{\ie}{i.e.,\xspace}
\newcommand{\bigie}{I.e.,\xspace}

\DeclareMathOperator{\DF}{DF}
\DeclareMathOperator{\IDF}{IDF}

\DeclareMathOperator{\NN}{NN}

\DeclareMathOperator{\score}{{score}}

\newcommand{\I}{\mathbf{I}}

\newcommand{\XmD}{\cX \setminus \cD}
\newcommand{\XmQ}{\cX \setminus \cQ}
\newcommand{\Z}{\mathbf{0}}
\newcommand{\alphatil}{\tilde{\alpha}}
\newcommand{\alphat}{\tilde{\alpha}_\cQ}
\newcommand{\ath}[1]{#1^{\textrm{th}}}
\newcommand{\avgdl}{\bar{L}}

\newcommand{\bI}{\mathbb{I}}

\newcommand{\betatil}{\tilde{\beta}}
\newcommand{\betat}{\tilde{\beta}_\cQ}

\newcommand{\btau}{{\boldsymbol\tau}}

\newcommand{\bx}{\mathbf{x}}

\newcommand{\cDoc}{\mathcal{D}}
\newcommand{\cD}{\mathcal{Q}}
\newcommand{\cE}{\mathcal{E}}

\newcommand{\cL}{\mathcal{L}}
\newcommand{\cM}{\mathcal{M}}
\newcommand{\cNc}{\mathcal{N}_c}
\newcommand{\cN}{\mathcal{N}}
\newcommand{\cQ}{\mathcal{Q}}
\newcommand{\cR}{\mathcal{R}}

\newcommand{\cVp}{\mathcal{V'}}
\newcommand{\cV}{\mathcal{V}}
\newcommand{\cX}{\mathcal{X}}
\newcommand{\cfont}[1]{{\tt #1}}

\newcommand{\erLong}{Entity Set Expansion\xspace}
\newcommand{\er}{ESE\xspace}
\newcommand{\fb}[1]{\textbf{#1}}
\newcommand{\fcq}{\bar{f}_\cQ}
\newcommand{\ffont}[1]{\underline{#1}}
\newcommand{\fracil}[2]{{#1}/{#2}}
\newcommand{\gennn}{\NN^{(g)}_\theta}
\newcommand{\infnn}{\NN^{(i)}_\phi}
\newcommand{\ip}[2]{\langle  #1, #2 \rangle}
\newcommand{\lambdat}{\tilde{\lambda}}

\newcommand{\nvgeLong}{Neural Variational Set Expansion\xspace}
\newcommand{\nvge}{NVSE\xspace}

\newcommand{\rDoc}{\mathrm{D}}
\newcommand{\rD}{\mathrm{Q}}
\newcommand{\rF}{\mathrm{F}}

\newcommand{\rQ}{\mathrm{Q}}
\newcommand{\rS}{\mathrm{S}}
\newcommand{\rVp}{\mathrm{V'}}
\newcommand{\rV}{\mathrm{V}}
\newcommand{\rX}{\mathrm{X}}

\newcommand{\rhot}{\tilde{\rho}}
\newcommand{\setX}{SetExpan\xspace}

\newcommand{\talpha}{\alphat}
\newcommand{\ta}{\theta}
\newcommand{\tbeta}{\betat}

\newcommand{\wTv}{Word2Vecf\xspace}
\newcommand{\xbar}{\bar{x}}

\makeatletter
\g@addto@macro\normalsize{%
  \setlength\abovedisplayskip{5pt}
  \setlength\belowdisplayskip{5pt}
  \setlength\abovedisplayshortskip{5pt}
  \setlength\belowdisplayshortskip{5pt}
}
\makeatother

\newcommand{\paratextbf}[1]{\textbf{#1}}
\newcommand{\BStrut}{\rule[-0.9ex]{0pt}{0pt}}   % = `bottom' strut
\newcommand{\TStrut}{\rule{0pt}{2.6ex}}         % = `top' strut
\newcommand{\numberthis}{\addtocounter{equation}{1}\tag{\theequation}}

\newtheorem{theorem}{Theorem}
\newtheorem{lemma}[theorem]{Lemma}
\newtheorem{proposition}[theorem]{Proposition}

\newenvironment{definition}[1][Definition]{\begin{trivlist}
\item[\hskip \labelsep {\bfseries #1}]}{\end{trivlist}}

\PassOptionsToPackage{noend}{algpseudocode}
\algnewcommand{\LeftComment}[1]{\(\triangleright\) #1}%
\algrenewcommand\algorithmicindent{0.6em}%
\makeatletter%
\xpatchcmd{\algorithmic}{\ALG@tlm\z@}{\leftmargin\z@\ALG@tlm\z@}{}{}%
\makeatother
\newcommand{\AR}{R\xspace}
\newcommand{\CT}{T\xspace}

\newcommand{\Je}{J^{(f)}}
\newcommand{\Jr}{J^{(f)}}

\newcommand{\T}{\mathbb{T}}

\newcommand{\ab}{\mathbf{a}}

\newcommand{\bE}{\mathbb{E}}
\newcommand{\bb}{\mathbf{b}}

\newcommand{\blde}{\mathbf{e}}
\newcommand{\bldr}{\mathbf{r}}
\newcommand{\bldt}{\mathbf{t}}

\newcommand{\cFc}{\mathcal{F}^c}
\newcommand{\cF}{\mathcal{F}}

\newcommand{\cT}{\mathcal{T}}
\newcommand{\cU}{\mathcal{U}}
\newcommand{\cb}{c}
\newcommand{\comp}{c}

\newcommand{\ec}{\mathbf{\bar{e}}}
\newcommand{\entail}{\textsc{Entail}}
\newcommand{\eqbref}[1]{(\ref{#1})}
\newcommand{\e}{\mathbf{e}}
\newcommand{\fc}{\bar{f}}

\newcommand{\grad}[2]{\frac{\partial #2}{\partial #1}}
\newcommand{\hfs}[3][-]{H^{#1}(#2, #3)}
\newcommand{\ifoif}{iff\xspace}

\newcommand{\implication}{\textsc{RelImp}}

\newcommand{\norm}[1]{{|}{|}#1{|}{|}^2}
\newcommand{\proj}{\textit{Proj}}
\newcommand{\proptrans}{\textsc{ProTrans}}

\newcommand{\rE}{\mathrm{E}}

\newcommand{\rT}{\mathrm{T}}
\newcommand{\rU}{\mathrm{U}}
\newcommand{\rc}{\mathbf{\bar{r}}}

\newcommand{\rp}[1][d]{\mathbb{R}^{#1}_{+}}
\newcommand{\rrelimpl}{\textsc{RevImp}}

\newcommand{\std}{\tilde{d}}
\newcommand{\supp}[2][\mathbb{T}]{h_{#1}(#2)}
\newcommand{\symm}{\textsc{Symm}}
\newcommand{\tc}{\mathbf{\bar{t}}}

\newcommand{\tworowb}[3][c]{\begin{tabular}{@{}#1 @{}}#2 \\ #3\end{tabular}}
\newcommand{\tworow}[2]{\begin{tabular}{c}#1 \\ [.5ex] \hline #2\end{tabular}}
\newcommand{\twrs}[1]{\texttt{\justify #1}\xspace}
\newcommand{\twr}[1]{\texttt{\justify #1}\xspace}
\newcommand{\typeimpl}{\textsc{TypeImp}}
\newcommand{\udg}{$^{\dagger}$}

\newcommand{\vvv}{v}
\newcommand{\x}{\mathbf{x}}
\newcommand{\y}{\mathbf{y}}
\newcommand*\justify{%
 \fontdimen2\font=1pt% interword space
 \fontdimen3\font=1pt% interword stretch
 \fontdimen4\font=0.1em% interword shrink
 \fontdimen7\font=0.1em% extra space
 \hyphenchar\font=`\-% allowing hyphenation
}

\newcommand{\Uk}{U^{(k)}}
\newcommand{\Sk}{\Sigma^{(k)}}
\newcommand{\resumes}{r\'{e}sum\'{e}s}
\DeclareMathOperator{\trace}{trace}
\DeclareMathOperator{\KL}{KL}
\newcommand{\midd}{\mid\mid}

\newcommand{\nepoch}{n_{\text{epoch}}}
\newcommand{\cmark}{\ding{51}}%
\newcommand{\xmark}{\ding{55}}%
\newcommand{\z}{\mathbf{z}}
\newcommand{\w}{\mathbf{w}}

\newcommand{\btheta}{{\boldsymbol\theta}}
\newcommand{\bphi}{{\boldsymbol\phi}}
\newcommand{\NDCG}{{\footnotesize NDCG@}\small}
\newcommand{\MAP}{{\footnotesize MAP}}
\DeclareMathOperator*{\argmin}{arg\,min}

\begin{document}
\newcommand{\bin}[2]{\left(\begin{array}{@{}c@{}} #1 \\ #2
             \end{array}\right) }
\renewcommand{\contentsname}{Table of Contents}
\baselineskip=24pt

% Create cover page of dissertation !
\pagenumbering{roman}
\thispagestyle{empty}
\begin{center}
\vspace*{.25in}
{\bf\LARGE{ Representation Learning for Words and Entities }}\\ % change it accordingly!
\vspace*{.75in}
{\bf by} \\*[18pt]
\vspace*{.2in}
{\bf Pushpendre Rastogi}\\ % change it accordingly!
\vspace*{1in}
{\bf A dissertation submitted to The Johns Hopkins University\\
in conformity with the requirements for the degree of\\
Doctor of Philosophy }\\
\vspace*{.75in}
{\bf Baltimore, Maryland} \\
{\bf March, 2019} \\     % change it accordingly!
\vspace*{.5in}
\begin{small}
{\bf \copyright{ }2019 by Pushpendre Rastogi} \\ % change the year if needed!
{\bf All rights reserved}
\end{small}
\end{center}
\newpage
\pagestyle{plain}
\pagenumbering{roman}
\setcounter{page}{2}
%auto-ignore
\chapter*{Abstract}

This thesis presents new methods for unsupervised learning of distributed representations of words and entities from text and knowledge bases. The first algorithm presented in the thesis is a multi-view algorithm for learning representations of words called Multiview Latent Semantic Analysis (MVLSA). By incorporating up to 46 different types of co-occurrence statistics for the same vocabulary of english words, I show that MVLSA outperforms other state-of-the-art word embedding models. Next, I focus on learning entity representations for search and recommendation and present the second method of this thesis, Neural Variational Set Expansion (NVSE). NVSE is also an unsupervised learning method, but it is based on the Variational Autoencoder framework.  Evaluations with human annotators show that NVSE can facilitate better search and recommendation of information gathered from noisy, automatic annotation of unstructured natural language corpora. Finally, I move from unstructured data and focus on structured knowledge graphs. I present novel approaches for learning embeddings of vertices and edges in a knowledge graph that obey logical constraints.

\textbf{Keywords:} Machine Learning, Natural Language Processing, Representation Learning, Knowledge Graphs, Entities, Word Embeddings, Entity Embeddings.

% \textbf{Advisor:} Benjamin Van Durme.

%auto-ignore
\chapter*{Thesis Committee}

\section*{}
\subsection*{Primary Readers}

\begin{singlespace}
\indent Benjamin Van Durme (Primary Advisor)\\
\indent \indent Assistant Professor \\
\indent \indent Department of Computer Science\\
\indent \indent  Johns Hopkins University \\

\smallskip

\noindent Raman Arora \\
\indent \indent Assistant Professor\\
\indent \indent Department of Computer Science \\
\indent \indent  Johns Hopkins University \\

\smallskip

\noindent Mark Dredze \\
\indent \indent Associate Professor \\
\indent \indent Department of Computer Science \\
\indent \indent Johns Hopkins University
\end{singlespace}

%auto-ignore
\chapter*{Acknowledgments}
First, I will like to thank my dissertation committee, Benjamin Van Durme, Raman Arora, Mark for all of their time and help. Without their advice, I will not have been able to finish this thesis. I am very grateful to the following people:

Benjamin Van Durme, my advisor: Ben admitted me to Hopkins and mentored me for six years. He worked tirelessly to improve my shortcomings, gave copious advice to me on writing, coding, and many other aspects of graduate life and did not give up even when I was slow on the uptake. Ben is an excellent thesis advisor, and he made me realize that I need to look at the bigger picture and stay focused on a topic to achieve results.

Raman Arora, Thank you from the bottom of my heart for your advice and support.  I thoroughly enjoyed every time I worked with you, and you taught me a lot about machine learning. My most significant accomplishments resulted from my collaborations with you, and for that, I am genuinely grateful.

Jason Eisner, your knowledge, diligence, and technical prowess are truly admirable, and I learned a lot from working with you. You worked with me for one year and taught me a lot in the process. My collaboration with Jason resulted in an excellent paper which unfortunately I could not include in this thesis because it was on a different topic.

Professors James Spall, Vince Lyzinski and Amitabha Basu from the AMS department for helping me find answers to many research problems that I got stuck on. Professor Spall taught me about stochastic optimization, and I  also collaborated with him on this area which resulted in a CISS publication. Vince introduced me to the stochastic block models, and the problem of Vertex Nomination and a significant portion of this thesis was work done in collaboration with him. Although I never got the chance to collaborate and publish with Amitabha formally, I brainstormed with him about logically constrained embeddings with him, and he was always charitable with his time incredibly insightful and helpful.

I am also grateful to Professors Kevin Duh and Aaron Steven White, with whom I collaborated on the recasting semantics paper for all their hard work and tenacity in seeing the project to the end.

I will also like to thank Ruth Scally and Carl Pupa who have been incredibly helpful to not just me, but I think the whole CLSP department. Ruth went beyond her duty and helped me sort out some CPT-related issues in the summer of 2017, and I am ever grateful to her for that.

I will also like to thank all of the student collaborators that I have had the pleasure of working with. Ellie Pavlick and Juri Ganitkevich whom I assisted on the PPDB project. My collaboration with Ellie and Juri helped me get started on my research journey and resulted in my most cited paper to date. Manaal Faruqui -- Ph.D. from CMU -- who read my paper on MVLSA and invited me to collaborate with him to explicate problems of word embedding evaluation. That paper is now the second most cited paper on my resume! Ryan Cotterell, who is now a rising faculty member at Cambridge and is an inspirational figure in my mind, for his insight and initiative that made my WFST paper with Jason possible. He mentioned the importance of tying the weights required for copying characters in previous literature. I realized that this trick could be implemented even in the neural architecture that we had and coded it and immediately our experimental results improved beyond SOTA. He is a smart researcher, and I am grateful to have worked with him. I also enjoyed my collaboration with Jingyi Zhu who worked with me on the Efficient Adaptive SPSA paper.  Frank Ferraro, and Travis Wolfe, both one year senior to my Ph.D. cohort, with whom I published a paper each. They taught me about the tricks of Bayesian inference, max-margin learning and factor graphs. Rachel Rudinger who was my colleague and also a great friend who always gave well-reasoned and rational advice. Finally, I will like to thank Adam Poliak who has been both a great friend and a great research collaborator. We have published four papers together so far, and I will gladly work with him again given another chance. Finally, I will like to mention collaborators with whom I worked, but I was not to publish anything with just because I did not work smart enough. Nicholas Andrews, Mo Yu, Dingquan Wang, Nanyun Peng, and Elan Hourticolon-Retzler, hopefully, I will be to able to make up for the missed opportunities in future.

The Johns Hopkins University is a great university because of all the great students who come here. I enjoyed not only my collaborations with fellow students, but I also cherished their friendship. Finally, I will like to thank Adi Renduchintala, Gaurav Kumar, Keith Levin, Tongfei Chen,  Anirbit Mukherjee, Skyler Kim, Tim Vieira, Keisuke Sakaguchi, Ting Hua, Corbin Rosset, Poorya Mianjyi, Inayat Ullah, Yasamin Nazari, Xuchen Yao, Patrick Xia, Seth Ebner, Ryan Culkin, Elias Stengel-Eskin and Andrew Blair-Stanek, for their friendship.
 % Add acknowledgements
\cleardoublepage
\newpage
\pagestyle{plain}
\baselineskip=24pt
\tableofcontents
% for the three lines below, change the page numbers if needed!
%\addtocontents{toc}{\contentsline{chapter}{Table of Contents}{iii}}
%\addtocontents{toc}{\protect\contentsline{chapter}{\protect\numberline{}Table of Contents}{iii}}
%\addtocontents{toc}{\protect\contentsline{chapter}{\protect\numberline{}List of Tables}{iv}}
%\addtocontents{toc}{\protect\contentsline{chapter}{\protect\numberline{}List of Figures}{v}}
\listoftables
\listoffigures

\cleardoublepage % Needed because our intro chapter doesn't really have anything
\pagenumbering{arabic}

%auto-ignore
\chapter{Introduction}
\label{cha:introduction}

People communicate with each other about entities in the real world using natural language. Due to the increasing digitization of communication and advancement of the internet, large natural language corpora are now available for computational analysis. Beneath the sizeable natural language corpora composed of words and sentences, lie pools of information about real-world entities, such as the names and affiliations of people, and details about states and nations. My goal for this thesis is to develop methods for learning distributed representation of words, and entities, that can improve Natural Language Processing (NLP) systems, and facilitate better search and presentation of information inside unstructured natural language data.

Words are a fundamental unit of natural language. Automatically learning about words, and quantifying this information, can ultimately help many NLP tasks. One of the earliest methods for learning dense representations of words was the linguistic vector space model called Latent Semantic Analysis (LSA)~\cite{landauer1997solution}. LSA has been successfully used for information retrieval, but it has a limitation because it uses only a single view of the data via a single word co-occurrence matrix. My \indexed{first contribution} in this thesis is an algorithm called \indexed{MultiView LSA} (MVLSA)~\cite{rastogi2015multiview}. MVLSA overcomes the limitation of LSA to a single view because it can use an arbitrary number of views of data.

Generally speaking MVLSA is a multi-view algorithm and multiple view of data can help learning algorithms in two principle ways. First, the access to multiple views can help a learning algorithm to extract useful features that generalize across tasks and suppress spurious correlations that are dominant in one view and missing from the other. Second, multiple views may bring together complementary sources of information which a learning algorithm can combine to better learn the similarity between entities. More specifically, MVLSA is invariant to linear transformations of the data and this can also be considered as an advantage of MVLSA in comparison to LSA.
%% MVLSA allows for the use of an arbitrary number of views for learning word representations.
I show that using a large number of views containing diverse sources of information improves the quality of the MVLSA word representations on many NLP tasks and makes them competitive with other popular word representation learning methods.

On the surface level, natural language is composed of words and sentences, but at a deeper, more conceptual, level these words and sentences convey information about real-world entities. Therefore, I claim that learning about entities can be even more useful than learning word representations in some fields of application. Consider the field of Information Retrieval for example. Information Retrieval~(IR) is concerned with the search and presentation of information inside semi-structured and unstructured sources of data. Even though \textit{Keyword based IR}, in which a user inputs a query in the form of a list of keywords, is the standard way of interacting with industrial IR systems such as \textsc{Google} and \textsc{Bing}, the field of IR is not restricted to keyword-based retrieval of documents. In the IR community, Entity Set Expansion (ESE)~\footnote{ESE is also called Entity Recommendation in literature.} is an established task of recommending entities\footnote{Entities are also called \textsc{Items} or \textsc{Elements} in the literature.} in a knowledge graph that are similar to a provided seed set of entities. Even small improvements in this task can have a tremendous impact on many fields.

For instance, imagine a physician trying to pinpoint a specific diagnosis or a security analyst investigating a terrorist network. In both scenarios, a \textit{domain expert} may try to find answers based on prior known, relevant entities -- such as a list of diagnoses with similar symptoms that a patient is experiencing or, a list of known terrorists. Instead of manually looking for connections between the known entities, \textit{searchers}  can save time by using an automatic \textit{Recommender} that can recommend relevant entities to them. My \indexed{second contribution} in this thesis is the \indexed{Neural Variational Set Expansion} (NVSE) algorithm~\cite{rastogi2018nvse} that can operate on noisy knowledge graphs constructed automatically from a natural text document and recommend relevant entities. NVSE learns a probabilistic representation of an arbitrary subset of knowledge graph entities and uses this representation for the task of Entity Set Expansion. Through extensive experiments against existing state-of-the-art methods, I show that the NVSE algorithm can outperform existing methods for Entity Set Expansion.

Although one can learn much about entities in the world by the computational analysis of words associated to them -- indeed, the NVSE method was based on this intuition -- but there are important cases where the information about entities is stored in the form of a knowledge graph. A \indexed{Knowledge Graph} (KG) is a collection of relations of inter-connected entities. Large-scale semi-manually constructed KGs such as Freebase and YAGO2 have been heavily used for NLP tasks such as Relation Extraction, Question Answering, and Entity Recognition in Informal Domains. In the final part of this thesis, I consider the task of Knowledge Base Completion and propose methods for learning representations of knowledge graph entities that are not derived from the text.  My \indexed{third contribution} is a method for learning embeddings of entities in Knowledge Bases that obey logical constraints~\cite{rastogi2017elkb}. I show that the proposed algorithm performs better than other baseline systems.

\section{Thesis Outline}
\label{sec:thesis-outline}
\begin{description}
\item [\chapref{cha:mvlsa}] presents the MVLSA algorithm for learning word embeddings from multiple sources of data and compares its performance to other state-of-the-art methods. Through experiments on a large number of word-similarity and word-analogy tasks, I show that the MVLSA embeddings are competitive with other methods for learning word embeddings.
\item [\chapref{cha:nvse}] describes the NVSE algorithm for recommending entities grounded in natural language text. This task is called Entity Set Expansion (ESE). The NVSE algorithm is based on the Variational-Autoencoder (VAE) framework for training deep-generative models. Through human evaluations conducted on the Mechanical Turk platform, we verified that the NVSE algorithm outperforms pre-existing state of the art ESE methods.
\item [\chapref{cha:kbe}] presents the logically constrained representation learning algorithm and compares it to other methods for learning representations of KB entities.
\item [\chapref{cha:more}] This chapter compares the various word-level and entity level algorithms developed in the previous chapters with each other on a few benchmark tasks.
\item [\chapref{cha:conc}] This chapter summarizes the contributions of the thesis and outlines directions for future work.
\end{description}

%% \itodo{outline}{Hint in the chapters about the experiments that you will do later. And then say later that these are in the context of knowledge bases rather than text, and there is an overlap between structured knowledge and representation learning for KBs in the context of and a portion of the thesis is devoted to a chapter per task and so a task would be , the word similarity thing could be put in, just word similarity task in its chapter, and then go back and train a VAE representation to add an extra column in the table, and then have a few extra paragraphs and then have the various tables. And that’s a task, and tenth next chapter will be preferable something that uses word embeddings and then the entities. And if the NLP word tasks are tiny broken up into unsupervised/supervised. The purpose is not to layout new tasks just use them in a variety of new tasks, then how does X relate to human sim judgments. and then do an analysis  Entity search / CMR entity linking entity embedding, KB entity embedding.}

\section{Thesis Statement}
\label{sec:thesis-statement}
In this thesis, I present new algorithms for learning representation of words and entities from multiple views of data. I show that the proposed MVLSA algorithm is a generalization of the classical LSA method to multiple views of data and that incorporating various co-occurrence matrices for learning word representations improves the quality of the learned word representations. I then present a deep generative model for learning representations of entities present in natural language text and I also present the results of an approach I developed for enforcing logical constraints on the representations learned for representing entities in a knowledge base. Finally, I compare the algorithms developed in the thesis on the benchmark tasks of Contextual Mention Retrieval and Entity Disambiguation.

%auto-ignore
\chapter{Background and Motivation}
\label{cha:background}
This chapter provides the background and terminology necessary for understanding the methods used throughout this thesis. Since later chapters will refer back to these sections, one may choose to skim this chapter and refer back to it as a reference.

\section{Unsupervised Representation Learning}
\label{sec:repr-learn}
At a high-level machine learning algorithms can be divided into two classes based on the available data and the type of task that is being performed, Supervised Learning and Unsupervised Learning. Supervised machine learning methods receive a \textit{labeled} dataset containing pairs of inputs and outputs, and the goal is to construct a decision rule, that has high accuracy, based on the labeled data. On the other hand, unsupervised learning receives only a large dataset of input data, and the goal is to learn the regularities and patterns in the input data. The learned representations can be evaluated either by using the learned representations in a downstream task or by evaluating intrinsic properties of the representations such as invariances and nearest neighbors in the space of the learned representations.

Unsupervised learning can be highly beneficial, in comparison to supervised learning, because of the abscence of a single task and lack of sufficient labeled data. In such scenarios, learning the similarity between instances through unsupervised learning and leveraging that information can help to significantly reduce the sample complexity of learning. There are numerous examples of such applications. For example, \cite{nigam1998learning} used the expectation maximization algorithm over unlabeled data to learn feature weights for a Naive-Bayes classifier and showed that the number of samples required to achieve the same accuracy as a fully supervised naive bayes classifier decreased by as much as 50\%.

Besides the obvious benefit of reducing the requirement on the number of labeled samples, unsupervised learning can also help by making the learnt features more task-agnostic. This is because, in contrast to supervised learning, unsupervised learning does not fit its features to the labels provided for a single task. Therefore the representations learnt during unsupervised learning can help in scenarios such as multi-task learning~\cite{ye2016urban,he2011graphbased} and few shot learning~\cite{fu2015transductive}.

\subsection{Shallow Representation Learning}
\label{sec:shall-unsup-repr}
Principal Component Analysis (PCA) is one of the earliest statistical methods for unsupervised representation learning. Commonly PCA is known to be just an algorithm for linear dimensionality reduction shown in \algoref{alg:pca}. \begin{algorithm}
  \caption{The PCA Algorithm}\label{alg:pca}
  \begin{algorithmic}[1]
    \State \textbf{Given}: We are given $N$ data points each of which is $d$ dimensional.
    \Statex Let the $\ath{i}$ observation be called $x_i$ and, let $X$ be a $d
    \times N$ real matrix whose $\ath{i}$ column contains $x_i$.
    \Statex Finally, let $\mathbf{1}$ be a
    column vector with $N$ rows whose elements are all $1$.
    \Function{PCA}{$X, k$}
    \State Let $\bx = \frac{1}{N}{\mathbf{1}^TX}$ \Comment{ $\bx$ is simply the
      sample average, i.e. $\bx = \frac{1}{N}\sum_{i=1}^N x_i$.}
    \State Let $\Sk$, $\Uk$ equal the top $k$ singular values and
    corresponding left singular vectors of $X - \mathbf{1}\bx$.
    \State \textbf{return} $\bx, \Uk, \Sk$
    \EndFunction
  \end{algorithmic}
\end{algorithm}
However, PCA models the data with a single latent subspace, and it can be interpreted not just as a procedure for linear dimensionality reduction but also as a \textit{shallow} unsupervised representation learning algorithm, because the singular vectors $\Uk$ constitute an orthogonal basis for the optimal $k$ dimensional linear subspace that \textit{best} encodes $X-\mathbf{1}\bx$ according to the Frobenius norm of the total training error matrix. This interpretation of PCA as the solution of the optimal subspace learning problem is also referred to as the Geometric View or Synthesis View of PCA in the literature.

\paragraph{Probabilistic PCA} Although, the Geometric View of PCA shows us that the projection matrix output by PCA maps a datapoint to a subspace which is \textit{closest} to the training data, and therefore it motivates PCA as a representation learning algorithm. However, there is a more modern, probabilistic view of PCA, which frames PCA as a latent variable model and gives excellent insight into the type of representations that PCA learns. The probabilistic view of PCA was simultaneously introduced by ~\cite{tipping} and~\cite{roweis} where they showed that the output of the PCA algorithm could be used to estimate the parameters of a particular type of directed graphical model. Specifically, they considered the problem of parameter estimation for the following model.

Assume that the $\ath{i}$ observation is generated as follows, first a latent
variable $z_i \in \mathbb{R}^k$ is sampled from a normal distribution
$\mathcal{N}(\mathbf{0}, I)$. Then conditioned on the value of $z_i$, the vector
$x_i \in \mathbb{R}^D$ is drawn from $\mathcal{N}(Mz_i + \nu, \sigma^2 I)$. Here $M$ is the transformation matrix, $\nu$ is a vector and $\sigma$ is a positive scalar. In other
words \begin{equation}\label{eq:ppca}
p(x_i | z_i, M, \nu, \sigma^2) = \mathcal{N}(Mz_i+\nu, \sigma^2I).
\end{equation}

According to the above model the observed data is drawn from a continuous mixture model and since this is an unsupervised probabilistic model it makes sense to talk about the MLE estimate of $M$, $\hat{M}$. \cite{tipping} showed that $\hat{M}$ can be computed as: \[\hat{M} = \Uk(\frac{\Sk}{N} - \sigma^2I)^{1/2}R,\] where $R$ is an arbitrary orthonormal matrix, $\Uk$ is an orthogonal matrix containing the top $k$ eigen vectors, and $\Sk$ is a diagonal matrix containing the top $k$ eigen-values, of the empirical covariance matrix $1/n X^TX$. Therefore, if $\sigma = 0$ then $\Uk (\frac{\Sk}{{N}})^{1/2}$ is the maximum likelihood estimator of $M$ and $\Uk (\frac{\Sk}{{N}})^{1/2}x$ is the \textit{plugin estimate} of $z$. This interpretation also clarifies the relation between PCA and other probabilistic methods such as factor analysis~\cite{spearman1904general,thurstone1947mfa}. For example we can easily see from \eqnref{eq:ppca} that probabilistic PCA is just a more restricted form of factor analysis where the error variance along each dimension is assumed to be the same. In contrast to \eqnref{eq:ppca} factor analysis~\cite{} assumes that the errors along each dimension can have different variance but are still uncorrelated, i.e. \begin{equation}\label{eq:fa}
p(x_i | z_i, M, \nu, \sigma^2) = \mathcal{N}(Mz_i+\nu, \begin{bmatrix}\sigma^1_1 & 0 & 0 & \ldots \\ 0 & \sigma_2^2 & 0 & \ldots \\ 0 & 0 & \sigma_3^2 & \ldots \\ 0 & \ldots & 0 &\sigma_d^2 \end{bmatrix}).
\end{equation}

\subsection{Deep Representation Learning}
\label{sec:deep-unsup-repr}

Pedagogically simple examples of the utility of unsupervised learning arise in learning disentangled representations of low dimensional manifolds such as the swiss-roll dataset shown in \figref{fig:swissroll}.
\begin{figure}[htbp]
  \centering
  \includegraphics[width=.75\linewidth]{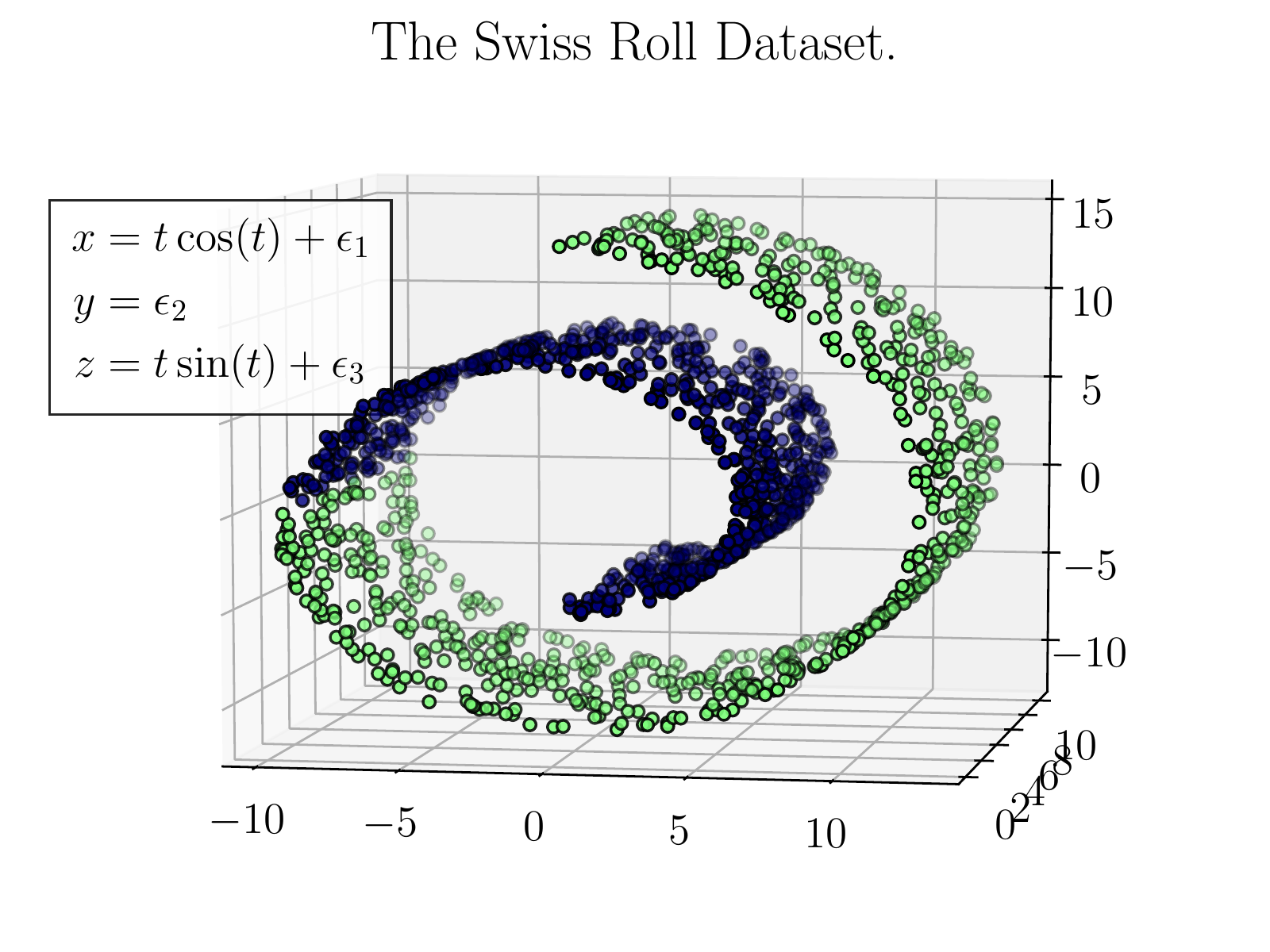}
  \caption[Pedagogical example to motivate unsupervised learning.]{The Swiss-Roll Dataset, where $3$ features $x,y,z$ are observed. $x,z$ depend on a single underlying parameter $t$, and $y$ is just noise. Learning a map from $x,y,z$ to $t$ from unsupervised data can decrease the number of samples needed to discriminate between green and blue points accurately.}
  \label{fig:swissroll}
\end{figure}
\figref{fig:swissroll} shows that learning a good representation of the observed data can be fruitful later on for supervised learning. Locally Linear Embeddings (LLE)~\cite{saul2003think}, Laplacian Eigenmaps~\cite{belkin2003laplacian} and IsoMaps~\cite{tenenbaum2000global} are a few of the best-known unsupervised machine learning algorithms that work on this principle. A seminal paper by~\cite{bengio2004out} unified these three algorithms -- amongst several others -- by recasting the problem of unsupervised learning of embeddings to learning eigenfunctions of a data-dependent kernel.

After the success of these early algorithms, the field of unsupervised representation learning grew rapidly. A significant development was the construction of algorithms such as the Restricted Boltzmann Machine~(RBM)~\cite{2006Sci...313..504H} and Deep Belief Networks~\cite{hinton2006fast} for training deep neural networks from unsupervised data. These methods learned the parameters of neural networks which can convert sparse input data to dense, distributed, vector representations. Moreover, these representations were proved to be useful for real-world tasks such as Collaborative Filtering~\cite{salakhutdinov2007restricted}.

More recently, a new method called the Variational Autoencoder (VAE)~\cite{kingma2014auto,rezende2014stochastic} for learning deep, non-linear, generative models from unlabeled data was proposed, which has resulted in tremendous advancements in unsupervised and semi-supervised machine learning~\cite{kingma2014semi,miao2016neural}. In the following section, I give an overview of the VAE framework.

\subsubsection{Variational Autoencoders}
\label{sssec:vae-bg}
\indexed{Generative modeling} of data is a broad topic in data science. A generative model of a dataset can make its underlying factors of variations more explicit, and it can help us summarize and understand large amounts of data quickly. An example of a generative model of data is the Latent Dirichlet Allocation Topic Model~\cite{blei2003latent}. A topic model can summarize a large dataset by discovering clusters of commonly co-occurring features.  Many ``component analysis'' methods such as PCA and CCA discussed in previous sections can be interpreted as statistical generative models as well~\cite{tipping,bach2005probabilistic}.

\indexed{Latent variable models} are a useful sub-type of generative model that can be useful in situations where the observed data lies in a high-dimensional space, but the elements of the data contain strong inter-dependencies. In common parlance, such data is said to lie on a low-dimensional manifold. If the inter-dependencies between the components can be de-coupled by introducing a small number of latent variables without introducing too much error, then such a latent variable model can be useful both for its predictive accuracy and also for its explanatory power. For example, suppose that we are trying to learn a generative model of human face images. The observed pixels in an image have to satisfy many constraints such as bi-lateral symmetry across the face, consistent skin coloring and relative proportion of eyes, nose, and ears. Because of all these constraints, the pixel intensity at the top-right corner of a face image may be highly correlated to the pixel intensity at the bottom-left corner in the general population of the entire dataset. However now consider a situation where we can stratify the dataset by the gender, age, race, and weight of the person. Within each stratum, the correlation between the top-left pixel intensity and the bottom-right pixel intensity will be closer to zero than the correlation in the general population. Even though these $4$ variables are not observed in the dataset, by introducing these $4$ factors as latent variables and then adding a conditional independence assumption amongst the observed variables given these unobserved variables we can make the model better suited to the data. Another motivation for latent variable models is that they are a mixture model and can approximate multi-modal distributions easily. Many excellent books on machine learning -- such as ~\cite{murphy2012ml,bishop2006prm,mackay2002iti} -- expand upon this point of view, and I will not expand upon this more.

The Variational Autoencoder is a framework for learning the parameters of a \textit{latent variable} generative model from i.i.d. unlabeled data samples. Formally, say that we are given a dataset $\mathcal{D}$ containing $n$ i.i.d. samples of a random variable $X$. We posit that there exists a latent random variable $Z$ such that instances of $X$ are conditionally independent given $Z$. In other words we posit the following generative story:
\begin{equation}\label{eq:gen-story}
Z \sim \pi(z), \qquad X | Z \sim p_\theta(x | z)
\end{equation}
Here $\theta$ parameterizes the conditional probability distribution of $X$ given $Z$. According to this model the marginal distribution of $X$ is given as
\begin{equation}\label{eq:x-marg}
p(X) = \int_z p_\theta(x|z) \pi(z) dz
\end{equation}

Maximum Likelihood Estimation (MLE) with regularization is perhaps the most common method for learning the parameters $\theta$ for a statistical model given $\mathcal{D}$. For simplicity I will omit discussion of the regularization for now and focus only on the likelihood function itself. The MLE procedure maximizes the likelihood of the parameters $\theta$ for a given dataset $\mathcal{D}$. Therefore the MLE procedure maximized the following objective:
\[ \mathcal{J}_{\text{MLE}}(\theta) = \sum_{i=1}^n \log p_\theta(x_i) = \sum_{i=1}^{n} \log \int p_\theta(x_i | z) \pi(z) dz .\]
In some situations it may be possible to compute the above objective, for example in a discrete mixture model without priors on the mixture parameters but in general it is not possible to compute the above sum over the dataset. A common solution for such problems utilizes the following identity called the \textit{Variational Identity} which introduces a new distribution over the latent variables that I denote $q(Z)$
\begin{subequations}\label{eq:var-eq}
\begin{align}
\log p(X) &= E_{Z \sim q(Z)}\left[\log \frac{p(X, Z)}{q(Z)}\right] + \KL(q(Z) \midd p(Z\mid X))\label{eq:energy}\\
\log p(X)&= E_{Z \sim q(Z)}\left[\log {p(X, Z)} \right] + \mathbb{H}[{q(Z)}] + \KL(q(Z) \midd p(Z\mid X))\label{eq:em}\\
  \log p(X)&= E_{Z \sim q(Z)}\left[\log {p(X | Z)} \right] - \KL[{q(Z)} \midd \pi(Z)] + \KL(q(Z) \midd p(Z\mid X))\label{eq:h}
\end{align}
\end{subequations}
Identity~\eqref{eq:energy} can be easily verified simply by expanding the definition of the KL divergence. In the above indentities, $q(Z)$ is actually shorthand for $q(Z \mid X, \phi)$ \ie $q(Z)$ is an arbitrary distribution over the latent variables that can freely depend on the values of $X$ and other parameters $\phi$. The variational auto-encoder specifies a special type of $q(Z \mid X, \phi)$ which is parameterized as a differentiable neural network. I will give more details about specific architectures  in \chapref{cha:nvse}.

\paragraph{Relation to Expectation Maximization} Consider identity~\eqref{eq:em} and note that if $q(Z)$ exactly equals $p(Z|X)$ then the $\KL(q(Z) \midd p(Z\mid X))$ term becomes zero. Moreover we get the formula that
\begin{align}
  p(X) = E_{Z \sim p(Z|X)}\left[\log {p(X, Z)} \right] + H[{p(Z|X)}]
\end{align}
Here the $H$ operator computes the entropy of a distribution. The EM procedure discards the entropy of $p(Z|X)$ and optimizes the Joint Likelihood with Current Parameters to get the following iterative learning rule:
\[\theta_{t+1} = \arg\max_\theta E_{Z \sim p_{\theta_t}(Z|X)}\left[\log {p_\theta(X, Z)} \right], \]
So we can see that the Expectation Maximization procedure is simply a special case of the variational optimization procedure. However, this special case of variational optimization enjoys a wonderful property of \textit{Monotonicity}, i.e. the value of the objective $E_{Z \sim p_{\theta_t}(Z|X)}\left[\log {p_\theta(X, Z)} \right]$ increases with $t$. \figref{fig:em-explanation} gives a graphical explanation of the EM procedure described above.
\begin{figure}[htbp]
  \centering
  \includegraphics[width=.75\linewidth]{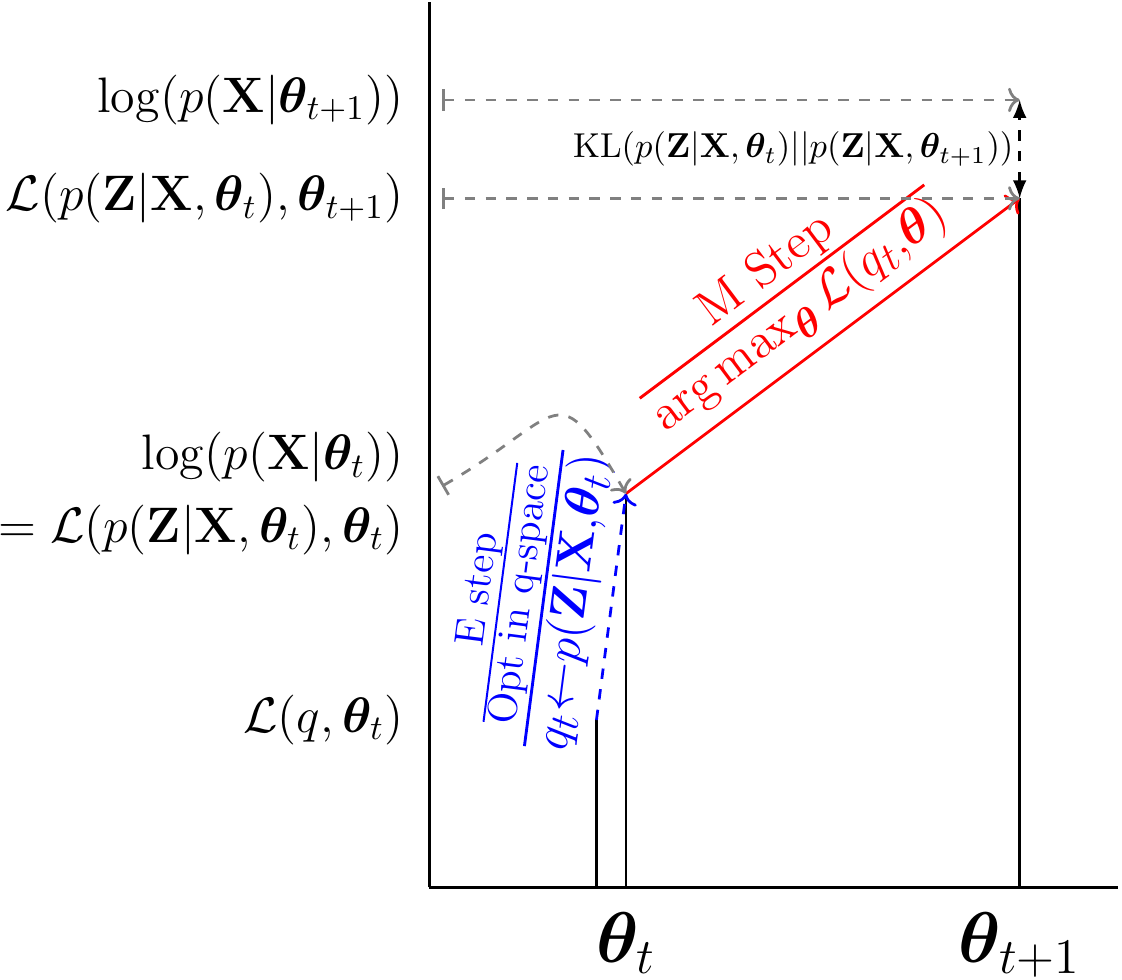}
  \caption[Graphical depiction of EM optimization]{A graphical depiction of the EM iteration as coordinate ascent in the space of probabilities and parameters. Let $\mathcal{L}(q,\theta)$ denote $\displaystyle\int_z q(Z) \log \frac{p(X,Z|\theta)}{q(Z)}$. In the so-called E-step we optimize by setting $q^t = p( Z| X, \theta_t)$ and in the M-step we set $\theta$ by maximizing $\cL(q_{t}, \theta_t)$.}
  \label{fig:em-explanation}
\end{figure}

\subsection{Multiview Representation Learning}
\label{sec:mult-repr-learn}
The goal of multi-view representation learning is to learn a single unified representation for data that is observed via multiple views/channels, but that has a single underlying source. There is no standard definition of a  ``view'' in multiview learning and it can even overlap with multimodal learning~\cite{ngiam2011multimodal} to some extent. Multiview learning can refer to any technique that learns one classifier/regressor per view and a common underlying representation, that is unknown apriori. \cite{sridharan2008information} state that the fundamental assumption underlying multi-view learning is that any of the views used for learning \textit{alone} has sufficient information about the target.

There are two natural classes of multi-view learning algorithms that utilize this assumption. The first class comprises the algorithms such as Co-Training~\cite{blum1998combining} and co-regularization~\cite{sindhwani2005co} which regularize the predictions of classifiers learnt from separate views of data to agree with each other. The second class of algorithms based on correlation analysis such as~\cite{kakade2007multi} and \cite{wang2015deep} focus on learning correlated representations of data using unsupervised learning methods and these \textit{correlation analysis} based methods will be the focus of this thesis.\footnote{We note that correlation analysis is not the sole method for unsupervised multi-view learning. A good example of such a technique is the work by \cite{ngiam2011multimodal} who presented an auto-encoder based framework for learning feature extractors that work well for feature learning from multimodal data.}

There are many natural applications for unsupervised multi-view learning especially in the field of natural language processing (NLP) because of the high dimensionality and sparsity of the bag-of-words feature representation that is typically employed for many NLP tasks. By properly utilizing multiple view of data we can learn better representations of data which can even provably reduce the number of samples required for learning. For example, consider models such as Glove~\cite{pennington2014glove} and Word2Vec~\cite{mikolov2013linguistic} that learn a dense vector representation of a word from a single view of linguistic data such as a large corpus of natural language text sequences. However these methods are not able to distinguish between antonyms such as the words \textsc{good} and \textsc{bad} because both the antonyms are used in very similar contexts. On the other hand by using an external view such as the dictionary entries for a word we can learn to distinguish between two antonyms. In this sense, the two views bring complementary information that a multi-view learning algorithms can utilize.
%% There is a word that is in one dataset,
%% Multiple views help when i bring complementary information or in suppressing bad correlations.
%% these views can bring complementary information.
%% One view can help figure out what is signal vs noise, having the second view suggests that these are somewhat not useful similarities. these views bring complementary information.
%% There are two words that co-occur in one view. (foreshadow using examples about test data)
Other successful applications include multiclass classification~\cite{arora2014multi}, clustering~\cite{chaudhuri2009multi,zhang2016deep} and ranking/retrieval~\cite{vinokourov2003inferring,cao2018deep}. Another application of multiview learning for natural language processing was presented by \cite{benton2016arora}  who also released a dataset for learning multiview representations of twitter users.\footnote{\url{https://www.cs.jhu.edu/~mdredze/datasets/multiview_embeddings}}

There has also been work done on providing guarantees for performance improvement so that multiple views can be guaranteed to not hurt the performance of a regressor or classifier. A typical assumption that is utilized in Multiview learning in order to improve the sample complexity of a learning method is that the views of data are conditionally independent given the underlying common representation. Conventional methods for machine learning do not make this assumption, and they may concatenate all the features from the different views and treat the multi-view dataset as a single-view dataset.
 For example \cite{kakade2007multi} showed that using unlabeled data from two views can reduce the sample complexity of prediction problems. Specifically they provided a semi-supervised algorithm which first uses unlabeled data to learn a kernel, and then regularized a ridge-regression classifier according to the learnt norm.

In the remaining part of this section, I describe a few classical multiview representation learning techniques that are pertinent to this thesis.

\subsubsection{Canonical Correlation Analysis}
\label{sssec:cca}
Canonical Correlation Analysis (CCA), first proposed by \cite{hotelling1935the}, is a procedure that finds linear projections of two datasets such that the correlation between the two projections is maximized. Let $X_1, X_2$ be two centered and standardized matrices denoting two views of the data with $n$ rows and $d_1, d_2$ columns respectively. Let $P_j = X_j U_j$ where $U_j \in \mathbb{R}^{d_j \times d}$ is a linear projection matrix and $j \in \{1, 2\}$. $d_1,d_2$ are the dimensionalities of $X_1, X_2$ respectively, and $d$ is the dimensionality of the latent space. Let $\hat{\Sigma}_{jj'} = \frac{1}{n} X_j^TX_{j'}$. The CCA procedure determines projection matrices $U_1, U_2$ according to the following optimization problem
\begin{align}
  \arg\max_{U_1, U_2}\ &U_1 \hat{\Sigma}_{12} U_2 \nonumber\\
  \text{ subject to }  U_1\hat{\Sigma}_{11}U_1 = 1 &\text{ and } &U_2\hat{\Sigma}_{22}U_2 = 1\label{eq:cca}
\end{align}

\cite{hardoon2004canonical} highlight two of the most important ways to motivate the CCA objective and to derive its solution and \cite{hastie1995penalized} highlight an interesting connections between CCA and Fisher's Linear Discriminant Analysis in the case of categorical classification, which I will not repeat here. Instead I list some of the important properties of the CCA procedure that are most relevant to us
\begin{itemize}
\item The CCA projections $P_j$ are invariant to shifting and scaling the data.
\item Let $A, S, B$ denote the singular value decompotion of $\hat{\Sigma}_{11}^{-1/2}\hat{\Sigma_{12}}\hat{\Sigma_{22}}^{-1/2}$. I.e. \[\hat{\Sigma}_{11}^{-1/2}\hat{\Sigma_{12}}\hat{\Sigma_{22}}^{-1/2} = ASB^T,\] then $U_1 = \hat{\Sigma}_{11}^{-1/2}A$ and $U_2 = \hat{\Sigma}_{22}^{-1/2}B$.
\end{itemize}

\paragraph{Computational Aspect: Algorithms for CCA}\label{sssec:algcca}
As described above the CCA learning problem can be reduced to solving a singular value decomposition problem of an asymmetric matrix $\hat{\Sigma}_{11}^{-1/2}\hat{\Sigma_{12}}\hat{\Sigma_{22}}^{-1/2}$. However, empirically computing this matrix may be intractable because of the quadratically increasing memory requirement as the number of examples in the dataset increases. More recently a number of works have proposed more scalable methods for computing CCA such as \cite{ge2016efficient,arora2017stochastic,gao2017stochastic,allen2017doubly}. For example, \cite{arora2017stochastic} proposed a convex relaxation of the original CCA optimization problem and they prosented stochastic approximation algorithms for optimizing the resulting objective in a streaming setting. And they showed that their proposed stochastic approximation algorithm outperformed existing state-of-the-art methods for CCA on a real dataset.

\paragraph{Nonlinear CCA}\label{sssec:deep-cca}
An interesting direction for generalizing Canonical Correlation is to use non-linear functions for projecting the views. Kernel CCA~\cite{akaho2001kernel,hardoon2004canonical} and Deep Canonical Correlation Analysis (DCCA) by~\cite{andrew2013deep} were two efforts in this direction. KCCA is a nonparametric method for learning non-linear transformations that produce high correlated projection in a reproducing kernel Hilbert space. On the other hand, DCCA trains two deep-neural networks to learn nonlinear transformations of respective data views by optimizing a regularized correlation objective.

\subsubsection{Generalized CCA}
\label{sec:generalized-cca}
Canonical Correlation Analysis is one of the earliest multiview learning algorithms; however, it is limited to only two views by construction. In order to remove this limitation, several generalizations of CCA have been proposed in the literature. \cite{kettenring1971canonical} proposed $5$ possible ways of generalizing CCA, and all of those $5$ methods possessed the special property that they reduced to standard CCA when using only two views of data. \cite{asendorf2015informative} further extended the work by \newcite{kettenring1971canonical} and proposed $20$ possible generalization of CCA. Instead of reviewing all of the possible generalizations of the CCA objective I will focus on one particular variant of Generalized CCA -- introduced by \newcite{carroll1968generalization} -- which \citeauthor{kettenring1971canonical} called the \textsc{MaxVar} generalized CCA method. Like all the other variants studied by \citeauthor{kettenring1971canonical} \textsc{MaxVar} GCCA projections are also equivalent to the standard CCA projections in the case that there are only two views of the data.

Let $X_1, \ldots, X_J$ be $J$ observed views of the same underlying data, with $J \ge 2$. The \textsc{MaxVar GCCA} procedure -- which I will call GCCA from now on -- finds $J$ projection matrices $U_j \mid j \in \{1, \ldots, J\}$, and a common latent representation $G$ such that the sum of the squared correlations between the view-projections $X_jU_j$ and the latent representation $G$ is maximized. Formally, the GCCA objective is:
\[ \arg\max_{Y, U_1, \ldots, U_J} \sum_{j=1}^J \trace(Y^T (X_j U_j)) \text{ subject to } Y^TY = I\]
This constrained maximization can be reframed as least squared error minimization optimization problem as follows
\begin{align}
   \arg\min_{Y, U_1, \ldots, U_J} \sum_{j=1}^J \trace((Y -  X_j U_j)^T(Y -  (X_j U_j))) \text{ subject to } Y^TY = I\label{eq:gcca-obj}
\end{align}
I will refer to this objective in \chapref{cha:mvlsa}.

\section{Information Retrieval and Set Completion}
\label{sec:irsc}
The field of \textit{Information Retrieval} (IR) is concerned with the search and presentation of information inside semi-structured sources. Examples of semi-structured information sources are web pages, images, research papers, and \resumes.  IR is different from querying databases because of the lack of precise semantics of data.
% This lack of semantics means that it is impossible to specify the right way of aggregating and filtering the data.
For example, IR systems, such as Internet search engines, need to answer queries like, ``What is the largest city in the world?'' without asking the user, which attribute of a city should be used to sort the cities, or what is the precise definition of a ``city''.\footnote{Clearly not all users will be satisfied with the results, in which case they will modify the query and retrieve a new set of web pages.}

Instead of asking for all sorts of clarifications IR systems work \textit{intelligently} and they find documents that are most likely to contain the answer for a query. Therefore, IR systems are best thought of as fast and efficient statistical prediction engines which incorporate:
\begin{enumerate}
\item   \textit{A document level prior about the importance of a document.} For example, the Pagerank algorithm is an unsupervised method that utilizes the hyperlinks in web-pages to learn the prior probability of the importance of a web-page~\cite{brin1998the-anatomy,yin2016ranking}.\footnote{If the document collection does not contain hyperlinks, then other meta-data such as the length of the document, author information and last modification time can be used to model the importance of a document.}
\item \textit{The probability of a document's relevance to a query.} {Search Engine Click-Logs that contain the URL that a person clicked amongst the search results in response to a query provide a useful signal for estimating this.} Finally, the document's content can be analyzed to predict whether it is relevant to a question.
\end{enumerate}

Even though \textit{Keyword based IR}, in which a user inputs a query in the form of a few keywords is the standard method for interacting with well known IR systems such as Google and Bing, research in IR has not restricted to keyword-based retrieval, nor has it restricted to the retrieval of documents~\cite{manning2010introduction}. \textit{Recommender Systems} on e-commerce websites that retrieve \textit{items}, or \textit{entities}, relevant to a customer, from a customer's profile and a large corpus of customer-item interactions are examples of \textit{non-keyword} IR systems that return items from a catalog without any textual query.\footnote{Other examples of \textit{non-keyword} IR systems are multimedia retrieval systems that allow users to use an audio recording or an image for retrieving the results. I will not focus on multimedia retrieval.}
% A second example of IR systems that do not retrieve documents are systems
% that retrieve entities from structured knowledge bases such as RDF triple stores
% or Knowledge Graphs. For example, major search engines
% nowadays populate a carousel of related films in response to queries like
% ``Movies in 2016''. This listing is generated by filtering a \textit{Knowledge
%   Graph} that contain meta-data about entities and their
% inter-relations and they are interesting objects of study by themselves.

Let us consider another application of non-keyword, non-document IR systems which will also motivate my research problem: Consider the situation of a recruiter who needs to find suitable candidates for a job from a large corpus of candidates. Moreover, the recruiter has access to the candidates' friendship network, \resumes, personal statements and publications. More specifically, the job may require people who possess a good understanding of ``information retrieval'' techniques and the ``Hindi'' language. However, it is possible that not every candidate's profile contains that information. Instead, publishing in the SIGIR conference, or being a citizen of India, or being friends with multiple people like that may be good indicators of the above qualities. These correlated qualities could be inferred from examples of desirable entities. This example suggests that:
\begin{enumerate}
\item Graph structured side information about entities can be useful for IR.
\item Examples of relevant entities provide useful feedback to an IR system.
\item Entity retrieval is a useful task.
\end{enumerate}

The problem of finding more items that are similar to a given set of items and the problem of finding items in response to a keyword query have both been studied extensively. In case the user does not provide any keyword query and only provides examples of items to be retrieved then the problem will be considered as the problem of \textbf{\textit{Set Completion}}

\subsection{Set Completion and Variants}
\label{sec:set-compl-vari}
Retrieving entities that are similar to a few example entities is an artificial intelligence task with broad utility.

\subsubsection{Examples of Set Completion Tasks}
\label{sssec:examples-set-comp}
The problem of finding a suitable candidate for a job can be framed as the problem of {Set Completion} if a few examples of suitable candidates are given and the system needs to rank the unknown candidates in a database according to their suitability for the job. This task of finding suitable candidates for a task is broadly referred to as the problem of \textit{Expertise Retrieval} and it operationalized and evaluated in various ways~\cite{balog2012expertise}. Expert Retrieval is a special case of a more general problem called \textit{Entity Retrieval} in which, a seed entity, a description of the relations between the target entity and the seed entity, and a few examples of the target entities are provided as inputs. Variants of this task, depending on whether entities are provided as examples or not, have been run multiple times in the annual TREC conference~\cite{balog2012expertise}.

Another example of a set completion task is the \textit{Document Routing} problem, in which documents related to a topic need to be retrieved given a natural language text describing a user's information need and some example documents is another example of the set completion task. Document Routing was an important shared task evaluated during the early annual TREC conferences~\cite{schutze1995a-comparison}.

A third example of the set completion problem is the task of
\textit{Item Recommendation} to customers from their purchase history and profile. If information about the purchase history of other customers is also available then the problem is called \textit{Collaborative Filtering} otherwise the problem is called \textit{Content Based Recommendation}. Note that technically the term \textit{filtering} should be used if the task is to classify whether an entity lies in a set instead of ranking the remaining entities. Such a problem, where the set of example entities needs to be increased is also called the \textit{Set Expansion} problem.

Set completion tasks where relational information amongst the entities is
available were called the \textit{Vertex Nomination} (VN) task
by~\cite{fishkind2015vertex} and \textit{Class-Instance acquisition}
by~\cite{talukdar2010experiments}. The Vertex Nomination terminology is apter in situations where the graphs are more homogeneous with lesser entity level features, and the edges between the entities are not too sparse.\footnote{Note that the research on VN is evolving and more algorithms are being introduced which get rid of some of these assumptions, and therefore these distinctions are not strict.} In the case where additional meta-data is available, and the graph is sparse, and there can be multiple possible types that a particular vertex can take on then the terminology of \textit{Class-Instance acquisition} is perhaps apter.

Finally, the task of \textit{Query By Entity}, or \textit{Query By Examples} has been studied in the database community and the semantic web community as a way of helping user's interact with databases or knowledge graphs~\cite{metzger2013qbees}. In these tasks, a computer system needs to infer and execute an unknown query on the basis of a few examples of the results, and therefore this task too can be considered to be a set completion task.

\subsubsection{Methods for Set Completion}
\label{sssec:methods-set-comp}
One of the earliest methods for set completion was a patented approach called ``Google Sets''~\cite{tong2008system}\footnote{Note that although the patent application was filed in 2003 it was only granted in 2008.} that performed set completion by modeling the input examples as samples from a mixture of distributions over pre-existing lists. After receiving a few examples of entities, the mixture components were estimated and then new entities were generated from this distribution. Inspired by this approach~\cite{ghahramani2005bayesian} introduced the method of ``Bayesian Sets''.  The Bayesian Sets method ranks the entities by the ratio of two probabilities. The first probability measures whether the entity and the data were generated from the same parameters and the second probability measures whether the data and the entity were generated independently. Some theoretical results about the \textit{stability} of the method in the presence of correlated features were presented in~\cite{letham2013growing}. However, the method by itself does not provide guarantees about the \textit{quality} of the rankings.

A different approach than the Bayesian sets method which creates a ``profile'' of the criterion for being in a set is to define a similarity function such that new items that are most similar to the example entities get a high rank. SEAL and its variants~\cite{wang2007language-independent,wang2008iterative} were early approaches that learned the similarity between the new entities and example entities using methods like ``Random Walks'', and ``Random Walks With Restart''. Other methods for computing similarity have also been explored such as the sum of cosine similarities amongst others. In modern parlance, any kernel method that computes similarities between two entities can be used for ranking the entities. The hyper-parameters for the kernel methods can be trained via $5$-fold cross-validation on the training data.

A very different approach to this problem is to treat the problem as a binary classification task, where only positively labeled examples and unlabeled data is available. Viewed this way any generative probabilistic model usable for binary classification can be applied in a principled manner to this problem, for example, the generative Naive Bayes algorithm was applied for binary classification. \cite{nigam1998learning,nigam2000text} used precisely this model with the EM algorithm to utilize unlabeled data to learn the parameters of a naive Bayes text classifier.

A different principled approach is to apply the PU learning framework~\cite{denis1998pac-learning}.  Under the PU learning framework, the learning algorithm tries to minimize the total probability of labeling the unlabeled data as positive while holding the probability of correctly labeling the labeled data above the desired recall rate~\cite{liu2002partially,liu2016learning}. \cite{li2010distributional} compared the performance of a PU learning-based classifier to distributional similarity based methods and the Bayesian Sets method on the problem of entity set expansion. The distributional similarity based method ranks entities from the cosine similarity of its distributional signature (PMI features). They showed that their PU learning method outperformed both distributional methods and the Bayesian Sets method. Recent work by~\cite{natarajan2015learning} and~\cite{rao2015collaborative} amongst others applied  PU learning methods to problems of matrix completion and collaborative filtering while incorporating additional graph information as well. Note that PU Learning is a very active research area and activity in this area is accelerating.

The problem of set completion can also be solved using the \textit{learning to rank} approach~\cite{li2014learning}. In the learning to rank framework, the set completion algorithm learns a pairwise decision function that receives two entities as inputs and decides which of those two entities is more likely to be a member of the set. This function can be trained so that it always picks the input labeled entities to be members of the set in comparison to unlabeled entities. Learning to rank methods are very popular in the Information Retrieval community~\cite{manning2008introduction}.

The problem of vertex nomination was tackled using the Adjacency Spectral Embedding method and its extensions by~\cite{sussman2012consistent,fishkind2015vertex} on communication graphs where the Stochastic Block Model is a reasonable approximation to the generative process for the observed graph. On the other hand, when the graphs were manually created knowledge graphs such as Freebase, or automatically extracted OpenIE knowledge graphs such as the Textrunner graph, which naturally exhibits sparsity, and bipartiteness, then this problem has been tackled using \textit{Graph Based Semi-Supervised Learning} algorithms such as the Label Propagation algorithm, the Adsorption, and the Modified Adsorption algorithm by~\cite{talukdar2010experiments}.\footnote{The framework of \textit{querying by examples} on RDF triples provides a useful framework for framing the problem of set completion. This is also a large area, and I leave it out of this review.}

Finally, set completion can be reduced to the problem of \textit{Link  Prediction} or \textit{Knowledge Base Completion} by adding a new meta-vertex that represents the set and then connecting the labeled entities in the training set to the meta-vertex. I can represent the likelihood of an unlabeled entity being part of the group as the score assigned to an edge that connects the meta-vertex to the unlabeled entity. See~\cite{nickel2016review} for a review of this area.

\subsection{Existing Work on Entity Search}
\label{sssec:existing-work-entity-search}
Research in entity search over large text corpora was accelerated with the start of the TREC entity retrieval and expertise retrieval tracks~\cite{balog2012overview,balog2012expertise}.  These shared tasks considered the same problem as us, where a query was a bag of keywords, and the result of a query was a ranked list of individual entities, each of which was an answer. \cite{dalton2014entity} further used knowledge graphs for feature expansion. One of the dominant methods was introduced by~\cite{balog2011query}, which is based on entity language models and harnesses entity categories for ranking and for restricting answers to the desired type. However, these methods have been tested only on situations where large Wikipedia pages were available for estimating the language models associated with an entity. A similar approach, of creating entity language models, using only the text surrounding the mentions of an entity was earlier explored by~\cite{raghavan2004exploration}. However, they did not consider how to incorporate side information or relevance feedback.

Searching and exploring text corpora that are annotated with entities and linked to a KG has been addressed in various projects, most notably the Broccoli system~\cite{bast2014semantic}, Facetedpedia~\cite{li2010facetedpedia}, ERQ~\cite{li2012entity-relationship}, STICS~\cite{hoffart2014stics}, and DeepLife~\cite{ernst2016deeplife}. The work by~\cite{agrawal2012entity} Expanded upon the work by~\cite{li2012entity-relationship} and added a notion of similarity between entities, which they called ``near'' queries. Additionally, they utilized the ``Spreading Activation'' method for including graph structure into the ranking of entities. Furthermore, they only considered the Wikipedia graph as an example.

\cite{sawant2013learning} and \cite{joshi2014knowledge} considered the problem of answering short keyword-based text queries over a combination of textual and structured data. Their approach was to jointly learn the segmentation, the entity, class and predicate interpretation of the input query (in text form), and the ranking of candidate results. They did not consider the problem of entity-based relevance feedback, however. \cite{yahya2016question} worked on supporting complex queries on knowledge bases that also contain textual web content in their fields. They called such knowledge graphs ``Extended Knowledge Base''.

Recently~\cite{savenkov2016when} and \cite{xu2016enhancing} considered the extended knowledge graph as a starting model for performing question answering over knowledge bases. Specifically, they showed that access to related text could improve the performance of various sub-components in information extraction and semantic parsing pipelines, such as entity linking and coreference resolution.

\subsection{Distributed Representations of Knowledge Graphs}
\label{sssec:backg-dist-repr}
The singular vectors of word-document co-occurrence matrices were one of the first vector representations of words and documents used in the field of NLP and Information Retrieval. After the work of~\cite{mikolov2013linguistic,mikolov2013distributed} there was an explosion of activity in the area of learning vector representations of words, graphs, and other discrete structures. Interesting new directions were proposed by~\cite{vilnis2015word} and~\cite{rudolph2016exponential}. \cite{vilnis2015word} proposed to represent each word in a sequence by a Gaussian distribution, and this work was extended to learning representations for entities in a knowledge graph by~\cite{he2015learning}. On the other hand, \cite{rudolph2016exponential} proposed the \textrm{EF-EMB} model to represent the conditional distribution of a ``related'' entity given a ``base'' entity using exponential family distributions. They applied their model to the task of predicting a neuron's activity from its neighbors' activities.

Recently~\cite{he2015learning} applied the Gaussian Embedding method presented by \cite{vilnis2015word} to learn vector representations of graph vertices, and they tested their  learned representations on tasks such as link prediction.\footnote{The link prediction task aims to find the correct entity that should be linked to a given entity with a given relation and measures performance using IR metrics.}

%auto-ignore
\chapter{Multiview LSA}
\label{cha:mvlsa}
The primary goal of this chapter\footnote{A previous version of this work was published in~\cite{rastogi2015multiview}.}  is to motivate and describe the Multiview LSA (MVLSA)  algorithm which is a significant generalization of classical methods such as LSA.  To that end, I compared the performance of MVLSA against single view LSA as well as other contemporary methods such as Glove~\cite{pennington2014glove} and SkipGram Word2Vec~\cite{mikolov2013distributed} on the tasks of word-similarity and word-analogy. These tasks measure whether the representation of words learned from an unsupervised text corpus contains information about the semantic similarity between words or not.

A possible criticism of this choice of task for evaluation is that word-similarity and analogy do not represent an end-task. To that end, I will present experiments on the downstream tasks of Contextual Mention Retrieval and Entity Linking in \chapref{cha:more}. In this chapter I focus on the tasks of word similarity and analogy for two main reasons:
\begin{enumerate}
\item A large number of resources exist for extracting word co-occurrence matrices, and therefore it is possible to evaluate large-scale multi-view embeddings of words.
\item Words form the basis of language, and a large number of downstream NLP models benefit when initialized via word-embeddings, especially in low-resource scenarios. Therefore, the performance of a method on tasks such as analogy and similarity is important in its own right.
\end{enumerate}

\section{Introduction}
\newcite{winograd1972understanding} wrote that: \emph{``Two sentences are paraphrases if they produce the same representation in the internal formalism for meaning''}.  This intuition is made soft in vector-space models \cite{turney2010frequency}, where one says that expressions in language are paraphrases if their representations are \emph{close} under some distance measure.

One of the earliest linguistic vector space models was Latent Semantic Analysis (LSA). LSA has been successfully used for Information Retrieval, but it is limited in its reliance on a single matrix, or \emph{view}, of term co-occurrences. In this chapter, I address the single-view limitation of LSA by demonstrating that the framework of Generalized Canonical Correlation Analysis (GCCA) can be used to perform Multiview LSA (MVLSA). This approach allows for the use of an arbitrary number of views in the induction process, including embeddings induced using other algorithms. I also present a fast approximate method for performing GCCA and approximately recover the objective of \cite{pennington2014glove} while accounting for missing values.

My experiments show that MVLSA is competitive with state of the art approaches for inducing vector representations of words and phrases. As a methodological aside, I discuss the \mbox{(in-)significance} of conclusions being drawn from comparisons done on small sized datasets.

\section{Motivation}
LSA is an application of Principal Component Analysis (PCA) to a term-document cooccurrence matrix.  The principal directions found by PCA form the basis of the vector space in which to represent the input terms~\cite{landauer1997solution}. A drawback of PCA is that it can leverage only a single source of data and it is sensitive to scaling.

An arguably better approach to representation learning is Canonical Correlation Analysis (CCA) that induces representations that are maximally \emph{correlated} across two views, allowing the utilization of two distinct sources of data. While an improvement over PCA, being limited to only two views is unfortunate because many sources of data (perspectives) are frequently available in practice.  In such cases, it is natural to extend CCA's original objective of maximizing the correlation between two views by maximizing some measure of the matrix $\Phi$ that contains all the pairwise correlations between linear projections of the \emph{covariates}. This is how Generalized Canonical Correlation Analysis (GCCA) was first derived by~\newcite{horst1961generalized}. Recently these intuitive ideas about benefits of leveraging multiple sources of data have received strong theoretical backing due to work by \newcite{sridharan2008information} who showed that learning with multiple views is beneficial since it reduces the complexity of the learning problem by restricting the search space.  Recent work by \newcite{anandkumar2014tensor} showed that at least three views are necessary for recovering hidden variable models.

Note that there exist different variants of GCCA depending on the measure of $\Phi$ that one chooses to maximize. \newcite{kettenring1971canonical} enumerated a variety of possible measures, such as the spectral-norm of $\Phi$. Kettenring noted that maximizing this spectral-norm is equivalent to finding linear projections of the \emph{covariates} that are most amenable to rank-one PCA, or that can be best explained by a single term factor model. This variant was named \emph{MAX-VAR GCCA} and was shown to be equivalent to a proposal by \newcite{carroll1968generalization}, which searched for an auxiliary orthogonal representation $G$ that was maximally correlated to the linear projections of the covariates.  Carroll's objective targets the intuition that representations leveraging multiple views should correlate with all provided views as much as possible.

\section{Proposed Method: MVLSA}
\label{sec:gcca}
Let $X_j \in \mathbb{R}^{N\times d_j} \; \forall j \in [1,\ldots,J]$ be the mean centered matrix containing data from view $j$ such that row $i$ of $X_j$ contains the information for word $w_i$. Let the number of words in the vocabulary be $N$ and number of contexts (columns in $X_j$) be $d_j$. Note that $N$ remains the same and $d_j$ varies across views. Following standard notation \cite{hastie2009elements} I call $X_j^\top X_j$ the scatter matrix and $X_j (X_j^\top X_j)^{-1}X_j^\top$ the projection matrix.

The objective of \emph{MAX-VAR GCCA} can be written as the following optimization problem: Find $G \in \mathbb{R}^{N\times r}$ and $U_j \in \mathbb{R}^{d_j \times r}$ that solve:
\begin{equation}
  \label{eq:gcca}
\begin{split}
  \operatorname*{\arg\,\min}_{G,U_j} & \sum_{j=1}^J \begin{Vmatrix} G - X_jU_j \end{Vmatrix}^2_F \\
  \text{subject to } & G^\top G = I.
\end{split}
\end{equation}
The matrix $G$ that solves problem~(\ref{eq:gcca}) is my vector representation of the vocabulary. Finding $G$ reduces to spectral decomposition of sum of projection matrices of different views: Define
\begin{align}
P_j =& X_j(X_j^\top X_j)^{-1}X_j^\top, \label{eq:pp}\\
M =& \sum_{j=1}^J P_j. \label{eq:mm}
\end{align}
Then, for some positive diagonal matrix $\Lambda$, $G$ and $U_j$ satisfy:
\begin{align}
M G =& G \Lambda,\\
U_j =& \left(X_j^\top X_j\right)^{-1} X_j^\top G.
\end{align}

The above expressions tell us that my word representations are the eigenvectors of the sum of $J$ projection matrices. Also, note that the dimensions of $G$ are orthogonal to each other. Orthogonality of representations can be a desirable property that I will discuss in more detail at the end of this chapter.

Computationally storing $P_j \in \mathbb{R}^{N \times N}$ is problematic owing to memory constraints.  Further, the scatter matrices may be non-singular leading to an ill-posed procedure. I now describe a novel scalable GCCA with $\ell_2$-regularization to address these issues.

\noindent\textbf{Approximate Regularized GCCA}: GCCA can be regularized by adding $r_jI$ to scatter matrix $X_j^\top X_j$ before doing the inversion where $r_j$ is a small constant e.g. $10^{-8}$. Projection matrices in~(\ref{eq:pp}) and (\ref{eq:mm}) can then be written as
\begin{align}
  \widetilde{P}_{j} =& X_j(X_j^\top X_j+r_jI)^{-1}X_j^\top, \label{eq:6}\\
  M =& \sum_{j=1}^J \widetilde{P}_{j}. \label{eq:mmm}
\end{align}

Next, to scale up GCCA to large datasets, I first form a rank-$m$ approximation of projection matrices~\cite{arora2012kernel} and then extend it to an eigendecomposition for $M$ following ideas by~\newcite{savostyanov}. Consider the rank-$m$ SVD of $X_j$: $$X_j = A_{j} S_{j} B^\top_{j},$$ where $S_j \in \R^{m \times m}$ is the diagonal matrix with $m$-largest singular values of $X_j$ and $A_j \in \R^{N \times m}$ and $B_j \in \R^{m \times d_j}$ are the corresponding left and right singular vectors. Given this SVD, write the $j^{th}$ projection matrix as
\begin{eqnarray}
\widetilde{P}_j &=& A_j S_j^\top(r_j I + S_j S_J^\top)^{-1}S_j A_j^\top, \nonumber \\
&=& A_j T_j T_j^\top A_j^\top, \nonumber
\end{eqnarray}
where $T_j \in \mathbb{R}^{m \times m}$ is a diagonal matrix such that $T_jT_j^\top = S_j^\top(r_j I + S_j S_J^\top)^{-1}S_j$. Finally, I note that the sum of projection matrices can be expressed as $M = \tilde{M} \tilde{M}^\top$ where $$\tilde{M} = \left[ A_1T_1 \ldots A_JT_J \right] \in \mathbb{R}^{N \times mJ}.$$ Therefore, eigenvectors of matrix $M$, i.e. the matrix $G$ that I am interested in finding, are the left singular vectors of $\tilde{M}$, i.e. $\tilde{M}=GSV^\top$. These left singular vectors can be computed by using Incremental PCA \cite{brand2002incremental} since $\tilde{M}$ may be too large to fit in memory.

Let $SVD_m$ denote a partial SVD where $S_j$ is a rectangular diagonal matrix that contains only the $m$ largest singular values and $A_j, B_j$ are square, orthonormal, unitary matrices. Defining $SVD_m$ like this ensures correctness but in practice one only needs to compute $m$ columns of $A_j$. Take the SVD of $X_j$: $$A_{j} S_{j} B^\top_{j} \xleftarrow{SVD_{m}} X_j$$ and substitute the above in equation~\ref{eq:6} to get $$\widetilde{P}_j = A_j S_j^\top(r_j I + S_j S_J^\top)^{-1}S_j A_j^\top$$. Define $T_j \in \mathbb{R}^{m \times m}$ to be the diagonal matrix such that $T_jT_j^\top = S_j^\top(r_j I + S_j S_J^\top)^{-1}S_j $ then $$\widetilde{P}_j = A_j T_j T_j^\top A_j^\top$$. Now $\tilde{M} = \left[ A_1T_1 \ldots A_JT_J \right] \in \mathbb{R}^{N \times mJ}$, then $$M = \tilde{M} \tilde{M}^\top.$$ Performing QR decomposition of $\tilde{M}$ gives $$M = Q R R^\top Q$$. Eigen decomposition of $R R^\top \in \mathbb{R}^{mJ \times mJ}$ results in eigen vectors $U$ and eigen values $S$. $$M = Q U S U^\top Q^\top$$ which implies $G = QU$.

\subsection{Computing SVD of mean centered $X_j$}
\label{ssec:svdmc}
Recall that I assumed $X_j$ to be mean centered matrices. Let $Z_j \in \mathbb{R}^{N \times d_j}$ be sparse matrices containing mean-uncentered cooccurrence counts. Let $f_j = n_j \circ t_j $ be the preprocessing function that I will apply to $Z_j$:
\begin{align}
  Y_j =& f_j (Z_j), \\
  X_j =& Y_j - 1 (1^\top Y_j).
\end{align}
To compute the SVD of mean-centered matrices $X_j$ I first compute the partial SVD of an uncentered matrix $Y_j$ and then update it (\newcite{brand2006fast} provides details). I experimented with representations created from the uncentered matrices $Y_j$ and found that they performed as well as the mean centered versions, but I will not mention them further since it is computationally efficient to follow the principled approach. I should note, however, that even the method of mean-centering the SVD produces an approximation.

\subsection{Handling missing rows across views}
\label{ssec:missing}
With real data, it may happen that a term was not observed in a view at all. A large number of missing rows can corrupt the learned representations since the rows in the left singular matrix become zero. The procedure described above can not recover from this, and the representation for those words may become a one-hot vector.  To counter this problem, I adopt a variant of the ``missing-data passive'' algorithm from \newcite{van2006generalized} who modified the GCCA objective to counter the problem of missing rows.\footnote{A more recent effort, by \newcite{van2012generalized}, describes newer iterative and non-iterative (Test-Equating Method) approaches for handling missing values. It is possible that using one of those methods could improve performance. \label{ftn:mis}} The objective now becomes:
\begin{equation}
  \label{eq:gcca2}
\begin{split}
  \operatorname*{arg\,min}_{G,U_j} & \sum_{j=1}^J \begin{Vmatrix} K_j(G - X_jU_j) \end{Vmatrix}^2_F \\
  \text{subject to } & G^\top G = I,
\end{split}
\end{equation}
where $[K_j]_{ii} = 1$ if row $i$ of view $j$ is observed and zero otherwise. Essentially $K_j$ is a diagonal row-selection matrix which ensures that I will optimize the GCCA representations only on the observed rows. Note that $X_j = K_jX_j$ since the rows that $K_j$ removed were already zero. Let, $K = \sum_j K_j$ then the optima of the objective can be computed by modifying equation~(\ref{eq:mmm}) as: \begin{align} M =& K^{-\frac{1}{2}}(\sum_{j=1}^J P_j)K^{-\frac{1}{2}}. \end{align} Again, if I regularize and approximate the GCCA solution then I get $G$ as the left singular vectors of $K^{-\frac{1}{2}}\tilde{M}$. I mean center the matrices using only the observed rows.

Also note that other heuristic weighting schemes could be used here. For example if I modify my objective as follows then I will approximately recover the objective of \newcite{pennington2014glove}:
\begin{eqnarray} \label{eq:gcca3} \minimize{G,U_j}{\sum_{j=1}^J \begin{Vmatrix} W_j K_j(G - X_jU_j) \end{Vmatrix}^2_F}{G^\top G = I } \end{eqnarray}
where
\begin{eqnarray}
[W_j]_{ii} &=& \left(\frac{w_i}{w_{\max}}\right)^{\frac{3}{4}} \text{ if } w_i < w_{\max} \text{ else } 1, \nonumber \\
  \text{and } w_i &=&  \sum_k [X_j]_{ik}. \nonumber
\end{eqnarray}

\section{Data}
\label{sec:data}
\noindent\textbf{Training Data} I used the English portion of the \textit{Polyglot} Wikipedia dataset released by \newcite{al2013polyglot} to create 15 \emph{irredundant} views of co-occurrence statistics where element $[z]_{ij}$ of view $Z_k$ represents that number of times word $w_j$ occurred $k$ words behind $w_i$. I selected the top 500K words by occurrence to create my vocabulary for the rest of the chapter. I lowercased all the words and discarded all words which were longer than 5 characters and contained more than 3 non-alphabetical symbols. This was done to preserves years and smaller
numbers.

I extracted co-occurrence statistics from a large bitext corpus that was made by combining a number of parallel bilingual corpora as part of the ParaPhrase DataBase (PPDB) project: Table~\ref{tab:dataperlang} gives a summary, \newcite{ganitkevitch2013ppdb} provides further details. Element $[z]_{ij}$ of the \textit{bitext} matrix represents the number of times English word $w_i$ was automatically aligned to the foreign word $w_j$.

I also used the dependency relations in the \textit{Annotated Gigaword Corpus}~\cite{annotatedGigaword12} to create 21 views\footnote{Dependency relations employed: nsubj, amod, advmod, rcmod, dobj, prep\xline{}of, prep\xline{}in, prep\xline{}to, prep\xline{}on, prep\xline{}for, prep\xline{}with, prep\xline{}from, prep\xline{}at, prep\xline{}by, prep\xline{}as, prep\xline{}between, xsubj, agent, conj\xline{}and, conj\xline{}but, pobj.}  where element $[z]_{ij}$ of view $Z_d$ represents the number of times word $w_j$ occurred as the governor of word $w_i$ under dependency relation $d$.

I selected these dependency relations since they seemed to be particularly interesting which could capture different aspects of similarity.

I combined the knowledge of paraphrases present in FrameNet and PPDB by using the dataset created by \newcite{rastogi2014augmenting} to construct a \textit{FrameNet} view. Element $[z]_{ij}$ of the \textit{FrameNet} view represents whether word $w_i$ was present in frame $f_j$. Similarly I combined the knowledge of morphology present in the \textit{CatVar} database released by \newcite{habash2003catvar} and \textit{morpha} released by \newcite{minnen2001applied} along with \textit{morphy} that is a part of WordNet. The morphological views and the frame semantic views were especially sparse with densities of 0.0003\% and 0.03\%. While the approach allows for an arbitrary number of distinct sources of semantic information, such as going further to include cooccurrence in WordNet synsets, I considered the described views to be representative, with further improvements possible as future work.
%%Figure~\ref{fig:cartoon} summarizes our description.

\begin{figure}
  \centering
  \includegraphics[width=\linewidth]{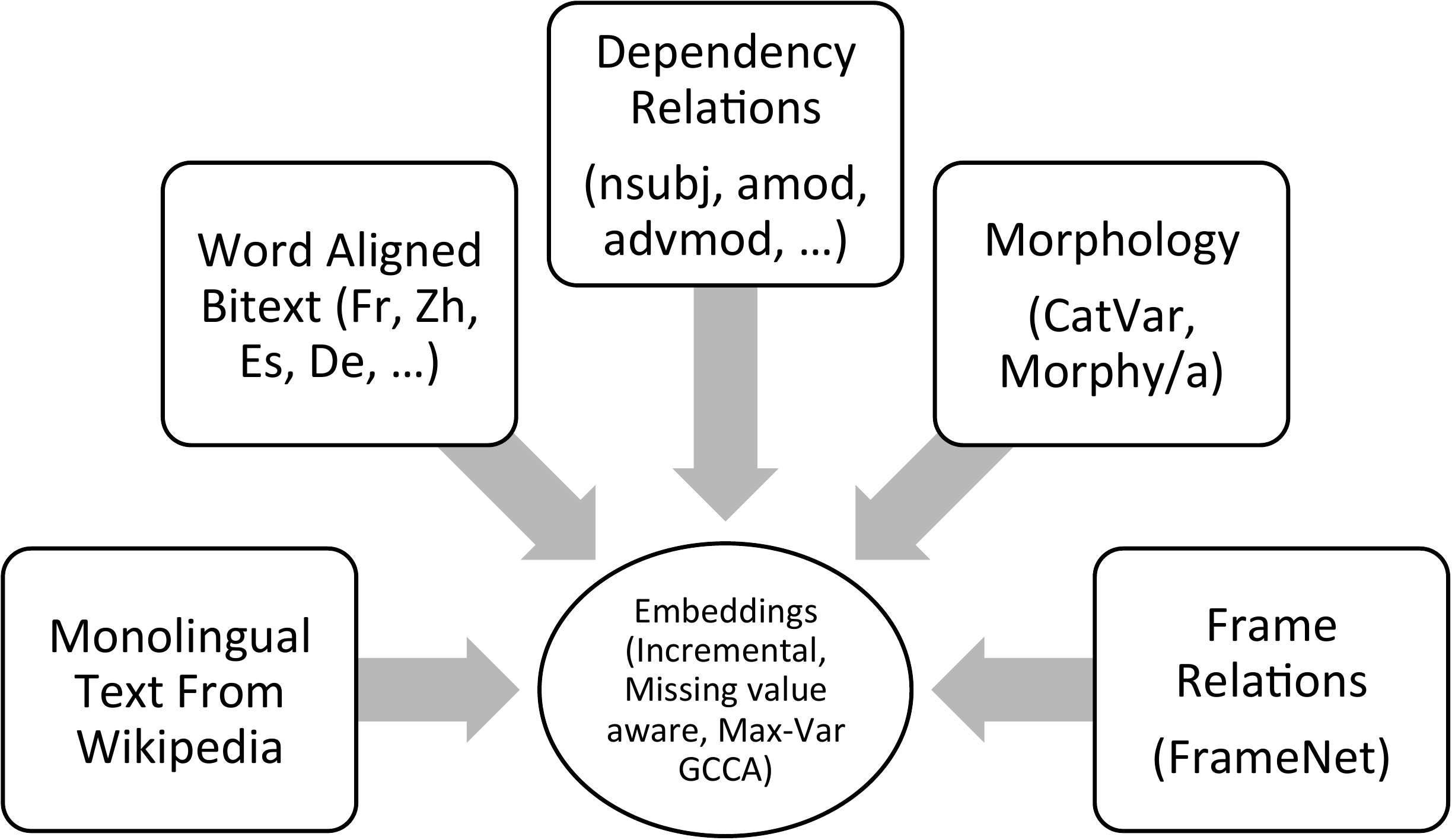}
  \caption[Input datasets for MVLSA.]{An illustration of datasets used.}
  \label{fig:cartoon}
\end{figure}

\begin{table}[htbp]
  \centering
  %\rowcolors{1}{}{lightgray}
  \begin{tabular}{l | rr}
    Language & Sentences & English Tokens \\
    \hline
    Bitext-Arabic   & 8.8M   & 190M  \\
    Bitext-Czech    & 7.3M   & 17M   \\
    Bitext-German   & 1.8M   & 44M   \\
    Bitext-Spanish  & 11.1M  & 241M  \\
    Bitext-French   & 30.9M  & 671M  \\
    Bitext-Chinese  & 10.3M  & 215M  \\
    Monotext-En-Wiki& 75M    & 1700M
  \end{tabular}
  \caption[Dataset statistics for MVLSA experiments]{Portion of data used to create GCCA representations (in millions).}
  \label{tab:dataperlang}
\end{table}

\noindent\textbf{Test Data} I evaluated the representations on the word similarity datasets listed in Table~\ref{tab:testlist}. The first 10 datasets in Table~\ref{tab:testlist} were annotated with different rubrics and rated on different scales. However, broadly they all contain human judgments about how similar two words are. The ``AN-SYN'' and ``AN-SEM'' datasets contain 4-tuples of analogous words, and the task is to predict the missing word given the first three. Both of these are open vocabulary tasks while TOEFL is a closed vocabulary task.
\begin{table*}[ht]
  \begin{adjustwidth}{-1cm}{}
  %\rowcolors{1}{}{lightgray}
  \begin{tabular}{lr | ccc  | ccc | l}
    Acronym & Size  &
    $\sigma_{0.01}^{0.5}$ & $\sigma_{0.01}^{0.7}$ & $\sigma_{0.01}^{0.9}$ &
    $\sigma_{0.05}^{0.5}$ & $\sigma_{0.05}^{0.7}$ & $\sigma_{0.05}^{0.9}$ &
    Reference  \\
    \hline
    MEN    & 3000  & 4.2  & 3.2  & 1.8  & 3.0  & 2.3  & 1.3  & \cite{bruni2012distributional}  \\
    RW     & 2034  & 5.1  & 3.9  & 2.3  & 3.6  & 2.8  & 1.6  & \cite{Luong2013morpho}          \\
    SCWS   & 2003  & 5.1  & 4.0  & 2.3  & 3.6  & 2.8  & 1.6  & \cite{Huang2012Improving}       \\
    SIMLEX & 999   & 7.3  & 5.7  & 3.2  & 5.2  & 4.0  & 2.3  & \cite{hill2014simlex}           \\
    WS     & 353   & 12.3 & 9.5  & 5.5  & 8.7  & 6.7  & 3.9  & \cite{finkelstein2001placing}   \\
    MTURK  & 287   & 13.7 & 10.6 & 6.1  & 9.7  & 7.5  & 4.3  & \cite{Radinsky2011word}         \\
    WS-REL & 252   & 14.6 & 11.3 & 6.5  & 10.3 & 8.0  & 4.6  & \cite{agirre2009study}          \\
    WS-SEM & 203   & 16.2 & 12.6 & 7.3  & 11.5 & 8.9  & 5.1  & -Same-As-Above-                 \\
    RG     & 65    & 28.6 & 22.3 & 12.9 & 20.6 & 16.0 & 9.2  & \cite{Rubenstein1965Contextual} \\
    MC     & 30    & 41.7 & 32.7 & 19.0 & 30.6 & 23.9 & 13.8 & \cite{miller1991contextual}     \\ \hline
    AN-SYN  & 10675 & -    & -    & 0.95 & -    & -    & 0.68 & \cite{mikolov2013distributed}   \\
    AN-SEM  & 8869  & -    & -    & 1.03 & -    & -    & 0.74 & -Same-As-Above-                 \\
    TOEFL  & 80    & -    & -    & 8.13 & -    & -    & 6.63 & \cite{landauer1997solution}
  \end{tabular}
  \caption[Testsets for MVLSA evaluation.]{List of test datasets used. The columns headed $\sigma_{p_0}^r$ contain \emph{MRDS} values. The rows for accuracy based test sets contain $\sigma_{p_0}$ which does not depend on $r$. See \S~\ref{ssec:mrds} for details.}
  \label{tab:testlist}
  \end{adjustwidth}
\end{table*}

\subsection{Significance of comparison} \label{ssec:mrds}
While surveying the literature, I found that performance on word similarity datasets is typically reported in terms of the Spearman correlation between the gold ratings and the cosine distance between normalized embeddings.  However, researchers do not report measures of significance of the difference between the Spearman Correlations even for comparisons on small evaluation sets.\footnote{For example, the relative difference between competing algorithms reported by \newcite{faruqui2014retrofitting} could not be significant for the Word Similarity test set released by \newcite{finkelstein2001placing}, even if I assumed a correlation between competing methods as high as 0.9, with a p-value threshold of 0.05.  Similar such comparisons on small datasets are performed by \newcite{hill2014not}.} This motivated me to define a method for calculating the \emph{Minimum Required Difference for Significance (MRDS)}.

\noindent\textbf{Minimum Required Difference for Significance (MRDS)}: Imagine two lists of ratings over the same items, produced respectively by algorithms $A$ and $B$, and then a list of gold ratings $T$. Let $r_{AT}$, $r_{BT}$ and $r_{AB}$ denote the Spearman correlations between $A:T$, $B:T$ and $A:B$ respectively. Let $\hat{r}_{AT}, \hat{r}_{BT}, \hat{r}_{AB}$ be their empirical estimates and assume that $\hat{r}_{BT} > \hat{r}_{AT}$ without loss of generality.

For word similarity datasets I define $\sigma_{p_0}^r$ as the MRDS,  such that it satisfies the following proposition: {\small $$ (r_{AB} < r) \land (|\hat{r}_{BT} - \hat{r}_{AT}|{<}\sigma_{p_0}^r) {\implies} \textit{pval} > p_0$$}. Here $\textit{pval}$ is the probability of the test statistic under the null hypothesis that $r_{AT} = r_{BT}$ found using the Steiger's test \cite{steiger1980tests}. The above constraint ensures that as long as the correlation between the competing methods is less than $r$ and the difference between the correlations of the scores of the competing methods to the gold ratings is less than $\sigma_{p_0}^r$, then the p-value of the null hypothesis will be greater than $p_0$.  Now let us ask what is a reasonable upper bound on the agreement of ratings produced by competing algorithms: for instance, two algorithms correlating above $0.9$ might not be considered meaningfully different.  That leaves us with the second part of the predicate which ensures that as long as the difference between the correlations of the competing algorithms to the gold scores is less than $\sigma_{p_0}^r$ then the null hypothesis is more likely than $p_0$.

I can find $\sigma_{p_0}^r$ as follows: Let $\textit{stest}$ denote Steiger's test predicate which satisfies the following: $$\textit{stest-p}(\hat{r}_{AT}, \hat{r}_{BT}, r_{AB}, p_0, n) {\implies} \textit{pval} < p_0$$ Once I define this predicate then I can use it to set up an optimistic problem where my aim is to find $\sigma_{p_0}^r$ by solving the following: {\small $$\sigma_{p_0}^r = \min\{\sigma | \forall\, 0 {<} r' {<} 1\, \textit{stest-p}(r', \min(r'+\sigma, 1), r, p_0, n) \} $$} Note that MRDS is a liberal threshold and it only guarantees that differences in correlations below that threshold can never be statistically significant (under the given parameter settings). MRDS might optimistically consider some differences as significant when they are not, but it is at least useful in reducing some of the noise in the evaluations.  The values of $\sigma_{p_0}^r$ are shown in Table~\ref{tab:testlist}.

For the accuracy based test-sets I found MRDS$=\sigma_{p_0}$ that satisfied the following: $$ 0< (\hat{\theta}_{B} - \hat{\theta}_{A})<\sigma_{p_0} {\implies} \text{p}(\theta_{B} \le \theta_{A}) > p_0$$ Specifically, I calculated the posterior probability $\text{p}(\theta_{B} \le \theta_{A})$ with a flat prior of $\beta(1,1)$ to solve the following:\footnote{This instead of using  McNemar's test \cite{mcnemar1947note} since the Bayesian approach is tractable and more direct. A calculation with $\beta(0.5, 0.5)$ as the prior changed $\sigma_{0.5}$ from 6.63 to 6.38 for the TOEFL dataset but did not affect MRDS for the AN-SEM and AN-SYN datasets.}{ $\sigma_{p_0}=\min\{\sigma |\forall\, 0{<}\theta{<}\min(1{-}\sigma,0.9)\,$ $\text{p}(\theta_{B}{\le} \theta_{A}| \hat{\theta}_A{=}\theta, \hat{\theta}_B{=}\theta+\sigma, n) < p_0\}$} Here $\theta_{A}$ and $\theta_{B}$ are probability of correctness of algorithms $A$, $B$ and $\hat{\theta}_{A}$, $\hat{\theta}_{B}$ are observed empirical accuracies.

Unfortunately, there are no widely reported train-test splits of the above datasets, leading to potential concerns of \emph{soft supervision} (hyper-parameter tuning) on these evaluations throughout the existing literature.  I report on the resulting impact of various parameterizations, and my final results are based on a single set of parameters used across all evaluation sets.

%% Here I follow
%% this accepted practice, but on ongoing work are exploring evaluation
%% of these learned representations in downstream systems as

%% and I also evaluated the effects of hyper-parameter tuning
%% on the entire test set, therefore, my final comparison could have
%% favored us due to \emph{soft supervision}.
%% However the consistent performance of our
%% method across the test sets lends hope that the trends I report would
%% generalize.

\section{Experiments and Results}
\label{sec:exp}
I wanted to answer the following questions through my experiments: (1) How do hyper-parameters affect performance? (2) What is the contribution of the multiple sources of data to performance? (3) How does the performance of MVLSA compare with other methods? I show the tuning runs on both larger and smaller datasets. I also highlight the top performing configurations in bold using the small threshold values in column~$\sigma_{0.05}^{0.09}$ of Table~\ref{tab:testlist}.

\noindent\textbf{Effect of Hyper-parameters} $f_j$: I modeled the preprocessing function $f_j$ as the composition of two functions, $f_j = n_j \circ t_j$. $n_j$ represents nonlinear preprocessing that is usually employed with LSA. I experimented by setting $n_j$ to be: identity; logarithm of count plus one; and the fourth root of the count.\footnote{I also experimented with other powers of the counts (0.12, 0.5 and 0.75) on a smaller dataset and found that the fourth root performed the best.} $t_j$ represents the truncation of columns and can be interpreted as a type of regularization of the raw counts themselves through which I prune away the noisy contexts. The decrease in $t_j$ also reduces the influence of views that have a large number of context columns and emphasizes the sparser views. Table~\ref{tab:n} and Table~\ref{tab:t} show the results.
\begin{table}[htbp]
  \centering
  \begin{tabular}{=l| +c +c +c}
    Test Set                            & Log  & Count & Count$^{\frac{1}{4}}$ \\ \hline
    MEN                                 & 67.5 & 59.7  & \mb{70.7}                  \\
    RW                                  & 31.1 & 25.3  & \mb{37.8}                  \\
    SCWS                                & 64.2 & 58.2  & \mb{66.6}                  \\%\remove{
    SIMLEX                              & 36.7 & 27.0  & \mb{38.0}                  \\
\rowstyle{\color{darkergray}}    WS     & 68.0 & 60.4  & \mb{70.5}                  \\
\rowstyle{\color{darkergray}}    MTURK  & 57.3 & 55.2  & \mb{60.8}                  \\
\rowstyle{\color{darkergray}}    WS-REL & 60.4 & 52.7  & \mb{62.9}                  \\
\rowstyle{\color{darkergray}}    WS-SEM & 75.0 & 67.2  & \mb{76.2}                  \\
\rowstyle{\color{darkergray}}    RG     & 69.1 & 55.3  & \mb{75.9}                  \\
\rowstyle{\color{darkergray}}    MC     & 70.5 & 67.6  & \mb{80.9}                  \\%}
    AN-SYN                               & 45.7 & 21.1  & \mb{53.6}                  \\
    AN-SEM                               & 25.4 & 15.9  & \mb{38.7}                  \\%\remove{
  \rowstyle{\color{darkergray}}  TOEFL  & 81.2 & 70.0  & \mb{81.2} %}
  \end{tabular}
  \caption[MVLSA performance versus non-linear pre-processing.]{Performance versus $n_j$, the non linear processing of cooccurrence counts.$\, t =200K, \; m=500, \; v=16, \; k=300$. All the top configurations determined by $\sigma_{0.05}^{0.09}$ are in bold font.}
  \label{tab:n}
\end{table}

\begin{table}[htbp]
  \centering
  {
  \begin{tabular}{=l | +c +c +c +c H +c H +c}
Test Set                            & 6.25K & 12.5K & 25K  & 50K  & 75K  & 100K & 150K & 200K \\ \hline
MEN                                 & 70.2  & \mi{71.2}  & \mi{71.5} & \mi{71.6} & \mi{71.4} & \mi{71.2} & \mi{71.0} & \mi{70.7} \\
RW                                  & \mi{41.8}  & \mi{41.7}  & \mi{41.5} & \mi{40.9} & \mi{40.7} & 39.6 & 38.3 & 37.8 \\
SCWS                                & \mi{67.1}  & \mi{67.3}  & \mi{67.1} & \mi{67.0} & \mi{67.3} & \mi{66.9} & \mi{66.8} & \mi{66.6} \\ %\remove{
SIMLEX                              & 42.7  & \mb{42.4}  & 41.9 & 41.3 & 40.5 & 39.5 & 38.4 & 38.0 \\
\rowstyle{\color{darkergray}}WS     & 68.1  & 70.8  & 71.6 & 71.2 & 71.3 & 70.2 & 70.8 & 70.5 \\
\rowstyle{\color{darkergray}}MTURK  & 62.5  & 59.7  & 59.2 & 58.6 & 58.3 & 60.3 & 61.0 & 60.8 \\
\rowstyle{\color{darkergray}}WS-REL & 60.8  & 65.1  & 65.7 & 64.8 & 65.2 & 63.7 & 63.7 & 62.9 \\
\rowstyle{\color{darkergray}}WS-SEM & 77.8  & 78.8  & 78.8 & 78.2 & 77.7 & 76.5 & 77.0 & 76.2 \\
\rowstyle{\color{darkergray}}RG     & 72.7  & 74.4  & 74.7 & 75.0 & 75.0 & 74.3 & 75.6 & 75.9 \\
\rowstyle{\color{darkergray}}MC     & 75.2  & 75.9  & 79.9 & 80.3 & 81.0 & 76.9 & 79.6 & 80.9 \\%}
AN-SYN                               & 59.2  & \mi{60.0}  & \mi{59.5} & 58.4 & 57.4 & 56.1 & 54.3 & 53.6 \\
AN-SEM                               & 37.7  & \mi{38.6}  & \mi{39.4} & \mi{39.2} & \mi{39.4} & 38.4 & \mi{38.8} & \mi{38.7} \\%\remove{
\rowstyle{\color{darkergray}}TOEFL  & 88.8  & 87.5  & 85.0 & 83.8 & 83.8 & 83.8 & 82.5 & 81.2%}
      \end{tabular}
  }
  \caption[Performance versus truncation threshold hyperparameter.]{Performance versus the truncation threshold, $t$, of raw cooccurrence counts. I used $n_j=\textrm{Count}^{\frac{1}{4}}$ and other settings were the same as Table~\ref{tab:n}.}
  \label{tab:t}
\end{table}
$m$: The number of left singular vectors extracted after SVD of the preprocessed cooccurrence matrices can again be interpreted as a type of regularization, since the result of this truncation is that I find cooccurrence patterns only between the top left singular vectors. I set $m_j = max(d_j, m)$ with $m=[100, 300, 500]$. See table~\ref{tab:m}.

\begin{table}[htbp]
  \centering
  \begin{tabular}{=l | +c +c +c +c}
Test Set                            & 100  & 200  & 300  & 500  \\\hline
MEN                                 & 65.6 & 68.5 & \mi{70.1} & \mi{71.1} \\
RW                                  & 34.6 & \mi{36.0} & \mi{37.2} & \mi{37.1} \\
SCWS                                & 64.2 & \mi{65.4} & \mi{66.4} & \mi{66.5} \\%\remove{
SIMLEX                              & 38.4 & 40.6 & \mb{41.1} & 40.3 \\
\rowstyle{\color{darkergray}}WS     & 60.4 & 67.1 & 69.4 & \mb{71.1} \\
\rowstyle{\color{darkergray}}MTURK  & 51.3 & 58.3 & 58.4 & \mb{58.9} \\
\rowstyle{\color{darkergray}}WS-REL & 49.0 & 58.2 & 61.6 & \mb{65.1} \\
\rowstyle{\color{darkergray}}WS-SEM & 73.6 & 76.8 & 76.8 & \mb{78.0} \\
\rowstyle{\color{darkergray}}RG     & 61.6 & 69.7 & 73.2 & \mb{74.6} \\
\rowstyle{\color{darkergray}}MC     & 65.6 & 74.1 & \mb{78.3} & 77.7 \\%}
AN-SYN                               & 50.5 & \mi{56.2} & \mi{56.4} & \mb{56.4} \\
AN-SEM                               & 24.3 & 31.4 & 34.3 & \mb{40.6} \\%\remove{
\rowstyle{\color{darkergray}} TOEFL & 80.0 & 81.2 & \mb{82.5} & 80.0%}
  \end{tabular}
  \caption[Performance versus intermediate dimensionality.]{Performance versus $m$, the number of left singular vectors extracted from raw cooccurrence counts. I set $n_j=\textrm{Count}^\frac{1}{4}, \; t=100K, \; v=25, \; k=300$.}
  \label{tab:m}
\end{table}

$k$: Table~\ref{tab:k} demonstrates the variation in performance versus the dimensionality of the learned vector representations of the words. Since the dimensions of the MVLSA representations are orthogonal to each other therefore creating lower dimensional representations is a trivial matrix slicing operation and does not require retraining.
  \begin{table}[htbp]
    \centering
  \begin{tabular}{=l | +c H +c +c +c +c +c}
Test Set                            & 10   & 25   & 50   & 100  & 200       & 300       & 500       \\\hline
MEN                                 & 49.0 & 59.3 & 67.0 & \mb{69.7} & \mb{70.2} & \mi{70.1} & \mb{69.8}\\
RW                                  & 28.8 & 33.1 & 33.3 & 35.0 & 35.2      & \mb{37.2} & \mi{38.3} \\
SCWS                                & 57.8 & 62.8 & 64.4 & \mi{65.2} & \mi{66.1}      & \mb{66.4} & \mi{65.1}      \\%\remove{
SIMLEX                              & 24.0 & 30.1 & 33.9 & 36.1 & 38.9      & 41.1      & \mb{42.0} \\
\rowstyle{\color{darkergray}}WS     & 46.8 & 57.5 & 63.4 & 69.5 & 69.5      & 69.4      & 66.0      \\
\rowstyle{\color{darkergray}}MTURK  & 54.6 & 65.9 & 67.7 & 61.6 & 60.5      & 58.4      & 57.4      \\
\rowstyle{\color{darkergray}}WS-REL & 38.4 & 49.5 & 55.8 & 63.1 & 62.4      & 61.6      & 56.3      \\
\rowstyle{\color{darkergray}}WS-SEM & 55.3 & 64.7 & 69.9 & 76.9 & 77.1      & 76.8      & 75.6      \\
\rowstyle{\color{darkergray}}RG     & 48.8 & 60.5 & 66.1 & 69.7 & 75.1      & 73.2      & 72.5      \\
\rowstyle{\color{darkergray}}MC     & 37.0 & 57.5 & 59.0 & 71.3 & 79.1      & 78.3      & 75.7      \\%}
AN-SYN                               & 9.0  & 28.4 & 41.2 & 52.2 & 55.4      & \mb{56.4} & 54.4      \\
AN-SEM                               & 2.5  & 10.8 & 21.8 & 34.8 & \mb{35.8} & 34.3      & 33.8      \\%\remove{
\rowstyle{\color{darkergray}} TOEFL & 57.5 & 73.8 & 72.5 & 76.2 & 81.2      & 82.5      & 85.0%}
  \end{tabular}
  \caption[Performance versus MVLSA embedding dimensionality.]{Performance versus $k$, the final dimensionality of the embeddings. I set $ m=300$ and other settings were same as Table~\ref{tab:m}.}
  \label{tab:k}
\end{table}

$v$: Expression~\ref{eq:gcca3} describes a method to set $W_j$. I experimented with a different, more global, heuristic to set $[W_j]_{ii} = (K_{ww} \ge v)$, essentially removing all words that did not appear in $v$ views before doing GCCA. Table~\ref{tab:v} shows that changes in $v$ are largely inconsequential for performance. {In absence of clear evidence in favor of regularization I decided to regularize as little as possible and chose $v=16$.}
  \begin{table}[htbp]
    \centering
  \begin{tabular}{=l | +c +c H +c H +c H +c}
Test Set                            & 16   & 17   & 19   & 21   & 23   & 25   & 27   & 29   \\ \hline
MEN                                 & \mb{70.4} & \mb{70.4} & \mi{70.2} & \mi{70.2} & \mi{70.1} & \mi{70.1} & \mi{70.0} & \mi{70.0} \\
RW                                  & \mb{39.9} & \mi{38.8} & \mi{40.1} & \mi{39.7} & 38.3 & 37.2 & 35.3 & 33.5 \\
SCWS                                & \mb{67.0} & \mb{66.8} & \mb{66.8} & \mb{66.5} & \mb{66.3} & \mb{66.4} & \mb{66.1} & \mb{65.7} \\%\remove{
SIMLEX                              & 40.7 & 41.0 & 41.1 & \mb{41.2} & 41.2 & 41.1 & 41.1 & 41.0 \\
\rowstyle{\color{darkergray}}WS     & 69.5 & 69.4 & 69.5 & 69.5 & 69.4 & 69.4 & 69.3 & 69.1 \\
\rowstyle{\color{darkergray}}MTURK  & 59.4 & 59.2 & 59.3 & 59.2 & 58.7 & 58.4 & 58.0 & 58.0 \\
\rowstyle{\color{darkergray}}WS-REL & 62.1 & 61.9 & 62.1 & 62.3 & 61.9 & 61.6 & 61.4 & 61.1 \\
\rowstyle{\color{darkergray}}WS-SEM & 76.8 & 76.8 & 76.9 & 77.0 & 76.7 & 76.8 & 76.7 & 76.8 \\
\rowstyle{\color{darkergray}}RG     & 73.0 & 72.8 & 72.7 & 72.8 & 73.6 & 73.2 & 73.4 & 73.7 \\
\rowstyle{\color{darkergray}}MC     & 75.0 & 76.0 & 76.4 & 76.5 & 78.2 & 78.3 & 78.6 & 78.6 \\%}
AN-SYN                               & \mb{56.0} & \mb{55.8} & \mb{56.0} & \mb{55.9} & \mb{56.3} & \mb{56.4} & \mb{56.3} & \mb{56.0} \\
AN-SEM                               & \mb{34.6} & \mb{34.3} & \mb{34.1} & \mb{34.0} & \mb{34.5} & \mb{34.3} & \mb{34.4} & \mb{34.3} \\%\remove{
\rowstyle{\color{darkergray}} TOEFL & 85.0 & 85.0 & 85.0 & 83.8 & 83.8 & 82.5 & 82.5 & 80.0%}
    \end{tabular}
  \caption[Performance versus the support threshold.]{Performance versus minimum view support threshold $v$, The other hyperparameters were $n_j=\textrm{Count}^{\frac{1}{4}}, \; m=300, \; t=100K$. Though a clear best setting did not emerge, I chose $v=25$ as the middle ground.}
  \label{tab:v}
\end{table}

$r_j$: The regularization parameter ensures that all the inverses exist at all points in my method. I found that the performance of my procedure was invariant to $r$ over a broad range from 1 to 1e-10. This was because even the 1000th singular value of my data was much higher than 1.

\noindent\textbf{Contribution of different sources of data} Table~\ref{tab:jkjk} shows an ablative analysis of performance where I remove individual views or some combination of them and measure the performance.  It is clear by comparing the last column to the second column that adding in more views improves performance. Also I can see that the Dependency based views and the Bitext based views give a larger boost than the morphology and FrameNet based views, probably because the latter are so sparse.
 \begin{table*}[ht]
  \centering
  \begin{tabular}{l|     p{1cm}         p{1cm}       p{1.3cm} p{1cm} p{1cm} p{1.3cm} p{2cm} p{2cm}}
Test Set              & {All Views} & !Fram- enet & !Morph- ology & !Bitext & !Wiki- pedia & !Depen- dency & {!Morphology !Framenet} & {!Morphology !Framenet !Bitext} \\\hline
MEN                                 & \mb{70.1} & \mi{69.8} & \mi{70.1} & \mi{69.9} & 46.4 & 68.4 & \mi{69.5} & 68.4 \\
RW                                  & \mb{37.2} & \mi{36.4} & \mi{36.1} & 32.2 & 11.6 & 34.9 & 34.1 & 27.1 \\
SCWS                                & \mb{66.4} & \mi{65.8} & \mi{66.3} & 64.2 & 54.5 & \mi{65.5} & \mi{65.2} & 60.8 \\%\remove{
SIMLEX                              & 41.1 & 40.1 & 41.1 & 37.8 & 32.4 & \mb{44.1} & 38.9 & 34.4 \\
\rowstyle{\color{darkergray}}WS     & 69.4 & 69.1 & 69.2 & 67.6 & 43.1 & 70.5 & 69.3 & 66.6 \\
\rowstyle{\color{darkergray}}MTURK  & 58.4 & 58.3 & 58.6 & 55.9 & 52.7 & 59.8 & 57.9 & 55.3 \\
\rowstyle{\color{darkergray}}WS-REL & 61.6 & 61.5 & 61.4 & 59.4 & 38.2 & 63.5 & 62.5 & 58.8 \\
\rowstyle{\color{darkergray}}WS-SEM & 76.8 & 76.3 & 76.7 & 75.9 & 48.1 & 75.7 & 75.8 & 73.1 \\
\rowstyle{\color{darkergray}}RG     & 73.2 & 72.0 & 73.2 & 73.7 & 45.0 & 70.8 & 71.9 & 74.0 \\
\rowstyle{\color{darkergray}}MC     & 78.3 & 75.7 & 78.2 & 78.2 & 46.5 & 77.5 & 76.0 & 80.2 \\%}
AN-SYN                               & \mb{56.4} & \mi{56.3} & \mi{56.2} & 51.2 & 37.6 & 50.5 & 54.4 & 46.0 \\
AN-SEM                               & 34.3 & 34.3 & 34.3 & \mb{36.2} & 4.1  & 35.3 & 34.5 & 30.6 \\%\remove{
\rowstyle{\color{darkergray}}TOEFL  & 82.5 & 82.5 & 82.5 & 71.2 & 45.0 & 85.0 & 82.5 & 65.0   %}
  \end{tabular}
  \caption[Ablation tests for MVLSA]{Performance versus views removed from the multiview GCCA procedure. !Framenet means that the view containing counts derived from Frame semantic dataset was removed. Other columns are named similarly. The other hyperparameters were $n_j=\textrm{Count}^{\frac{1}{4}}, \; m=300, \; t=100K, \; v=25, \; k=300$. }
  \label{tab:jkjk}
\end{table*}
\noindent\textbf{Comparison to other word representation creation methods} There are a large number of methods of creating representations both multilingual and monolingual. There are many new methods such as by \newcite{yu2014improving}, \newcite{faruqui2014retrofitting}, \newcite{felix2014learning}, and \newcite{weston2014hash} that are performing multiview learning and could be considered here as baselines: however it is not straight-forward to use those systems to handle the variety of data that I am using. Therefore, I directly compare my method to the Glove and the SkipGram model of Word2Vec as the performance of those systems is considered state of the art.  I trained these two systems on the English portion of the \textit{Polyglot} Wikipedia dataset.\footnote{I explicitly provided the vocabulary file to Glove and Word2Vec and set the truncation threshold for Word2Vec to 10.  Glove was trained for 25 iterations. Glove was provided a window of 15 previous words, and Word2Vec used a symmetric window of 10 words.} I also combined their outputs using MVLSA to create \emph{MV-G-WSG}embeddings.

I trained my best MVLSA system with data from all views and by using the individual best settings of the hyper-parameters. Specifically the configuration I used was as follows: $n_j = \text{Count}^\frac{1}{4}, t=12.5K, m=500, k=300, v=16$. To make a fair comparison, I also provide results where I used only the views derived from the \textit{Polyglot} Wikipedia corpus. See column \emph{MVLSA (All Views)} and \emph{MVLSA (Wiki)} respectively. It is visible that MVLSA on the monolingual data itself is competitive with Glove but worse than Word2Vec on the word similarity datasets and it is substantially worse than both the systems on the AN-SYN and AN-SEM datasets. However with the addition of multiple views, MVLSA makes substantial gains, shown in column \emph{MV Gain}, and after consuming the Glove and WSG embeddings, it again improves performance by some margins, as shown in column \emph{G-WSG Gain}, and outperforms the original systems.  Using GCCA itself for system combination provides closure for the MVLSA algorithm since multiple distinct approaches can now be simply fused using this method. Finally, I contrast the Spearman correlations $r_s$ with Glove and Word2Vec before and after including them in the GCCA procedure. The values demonstrate that including Glove and WSG during GCCA increased the correlation between them and the learned embeddings, which supports my motivation for performing GCCA in the first place.

\begin{table*}[ht]
  \begin{adjustwidth}{-1cm}{}
      %\rowcolors{1}{}{lightgray}  \mcinherit
  \setlength\tabcolsep{2.2pt}{
  \begin{tabular}{+l |                 +d{2.1}     +d{2.1} |    +d{2.1}     +d{2.1}     +d{2.1}       +d{2.1}       |   +d{2.1}       +d{2.1}|| +c                                  +c |     +c            +c }
    \mm{Test Set}                  & \m{Glove} & \mm{WSG}   &   \m{MV}          & \m{MVLSA} & \m{MVLSA}     & \mm{MVLSA} & \m{MV}    & \mmm{G-WSG} & \multicolumn{2}{c|}{$r_s$ MVLSA} & \multicolumn{2}{c}{$r_s$ MV-G-WSG} \\
                \mm{}              & \m{}      & \mm{}      & \m{G-WSG} & \m{Wiki } & \m{All Views} & \mm{Combined} & \m{Gain}  & \mmm{Gain } & \m{Glove}                        & \mm{WSG} & \m{Glove} & \m{WSG}     \\\hline
MEN                                & 70.4      & 73.9      & \my{76.0} & 71.4      & 71.2          & \myy{75.8}    & -0.2      & \ma{4.6}   & 71.9                             & 89.1     & 85.8      & 92.3        \\
RW                                 & 28.1      & 32.9      & 37.2       & 29.0      & \my{41.7}     & \myy{40.5}    & \ma{12.7} & -1.2       & 72.3                             & 74.2     & 80.2      & 75.6        \\
SCWS                               & 54.1      & 65.6      & 60.7       & 61.8      & \my{67.3}     & \myy{66.4}    & \ma{5.5}  & -0.9       & 87.1                             & 94.5     & 91.3      & 96.3        \\
SIMLEX                             & 33.7      & 36.7      & 41.1       & 34.5      & \my{42.4}      & \myy{43.9}    & \ma{7.9}  & 1.5        & 62.4                             & 78.2     & 79.3      & 86.0        \\
WS                                 & 58.6      & \myy{70.8} & \my{67.4} & \my{68.0}  & \my{70.8}     & \myy{70.1}    & \ma{2.8}  & -0.7       & 72.3                             & 88.1     & 81.8      & 91.8        \\
MTURK                              & \my{61.7}  & \myy{65.1} & 59.8       & 59.1      & 59.7          & \myy{62.9}    & 0.6       & 3.2        & 80.0                             & 87.7     & 87.3      & 92.5        \\
WS-REL                             & 53.4      & \myy{63.6} & 59.6       & 60.1      & \my{65.1}      & \myy{63.5}    & \ma{5.0}  & -1.6       & 58.2                             & 81.0     & 69.6      & 85.3        \\
WS-SEM                             & 69.0      & \myy{78.4}  & \my{76.1} & \my{76.8}  & \my{78.8}      & \myy{79.2}    & 2.0       & 0.4        & 74.4                             & 90.6     & 83.9      & 94.0        \\
\rowstyle{\color{darkergray}}RG    & \my{73.8}  & \myy{78.2}  & \my{80.4} & 71.2      & \my{74.4}      & \myy{80.8}    & 3.2       & \ma{6.4}   & 80.3                             & 90.6     & 91.8      & 92.9        \\
\rowstyle{\color{darkergray}}MC    & \my{70.5}  & \myy{78.5}  & \my{82.7} & \my{76.6}  & \my{75.9}      & \myy{77.7}    & -0.7      & 2.8        & 80.1                             & 94.1     & 91.4      & 95.8        \\
AN-SYN                              & 61.8      & 59.8      & 51.0       & 42.7      & 60.0          & \myy{64.3}    & \ma{17.3} & \ma{4.3}   &                                  &          &           &             \\
AN-SEM                              & \my{80.9} & 73.7      & 73.5       & 36.2      & 38.6          & 77.2          & \ma{2.4}  & \ma{38.6}  &                                  &          &           &             \\
\rowstyle{\color{darkergray}}TOEFL & \my{83.8}  & 81.2      & \my{86.2} & 78.8      & \my{87.5}      & \myy{88.8}    & \ma{8.7}  & 1.3        &                                  &          &           &
  \end{tabular}}
  \end{adjustwidth}
  \caption[Comparison of MVLSA to Glove and Word2Vec.]{Comparison of Multiview LSA against Glove and WSG(Word2Vec Skip Gram). Using $\sigma_{0.05}^{0.9}$ as the threshold I highlighted the top performing systems in bold font. $^\dagger$ marks significant increments in performance due to use of multiple views in the \emph{Gain} columns. The $r_s$ columns demonstrate that GCCA increased Pearson correlation.\remove{I also considered the best performance of Glove and Word2Vec to be my baseline and marked systems that were significantly higher than baseline with $^*$ and systems that were significantly lower than baseline by $^\downarrow$.}  }
  \label{tab:c}
\end{table*}

\section{Previous Work}
\label{sec:previouswork}
Vector space representations of words have been created using diverse frameworks including Spectral methods \cite{dhillon2011multi,dhillon2012two}, \footnote{\url{cis.upenn.edu/~ungar/eigenwords}} Neural Networks \cite{mikolov2013linguistic,collobert2013word}, and Random Projections \cite{ravichandranACL05,bhagatACL08,chanGEMS2011}. \footnote{\url{code.google.com/p/word2vec},\url{metaoptimize.com/projects/wordreprs}} They have been trained using either one \cite{pennington2014glove} \footnote{\url{nlp.stanford.edu/projects/glove}} or two sources of cooccurrence statistics \cite{zou2013bilingual,faruqui2014improving,bansal2014tailoring,levy2014dependency} \footnote{\url{ttic.uchicago.edu/~mbansal/data/syntacticEmbeddings.zip,cs.cmu.edu/~mfaruqui/soft.html}} or using multi-modal data \cite{felix2014learning,bruni2012distributional}.

\newcite{dhillon2011multi} and \newcite{dhillon2012two} were the first to use CCA as the primary method to learn vector representations and \newcite{faruqui2014improving} further demonstrated that incorporating bilingual data through CCA improved performance. More recently this same phenomenon was reported by \newcite{hill2014not} through their experiments over neural representations learned from MT systems. Outside of the NLP community \cite{sun2013generalized,tripathi2011data} are examples of works that have used GCCA for ``data fusion''. Various other researchers have tried to improve the performance of their paraphrase systems or vector space models by using diverse sources of information such as bilingual corpora~\cite{bannard2005paraphrasing,Huang2012Improving,zou2013bilingual},\footnote{An example of complementary views: \newcite{chanGEMS2011} observed that monolingual distributional statistics are susceptible to conflating antonyms, where bilingual data is not; on the other hand, bilingual statistics are susceptible to noisy alignments, where monolingual data is not.}  structured datasets~\cite{yu2014improving,faruqui2014retrofitting} or even tagged images~\cite{bruni2012distributional}. However, most previous work\footnote{\newcite{ganitkevitch2013ppdb} did employ a rich set of diverse cooccurrence statistics in constructing the initial PPDB, but without a notion of ``training'' a joint representation beyond random projection to a binary vector subspace (bit-signatures).} did not adopt the general, simplifying view that all of these sources of data are just cooccurrence statistics coming from different sources with underlying latent factors.\footnote{Note that while \newcite{faruqui2014retrofitting} performed belief propagation over a graph representation of their data, such an undirected weighted graph can be viewed as an adjacency matrix, which is then also a co-occurrence matrix.}
%% Also
%%   their framework could not fuse arbitrary views such as other vector
%%   representations.}

\newcite{bach2005probabilistic} presented a probabilistic interpretation for CCA. Though they did not generalize it to include GCCA, I believe that one could give a probabilistic interpretation of \emph{MAX-VAR GCCA}.  Such a probabilistic interpretation would allow for an online-generative model of lexical representations, which unlike methods like Glove or LSA would allow us to naturally perplexity or generate sequences. I also note that \newcite{via2007learning} presented a neural network model of GCCA and adaptive/incremental GCCA.  To the best of my knowledge, both of these approaches have not been used for word representation learning.

CCA is also an algorithm for multi-view learning \cite{kakade2007multi,ganchevuai08} and when I view my work as an application of multiview learning to NLP, this follows a long chain of effort started by \newcite{yarowsky1995unsupervised} and continued with \emph{Co-Training}~\cite{blum1998combining}, \emph{CoBoosting}~\cite{collins1999unsupervised} and \emph{2 view perceptrons}~\cite{brefeld2006efficient}.

\section{Conclusion}
This chapter is based on the following published paper:
\begin{itemize}
\item [] Rastogi, Pushpendre, Benjamin Van Durme, and Raman Arora (2015). ``Multi- view LSA: Representation Learning Via Generalized CCA''. In: Proceedings of NAACL.
\end{itemize}

\textbf{The main ideas and scientific contribution of this chapter are}:

\begin{itemize}
\item The first to construct word embeddings from massively multi-view datasets.
\item The first algorithm for scaling Generalized CCA to large datasets via a novel approximation technique.
\item A new procedure, \textit{MRDS}, for measuring the significance of results, based only on the spearman-correlation values and dataset size.
\end{itemize}

While previous efforts demonstrated that incorporating two views is beneficial in word-representation learning, I extended that thread of work to a logical extreme and created \emph{MVLSA} to learn distributed representations using data from 46 views!\footnote{Code and data available at \url{www.cs.jhu.edu/~prastog3/mvlsa}} Through evaluation of my induced representations, shown in Table~\ref{tab:c}, I demonstrated that the MVLSA algorithm could leverage the information present in multiple data sources to improve performance on a battery of tests against state of the art baselines. To perform MVLSA on large vocabularies with up to 500K words, I presented a fast, scalable algorithm. I also showed that a close variant of the Glove objective proposed by \newcite{pennington2014glove} could be derived as a heuristic for handling missing data under the MVLSA framework. To better understand the benefit of using multiple sources of data, I performed MVLSA using views derived only from the monolingual Wikipedia dataset thereby providing a more principled alternative of LSA that removes the need for heuristically combining word-word cooccurrence matrices into a single matrix. Finally, while surveying the literature I noticed that not enough emphasis was being given towards establishing the significance of comparative results and proposed a method, \emph{(MRDS)}, to filter out insignificant comparative gains between competing algorithms.

%% TODO
%% MVLSA - newer method - why it's important to have these. Different datasets give different
%%    Why is multi-view approach a better approach than just the

\noindent\textbf{Future Work} Column \emph{MVLSA Wiki} of
Table~\ref{tab:c} shows us that MVLSA applied to monolingual data has
 mediocre performance  compared to the baselines of Glove and
Word2Vec on word similarity tasks and performs surprisingly worse on
the AN-SEM dataset. I believe that the results could be improved by (1)
either using recent methods for handling missing values
mentioned in footnote~\ref{ftn:mis} or by using the heuristic count dependent
non-linear weighting mentioned by \newcite{pennington2014glove}
and that sits well within my framework as exemplified in Expression~\ref{eq:gcca3}
 (2) by using even more views, which
look at the future words as well as views that contain PMI values.
Finally, I note that Table~\ref{tab:jkjk} shows that certain datasets can
actually degrade performance over certain metrics. Therefore I am
exploring methods for performing discriminative optimization of weights
assigned to views, for purposes of task-based customization of learned representations.

%auto-ignore
\chapter{Neural Variational Set Expansion}
\label{cha:nvse}
People use words and sentences to communicate with each other about real-world entities. In the previous chapter, I presented a ``shallow'' algorithm MVLSA for learning representations of words. In this chapter, I go deeper and present a novel ``deep'' representation learning method for learning representations of entities grounded in natural language text.\footnote{A previous version of this work was published in \cite{rastogi2018nvse}.} For this chapter an entity is a set of mentions across multiple documents that refer to the same real-world object. Distributed representations of such mention-sets can aid Information extraction and retrieval systems. To that end, I focus on the task of \textit{Entity Recommendation}. Many existing information retrieval systems that operate on entities rely on clean, manually curated sets of entities for their operation. Because users often work with unclean, automatically generated KGs and require interpretable tools; therefore, they are often unable to incorporate such algorithms in their workflow fully. I propose \nvgeLong to extract actionable information from a noisy knowledge graph (KG) grounded in natural language and also propose a general approach for increasing the interpretability of recommendation systems.

Akin to prior entity-focused retrieval definitions, a query consists of one or more entities, with the intent of retrieving similar entities.  Differing from prior work, I focus neither on manually curated knowledge bases, nor collections of entity-labeled documents such as Wikipedia. I demonstrate the usefulness of applying a variational autoencoder to the \erLong task based on a realistic automatically generated KG.  Further, I describe an approach for \er, \nvgeLong, which supports humanly interpretable query rationales, and outperforms baselines such as Bayesian Sets and BM25.

\section{Introduction}
Imagine a physician trying to pinpoint a specific diagnosis or a security analyst attempting to uncover a terrorist network. In both scenarios, a \textit{domain expert} may try to find answers based on prior known, relevant entities -- either a list of diagnoses of with similar symptoms that a patient is experiencing or a list of known conspirators. Instead of manually looking for connections between potential answers and prior knowledge, a \textit{searcher} would like to rely on an automatic \textit{Recommender} to find the connections and answers for them, i.e., related entities.

In the information retrieval (IR) community, Entity Set Expansion (ESE) is the established task of recommending entities that are similar to a provided seed of entities.\footnote{I refer to the items in the seed as entities, but they can also be referred to as items or elements} ESE has been applied in Question Answering~\mycite{wang-EtAl:2008:EMNLP2}, Relation Extraction~\mycite{lang-henderson:2013:NAACL-HLT} and Information Extraction~\mycite{hegrishman2015demos} settings. The physician and journalist in my example cannot fully take advantage of IR advances in ESE for two main reasons. Recent advances 1) often assume access to a clean, large Knowledge Graph and 2) are uninterpretable.

Many advanced ESE algorithms rely on manually curated, clean Knowledge Graphs (KG), e.g. DBpedia \mycite{auer2007dbpedia} and Freebase \mycite{bollacker2008freebase}. In clean KGs duplicate entities are merged, entities rarely are isolated, and entities with similar names are properly disambiguated. However, in real-world settings, users do not always have access to clean KGs, and instead, they may rely on automatically generated KGs. Such KGs are often \textit{noisy} because they are created from complicated and error-prone NLP processes -- illustrated in \figref{fig:ff1}. For example, automatic KGs may include duplicate entities, associations (relations) between entities may be missing, and entities with similar names may be incorrectly disambiguated.
Similarly, faulty coreference or entity linking may fail to merge duplicate entities, may create many isolated entities, and may poorly disambiguate entities with similar names. These imperfections prevent machine learning approaches from performing well on automatically generated KGs. Furthermore, many ESE algorithm's performance degrades as the sparsity and unreliability of KGs increases~\mycite{pujara2017sparsity,rastogi2017vertex}. Therefore, in practice, users working with large KGs even now only rely on weighted boolean and keyword searches \mycite{jin2005query,gadepally2016recommender} instead of advanced KG completion algorithms

Advanced ESE methods, especially those that rely on neural networks, are uninterpretable \mycite{mitra2017neural}. If a physician can not explain decisions, patients may not trust her, and if a journalist can not demonstrate how a certain individual is acting unethically or above the law, a resulting article may lack credibility. Furthermore, uniterpretability may limit the applications of advancements in IR, and more broadly artificial intelligence, as humans ``won't trust an A.I. unless it can explain itself.''\footnote{\url{https://nyti.ms/2hR1S15}}

\begin{figure*}[t]
  \centering
  \includegraphics[width=\linewidth,trim=0 0 0 0,clip]{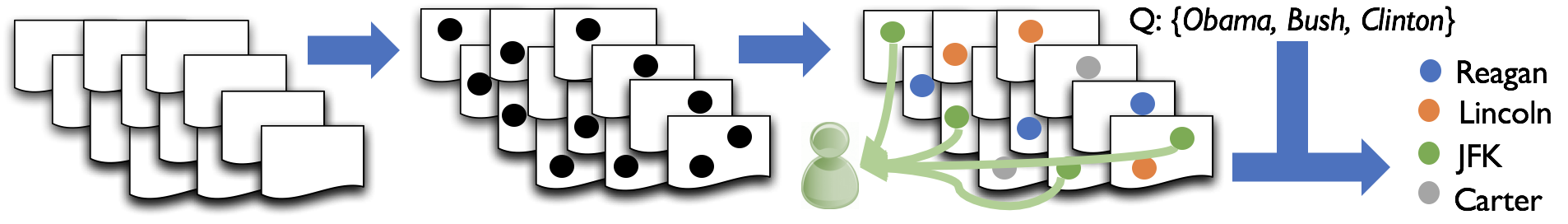}
  \caption[Typical NLP pipeline for Named Entity Recognition and Linking.]{ %\erLong.
    My \emph{\erLong} (\er) system assumes a corpus that has been labeled with entity mentions which are clustered via cross-document co-reference and linking to a knowledge base; together known as \emph{entity discovery and linking} (EDL). Given a query containing \textit{Obama}, \textit{Bush}, and \textit{Clinton}, the \er system returns other U.S. presidents found in the KG.}
\label{fig:ff1}
\end{figure*}

I introduce \nvgeLong (\nvge) to advance the applicability of ESE research. \nvge is an unsupervised model based on Variational Autoencoders (VAEs) that receives a query, consisting of a small set of entities and uses a Bayesian approach to determine a latent concept that unifies entities in the query, and returns a ranked list of similar entities based on the previously determined unified latent concept. I refer to my method as \nvgeLong since \nvge uses a VAE to model the latent concept as a Gaussian distributed random variable for the task of Set Expansion.

\nvge does not require supervised examples of queries and responses, nor a manually curated KG. It also does not require nor pre-built clusters of entities. Instead, my method
 only requires sentences with linked entity mentions, i.e., spans of tokens associated with a KG entity, often included in automatically generated KGs.

\nvge is robust to noisy automatically generated KGs, thus removing the need to rely on manually curated, clean KGs. I evaluate \nvge on the \er task using Tinkerbell \mycite{tinkerbell2017Tac}, an automatically generated KG that placed first in the TAC KGP shared-task. Unlike how ESE has been used to improve entity linking for KG construction~\mycite{gottipati2011linking}, my goal is the opposite: I leverage noisy automatically generated KGs to perform ESE. \nvge is interpretable; it outputs \textbf{query rationales} -- a summarization of features the model associated with the query -- and \textbf{result justifications} -- an ordered list of sentences from the underlying corpus that justify why my method returned that entity. Query rationales and result justifications are reminiscent of \textit{annotator rationales}~\mycite{zaidan-eisner-piatko-2008:nips}.

To my knowledge, this is the first unsupervised neural approach for \er as opposed to neural methods for supervised collaborative filtering~\mycite{lee2017augmented}. All code and data is available at \url{github.com/se4u/nvse} and a video demonstration of the system is available at \url{youtu.be/sGO_wvuPIzM}.

\section{Related Work}
{\bf Methods dependent on external information.} Since automatically generated KGs can be noisy, some methods utilize information beyond entity links and mentions to aid ESE. \mynewcite{pasca2007you} use search engine query logs to extract attributes related to entities and \mynewcite{pacsca2008weakly} extract sets of instances associated with class labels based on web documents and queries. \mynewcite{pantel-EtAl:2009:EMNLP} use a large amount of web data as they apply a learned word similarity matrix extracted from a 200 billion word Internet crawl to the \er task. Both \mynewcite{he2011seisa}'s SEISA system and \mynewcite{tong2008system}'s Google Sets use lists of items from the Internet and try to determine which elements in the lists are most relevant to a query. \mynewcite{sadamitsu-EtAl:2011:ACL-HLT2011} rely on given topic information about the queried entities to train a discriminative system. More recent approaches also use external information. \mynewcite{manzil2017deep} use LDA \mycite{blei2003latent} to create word clusters for supervision, and \mynewcite{vartak2017meta} use manual annotations by Twitter users. \mynewcite{zheng2017entity} uses inter-entity links in knowledge graphs which are very sparse in automatically generated KGs~\mycite{pujara2017sparsity,rastogi2017vertex}. All of these approaches use more information than just entity links and mentions. % their mentions.

{\bf Methods for comparing entities.} Set Expander for Any Language (SEAL) \mycite{wang2007language-independent} and its variants \mycite{wang2008iterative,wang2009automatic} learn similarities between new words and example words using methods like Random Walks and Random Walks With Restart. Similar to \mynewcite{lin1998automatic}'s using cosine and Jaccard similarity to find similar words, SEISA uses these metrics to expand sets. These methods are limited to only extracting words that co-occur. Because they are applied to web-scale data, SEAL and SEISA assume entities will eventually co-occur. This assumption might not be valid in an underlying corpus used to generate a KG automatically. In contrast to those approaches, \nvge finds similar entities based on a kernel between distributions.

{\bf Queries as natural language.} In the INEX-XER shared-task, queries were represented as natural language questions \mycite{demartini2009overview}. \mynewcite{metzger2014aspect} and \mynewcite{zhang2017entity} propose methods to extract related entities in a KG based on a natural language query. This scenario is similar to a person interacting with a system like Amazon Alexa. However, my setup better reflects users searching for similar entities in a KG as it is more efficient for users to type entities of interest instead of natural language text.

{\bf Neural Collaborative Filtering.} I am not the first to incorporate neural methods in a recommendation system. Recently, \mynewcite{He:2017:NCF:3038912.3052569} and \mynewcite{lee2017augmented} presented deep auto-encoders for collaborative filtering. Collaborative Filtering assumes a large dataset of previous user interactions with the search engine. For many domains, it is not possible to create such a dataset since new data is added every day and queries change rapidly based on different users and domains. Therefore, I propose the first neural method which does not use supervision for \erLong. \mynewcite{li2017collaborative} use a citation dataset and their recommendations only include users with less than ten articles. They only gave recommendations for entities that appeared in at least $10$ articles in the corpus.

%% TODO: Basically clustering features
%% LDA models over entities - Context and entity mentions - what is
%% the advantage of what I am trying to do.
{\bf Unsupervised Clustering for Entity Resolution} \cite{kugatsu2011entity} proposed to learn the latent topics of documents for alleviating problems of ``semantic drift'' in Entity Set Expansion. Semantic drift refers to the common problem  faced by entity set expansion algorithms of changes in the extraction criteria. In order to combat this problem they modeled the latent topics with LDA (Latent Dirichlet Allocation)~\cite{blei2003latent} and utilized the topic information in three ways:
\begin{itemize}
\item First they used the topic distribution of documents to generate features for their set expansion system.
\item Second, they selected negative examples for training a discriminative system.
\item And third they pruned certain examples in their iterative training method.
\end{itemize}
Instead of creating a pipelined approach using a pre-existing topic model, our approach allows us to create a topic model that can be trained end-to-end and which is directly amenable to learning non-linear features of the data. A similarity between the two methods is that variational inference can be used for learning the parameters of both the models.

\section{Notation}
\label{sec:notation}
Let $\cDoc$ be the corpus of documents and $\cV$ be the vocabulary of tokens that appear in $\cDoc$.  I define a document as a sequence of sentences and  a sentence as a sequence of tokens. Let $\cX$ be the set of entities discovered in $\cDoc$ and I refer to its size as $\rX$. Each entity $x \in \cX$ is linked to the tokens that mention $x$.\footnote{I ignore confidence scores that entity linking systems often assign to a link because confidence scores will prevent us from using a multinomial distribution to model a document as a bag-of-words.} Let $\cVp$ be the set of tokens linked to any $x \in \cX$, and let $\cM_x$ be the multiset of sentences that mention $x$ in the corpus. For example, consider an entity named ``Batman'' and a document containing three sentences \{\underline{Batman} is good., \underline{He} is smart. Life is good.\}. ``Batman'' is linked to tokens \underline{Batman} and \underline{He}, %in the first and second sentence respectively, $\cVp = \{\text{Batman}, \text{He}\}$, and $\cM_{\text{Batman}} = $  \{\underline{Batman} is good., \underline{He} is smart.\}.

In \er, a system receives query $\cQ$ -- a subset of $\cX$ -- and has to sort the elements remaining in $\cR = \XmQ$. The elements that are most similar to $\cQ$ should appear higher in the sorted order and elements dissimilar to $\cQ$ should be ranked lower.

\section{Baseline Methods}
Before introducing \nvge, I describe the four baselines systems: BM25, Bayesian Sets, \wTv, and \setX. I do not compare to DeepSets~\mycite{manzil2017deep}, as it is a supervised method that requires entity clusters.

For each $x$, I create a feature vector $f_x \in \mathbb{Z}^{\rF}$ from $\cM_x$, by concatenating three  vectors that count how many times 1) a token in $\rV$ appeared in $\cM_x$ 2) a document in $\cDoc$ mentioned $x$ and 3) a token in $\cVp$ appeared in $\cM_x$. Thus, $\rF = \rV + \rDoc + \rVp$.

\subsection{BM25}\label{ssec:bm25}
Best Match 25 (BM25) is ``one of the most successful text-retrieval algorithms'' \mycite{robertson2009prf}. \footnote{Lucene replaced tf-idf with BM25 as its default algorithm: \url{https://issues.apache.org/jira/browse/LUCENE-6789}} BM25 ranks remaining entities in $\cR$ according to the score function
\begin{equation*}
\underset{BM}\score(\cQ, x) = {\sum_{i=1}^\rF} {\frac{\IDF[i] f_x[i] \fcq[i] (k_1 + 1)  }{f_x[i] {+} k_{1} (1{-}b {+} b \fracil{\sum_{j} f_x[j]}{\avgdl})}},
%.
\end{equation*}
where $f_{x}[j]$ denotes the $j$-th feature value in $f_{x}$, $\fcq$ is the sum of $f_x \forall x \in \cQ$ and $\bI$ is the indicator function. $k_1$ and $b$ are hyperparameters that commonly set to $1.5$ and $0.75$ \mycite{introductionir}. $\avgdl$ is the average total count of a feature in the entire corpus and $\IDF[i]$ is the inverse document frequency of the $\ath{i}$ feature. They are computed as,
\begin{align*}
  \avgdl =& \fracil{\sum_{x \in \cX} \sum_j f_x[j] }{\rX}\\
\IDF[i]= &\log \frac{\rX - \DF[i] + 0.5}{\DF[i] + 0.5} \\
\DF[i] = &\sum_{x \in \cX} \bI[f_x[i] > 0].
\end{align*}

\subsection{Bayesian Sets}
\mynewcite{ghahramani2006bayesian} introduced the Bayesian Sets (BS) method which converts ESE into a bayesian model selection problem. BS compares the probabilities that the query entities are generated from a single sample of a latent variable $z \in \Delta^{\rF}$ with the probability that the entities were generated from independent samples. $\Delta^{\rF}$ is the $\rF-1$ dimensional probability simplex. Note that $z$ has the same dimensionality as the observed features. Given $\cQ$ and $\pi$, the prior distribution of $z$, BS infers the posterior distribution of $z$, $p(z | \cQ)$, and computes the following score
\begin{equation}
 \underset{BS}\score(\cQ, x) = %\log \frac{p(x|\cD)}{p(x)} =
 \log \frac{E_{p(z | \cQ)}[p(x|z)]}{E_{\pi(z)}[p(x|z)]}.
\label{eq:bs}
\end{equation}
\myciteauthor{ghahramani2006bayesian} computed $\score_{BS}$ in close form by selecting the conditional probability, $p(x|z)$, from an exponential family distribution and setting $\pi$ to be its conjugate prior. They showed that if $p(x|z)$ is multivariate Bernoulli then BS requires a single matrix multiplication \appref{sec:bs} and I use this setting for my experiments.% \footnote{I also implement Poisson-Gamma and Gaussian-Wishart variants of BS but they did not perform better.}
% They binarize the count vectors $f_{x}$ as described in ~\appref{app:bs-binarize}

\subsection{\wTv}
\mynewcite{levy2014dependency} generalize \mynewcite{mikolov2013distributed}'s Skip-Gram model as \wTv to include arbitrary contexts. I embed entities with \wTv by using the entity IDs as words \footnote{Converting entity mentions to entity IDs allows us to overcome issues related to embedding multi-word expressions as explained in \mynewcite{poliak-EtAl:2017:EACLshort}.} and the tokens in the sentences mentioning those entities as contexts. Note that all tokens in the sentence, except for some stop words, are used as contexts and not just co-occurrent entities. I rank the entities in the order of their total distance to the entities in the query set as
\begin{equation}\label{eq:wtv}
  \underset{W2V}\score(\cQ, x) = - \sum_{\tilde{x} \in \cQ} (v_x - v_{\tilde{x}})^2.
\end{equation}
Here, $v_{x}$ represents the L2-normalized embedding for $x$.

\subsection{SetExpan}
\mynewcite{shen2017setexpan} introduce \setX, a SOTA framework combining context feature selection with ranking ensembles, for set expansion. \setX outperformed other SE methods such as SEISA in their evaluation. \setX represents entities by the contexts that they are mentioned in. For example, the context features for \underline{Batman} from \secref{sec:notation} will be \{\_\_ is good, \_\_ is smart\}. The contexts are used to create a large feature vector which can be used to compute the inter-entity similarity. The authors argue that using all possible features for computing entity similarity can lead to overfitting and semantic drift. To combat these problems, \setX builds the entity set iteratively by cycling between a context feature selection step and an entity selection step. In context feature selection, each context feature is assigned a score based on the set of currently expanded entities. Based on these scores, the context-features are reranked, and the top few context features are selected. The entity selection proceeds by the bootstrap sampling of the chosen context features and using those features to create multiple different ranked lists of entities. Multiple different ranked lists are finally combined via a heuristic method for ensembling different ranked lists to create a new set of expanded entities. This process is repeated to convergence to get the final list of expanded entities.

\section{\nvgeLong} %\nvgeLong}
Like BS, \nvgeLong first determines the underlying concept, or topic, underlying the query and then ranks entities based on that concept. My method differs from BS because I  use a deep generative model with a low dimensional concept representation, to simulate how a concept may generate a query. Also I use a ``distance'' (\S~\ref{sec:using-pzcd-to-rank}) between posterior distributions for ranking entities in lieu of bayesian model comparison.
%% This same %% intuition has been implemented in different ways in~\myciter{SEISA, SEAL, QBEES} methods.
%\begin{figure}[t]%[thbp]
%  \centering%  trim={<left> <lower> <right> <upper>}
%  \includegraphics[width=\linewidth,trim=0 0 0 0,clip]{fig/q2.pdf}
%  \caption{The generative model of query generation is shown on the left.
%    The small nodes denote probability distributions.
%    The variational inference model is shown on the right.
%    The transformation from $z$ to $g$ is done via the G-NN
%    and the transformation from $f_x$ to $q_\phi(z|x)$ is done via the I-NN.}
%  \label{fig:q}
%\end{figure}

\subsection{Inference Step 1: Concept Discovery}
\label{sec:latent-concept}
My model (Fig.~\ref{fig:q}) is as follows: $z \in \mathbb{R}^d$ is a low dimensional latent gaussian random variable representing the concept of a query. $z$ is sampled from a fixed prior distribution $\pi = \cN(\Z, \sigma^2\I)$, i.e. $z \sim \pi$. The members of $\cQ$ are sampled conditionally independently given $z$. $z$ is mapped via a multi layer perceptron (MLP), called $\gennn$, to $g$, the p.m.f.  of a multinomial distribution that generates $f_x$, the features of $x$. $\gennn$ is a neural network with a softmax output layer and parameters $\theta$. $f_x \in \mathbb{Z}^\rF$ are sampled i.i.d. from $p(f|z,\theta) = \gennn(z)$.\footnote{My generative model is inspired by \mynewcite{miao2016neural}'s NVDM. They assume that a single latent variable generates only one observation, but I posit that the same latent variable $z$ generates all observations in $\cQ$.}

In other words, the vector $f_x$ contains the counts of observed features for $x$ that were sampled from $g$, and $g$ was itself sampled by passing a Gaussian random variable through a neural network.

\begin{figure}[t]
  \begin{subfigure}[t]{0.4\linewidth}
    \centering
      \includegraphics[width=\linewidth]{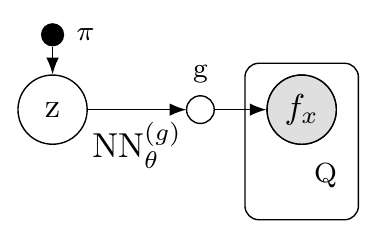}
      \caption{Generative Network}
    \end{subfigure}\hspace{0.1\linewidth}
  \begin{subfigure}[t]{0.49\linewidth}
    \centering
    \includegraphics[width=1\linewidth]{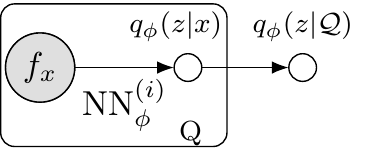}
    \caption{Inference Network}
  \end{subfigure}
  \caption[The NVSE generation and inference models.]{The generative model of query generation is on the left and the variational inference network is on the right. Small nodes denote probability distributions, gray nodes are observations and the black node $\pi$ is the known prior. $\gennn$ transforms $z$ to $g$ and the $\infnn$ transforms $f_x$ to $q_\phi(z|x)$.}
  \label{fig:q}
\end{figure}

Under this deep-generative model, a concept vector can simultaneously trigger multiple observed features. This allows us to capture the correlations amongst features triggered by a concept. For example, the concept of \cfont{president} can simultaneously trigger features such as \ffont{white house}, \ffont{executive order}, or \ffont{airforce one}.

To infer the latent variable $z$ ideally, I should compute $p_\theta(z|\cQ)$, the posterior distribution of $z$ given the observations $\cQ$. Unfortunately, this computation is intractable because the prior is not conjugate to the likelihood that has a neural network. Another problem is that it is unrealistic to assume access to a large set of ESE queries at training time, because user's information needs keep changing; therefore the approach used by \myciteauthor{manzil2017deep} in DeepSets to discriminatively learn a neural scoring function is \textit{impractical} for set expansion. For the same reason, it is also not possible to learn a single neural network whose input is $\cQ$ and which directly approximates $p_\theta(z|\cQ)$. Therefore it is non-trivial to apply the VAE framework to ESE. To overcomes these problems I make the assumption that a query $\cQ$ is conjunctive in nature, i.e. if entity $x_1$ and $x_2$ are present in $\cQ$ then results that are relevant to \textit{both} $x_1$ and $x_2$ simultaneously should be given a higher ranking than results that are related to $x_1$ but not $x_2$ or vice-versa. I implement the conjunction of entities in a query by combining the \textit{Product of Experts}~\mycite{hinton1999products} approach with the \textit{Variational Autoencoder} (VAE)~\mycite{kingma2014auto} method to approximate $p_\theta(z|\cQ)$.

I first map each $x$ to an approximate posterior $q_\phi(z|x)$ via a neural network $\infnn$ and then I take their product to approximate $p_\theta(z|\cQ)$. \[ p_{\theta}(z | \cQ) \approx q_\phi(z | \cQ) \propto \prod_{x \in \cQ} q_\phi(z | x). \] The $\phi$ parameters are estimated by minimizing $KL(q(z|x)\mid\mid p(z|x))$ as shown in \S~\ref{sec:est}.\footnote{ This is a generalization of~\mynewcite{bouchacourt2017multi} combining variational approximations of posterior distributions since the product of Gaussians is a Gaussian distribution.} The benefit of the POE approximation is that the posterior approximation $q_\phi(.|x)$ for each entity $x$ in $\cQ$ acts as an expert and the product of these experts will assign a high value to only that region where all the posteriors assign a high value. Therefore the POE approximation is a way of implementing conjunctive semantics for a query. Another benefit is that if $q_\phi(.|x)$ is an exponential family distribution with a constant base measure whose natural parameters are the output of $\infnn$, then the product of the distributions $\prod_x q_\phi(\cdot |x)$ lies in the same exponential family whose natural parameters are simply the sum of individual neural network outputs. Also, notice that the POE approach recommends adding the \textit{outputs} of the neural networks which is different than concatenating the features for all $x$ in $\cQ$ or naively adding the \textit{inputs} of the neural network.\footnote{Recently, \myciteauthor{manzil2017deep} gave a theorem that any permutation invariant function of sets must be representable as the function of a sum of features of elements of the set. I note that my POE approximation also has a similar form and is permutation invariant.}

I now  show in more detail how the product of experts can be computed simply by adding the output of the neural networks in the special case that the variational approximation has the following form:
\begin{align}
q_\phi(z | x) \propto h(z) \exp( \ip{\psi(z) }{\infnn(x)})
\end{align}
where $\psi(z)$ are the features of $z$. If $h$  is constant -- which is true for many exponential family distributions such as the Bernoulli, Exponential, Pareto, Laplace, Gaussian, Gamma, and the Wishart distributions -- then:
\begin{align*}
  q_\phi(z | x) &\propto \exp( \ip{\psi(z) }{\infnn(x)}).%\\
\end{align*}
In turn,
\begin{align*}
%\implies
\prod_{x \in \cQ}q_\phi(z | x) &\propto \exp( \ip{\psi(z) }{\sum_{x \in \cQ}\infnn(x)}).
\end{align*}
This shows that the product of experts can be computed simply by summing the outputs of the neural network activations for such \textit{deep-exponential} families with constant base measure.

I use $\infnn$ to compute the mean and log-variance of the gaussian distribution $q_\phi(z  | x)$ (\ref{eq:mean_log-var}) that I then convert to the natural parameters of a Gaussian  (\ref{eq:nat_params_gaus}). Next, I add the natural parameters of the individual variational approximations $\xi_x, \Gamma_x$ to compute the parameters $\xi_\cQ, \Gamma_\cQ$ for $q_\phi(z | \cQ)$ (\ref{eq:gamma}). Finally, I compute $q_\phi(z|\cQ)$ (\ref{eq:q_phi_z_Q}).
\begin{align}
  \mu_x, \Sigma_x &= \infnn(f_x)\label{eq:mean_log-var} \\
  \xi_x,\ \  \Gamma_x &= \mu_x  \Sigma_x^{-1},\ \  \Sigma_x^{-1}.\label{eq:nat_params_gaus} \\
  \xi_\cQ,\ \ \Gamma_\cQ &= \sum_{x \in \cQ} \xi_x,\ \ \sum_{x \in \cQ} \Gamma_x. \label{eq:gamma} \\
  {q_\phi(z|\cQ)} &= \cNc(z | \xi_\cQ, \Gamma_\cQ) \label{eq:q_phi_z_Q}
\end{align}

As explained above, the benefit of using the natural parameterization is that I can simply add the natural parameters of the individual variational approximations $\xi_x, \Gamma_x$ to compute the parameters $\xi_\cQ, \Gamma_\cQ$ for $q_\phi(z | \cQ)$ as
\begin{equation}%%\label{eq:gamma}
 \xi_\cQ, \Gamma_\cQ = \sum_{x \in \cQ} \xi_x, \sum_{x \in \cQ} \Gamma_x.
\end{equation}
Finally, I compute $q_\phi(z|\cQ)$ such that \[ {q_\phi(z|\cQ)} = \cNc(z | \xi_\cQ, \Gamma_\cQ), \] where $\cNc(z | \xi, \Gamma)$ is the multi-variate Gaussian distribution in terms of its natural parameters -- \[ \frac{|\Gamma|^{1/2}}{(2\pi)^{D/2}}\exp\left(-\frac{(z^T \Gamma z - 2\xi^T z + \xi^T\Gamma^{-1}\xi)}{2} \right). \]

\subsection{Inference Step 2: Entity Ranking}
%\subsection{Metric for entity rankings}
\label{sec:using-pzcd-to-rank}
In order to rank the entities $x \in \cR$, I design a similarity score between the probability distributions $q_\phi(z|\cQ)$ and $q_\phi(z|x)$ as an efficient substitute for bayesian model comparison. I use the distance between precision weighted means $\xi_{\cQ}$ and $\xi_{x}$ to define my ``distance'' function as
\begin{equation}
\underset{\nvge}\score(\cQ, x) = -||\xi_{\cQ} - \xi_{x}||^2 .
\label{eq:score_nver}
\end{equation}
My inter-distribution ``distance'' is not a proper distance because it changes as the location of both the input distributions is shifted by the same amount. I experimented with more standard, reparameterization invariant, divergences and kernels such as the KL-divergence and the Probability Product Kernel \mycite{jebara2004probability}, see \appref{app:rankings}, but I found my approach to be faster and more accurate. I believe this is because the regularization from the prior that encourages the posteriors to be close to the origin makes shift invariance unnecessary.

\subsection{Unsupervised Training}
\label{sec:est}
In general VAEs are a combination of deep neural generative models and deep approximations of posterior distributions of such generative models. \nvge is trained in an unsupervised fashion to learn its parameters $\theta$ and $\phi$. \mynewcite{kingma2014auto,rezende2014stochastic} proposed the VAE framework for learning richly parameterized conditional distributions $p_\ta(x | z)$ from unlabeled data. I follow \myciteauthor{kingma2014auto}'s reparameterization trick to train a VAE and maximize the \textit{Evidence Lower Bound}:
\begin{align}
  E_{q_\phi(z | x)}[\log p_\theta(x | z)] - KL(q_\phi(z|x) || p(z)).\label{bound:vae}
\end{align}
During training, I do not have access to any clustering information or side information that tells us which entities can be grouped. Therefore I assume that the entities $x \in \cX$ were generated i.i.d. The model during training looks the same as \figref{fig:q} but with one difference: $Q$ is a singleton set of just one entity.\footnote{More informally, I remove the plates from \figref{fig:q}.} Note that my learning method requires no supervision in contrast to methods like Deep Sets which require cluster information, or Neural Collaborative filtering methods which require a large dataset of user interactions.

To learn the parameters during training, I update $\phi$ and $\theta$ using stochastic back-propagation.

\subsection{Support for weighted queries}
\label{sec:weighted-queries}
Useful recommendation systems for users should be tunable.  If a recommendation system returns undesirable entities in response to the query, then the user should be able to easily tune the query so that the system excludes the undesirable results. Most search engines allow boolean exclusion operators or weighted query terms, but in the \er systems presented so far, a user can only change a query by either removing or adding entities. Furthermore, Weighted-queries enable users to tell the system what aspects of the query to focus on or ignore.

To apply user provided weights as the amount of influence that an entity should have on the final posterior over topics, I integrate the weights directly into the computation of the topic posteriors.  If the user provides weights $\btau = \{ \tau_x \mid x \in \cQ\}$, I compute the query features as

\begin{equation}\label{eq:weight}
\xi_{\cQ,\btau},\ \ \Gamma_{\cQ,\btau} = \sum_{x \in \cQ} \tau_x \xi_x,\ \ \sum_{x \in \cQ} |\tau_x|\Gamma_x.
\end{equation}
BM25 supports weights by multiplying each $f_x$ by $x$'s weights when computing $\fcq$. It is not clear how to incorporate weights in Bayesian Sets. instead of computing $\Gamma_\cQ = {\sum_{x \in \cQ}\infnn(x)}$ in (\ref{eq:gamma}), I perform a weighted sum \[\Gamma_{(\tau, \cQ)} = {\sum_{x \in \cQ}|\tau_x| \infnn(x)}\] Note that in computing the precision $\Gamma$  I only use the magnitude of the provided weights. To allow a user to tell the system to focus specificially on entities \textit{NOT} similar to a specific $x$, I enable a user to add negative weights. The signed weights are used for computing $\xi$ as follows: \[ \xi_{(\tau, \cQ)} = {\sum_{x \in \cQ} \tau_x \xi_{x}} \]

\section{Interpretability}
\label{sec:query-rational}
I introduce a general approach for interpreting \er models based on \textit{query rationales} to explain the latent concept the model discovered and \textit{result justifications} to provide evidence for why the system ranked an entity highly.

Useful recommendation systems for users should be tunable.  If a recommendation system returns undesirable entities in response to the query, then the user should be able to quickly tune the query so that the system excludes the undesirable results. Most search engines allow boolean exclusion operators or weighted query terms, but in the \er systems presented so far, a user can only change a query by either removing or adding entities. However, with the \nvge system, based on query rationales and result justifications, users can add weights to entities in a query to tell the system what aspects of the query to focus on or ignore.

\subsection{Query Rationale}
A \textit{Query Rationale} is a visualization of the latent beliefs of the \er system given the query $\cQ$. Given $\cQ$, I constructed a feature-importance-map $\gamma_{\cQ}$ that measures the relative importance of the features in $f_x$ and I show the top features according to $\gamma_\cQ$ as ``Query Rationales''. Recall that the $\ath{j}$ component of $f_x$, associated with entity $x$, measures how often the $\ath{j}$ feature co-occurred with $x$. I now present how I constructed $\gamma_\cQ$ for \nvge and the baselines.

For BM25, $\gamma_{\cQ}$ is simply $\bar{f_{\cQ}}$. In BS, $\gamma_{\cQ}$ is the weights from (\ref{eq:bs_score}): for each  $\ath{j}$ component of $f_x$,
\begin{equation*}
\gamma_{\cQ}[j] = \log \frac{\talpha[j] \beta[j]}{\alpha[j]\tbeta[j]}. %,
\end{equation*}
The benefit of generative methods such as BS and \nvge is that for them query rationales can be computed as a natural by-product of the generative process instead of as ad-hoc post-processing steps. For \nvge, ideally $\gamma_{\cQ}$ should be the posterior distribution $p_\theta(f | \cQ)$. Since this is intractable I approximate it by sampling the inference network: \[ p_\theta(f | \cQ) = E_{p_\theta(z | \cQ)} [p_{\theta}(f | z, \cQ)]  \approx E_{q_\phi(z | \cQ)} [p_{\theta}(f|z)] . \] I further approximate the expectation with a single sample of the mean of $q_\phi(z | \cQ)$. Finally the feature importance map for \nvge is: \[\gamma_{\cQ} = p_\theta(f | E[q_\phi(z | \cQ)]).\] Because \wTv finds the nearest-neighbor between entity embeddings, which are produced through a complicated learning process operating on the whole text corpus, it does not provide a natural way to determine the importance of a single sentence and therefore it is not possible to say what was the effect of a particular sentence on the query results. Similarly, since the \setX method works by extracting context features and iteratively expanding this feature set, it is not possible to determine the effect of a single sentence on the final search results.

\subsection{Result Justifications}
I define result justifications as sentences in $\cM_{x}$ that justify why an entity was ranked highly for a given query. Ranking the sentences that mention an entity is similar to ranking entities in $\cR$. Just as I create a feature vector for each $x$, I create a feature vector for each sentence in $\cM_{x}$ and use the same scoring function to rank the sentences based on the query. While computing a score for entity $x$ based on a query, I also score each sentence in $\cM_{x}$. My approach to generating interpretable result justifications is agnostic to \er methods with the caveat that for methods like \wTv and \setX this will require retraining or reindexing over the corpus for each query. My approach will not be feasible for such methods.

\subsection{Weighted queries}
Any recommendation system can occasionally fail to provide good results for a query. To improve a system's responses in such cases, I enable users to guide \nvge's results by using entity weights to influence the posterior distribution over topics.

If a user provides weights $\btau = \{ \tau_x \mid x \in \cQ\}$, I compute the query features via \eqnref{eq:weight}. The above formulae have an intuitive explanation that when an entity has a higher weight, then the precision over the concepts activated by that entity is increased according to the magnitude of the weight, and the value of the precision weighted mean is also weighted by the user-supplied weights. In turn, an entity with zero weight has zero effect on the final search result and entities with a high negative weight return entities diametrically opposite to that entity with higher confidence.

Weights can be applied to other methods as well. BM25 can multiply each $f_x$ by $x$'s weights when computing $\fcq$, and \wTv can use a weighted average. It is not straight-forward to incorporate weights in BS and \setX systems. One possible way is to use bootstrap resampling of the query entities according to a softmax distribution over entity weights, but bootstrapping makes the system non-deterministic and therefore even more opaque for a user. Also, bootstrap resampling requires multiple query executions, and it is not straight-forward to combine the outputs of different search queries; therefore I do not advocate for bootstrapping.

\section{Comparative Experiments}
My proposed method determines the latent random variable responsible for generating the query and then ranks the entities in $\cR$ by computing a distance between the probability of the latent variable given the given query and the probability of the latent variable given each entity. I test the hypothesis that \nvge can help bridge the gap between advances in IR and real-world use cases. I use human annotators on Amazon Mechanical Turk (AMT) to determine whether \nvge finds more relevant entities than my baseline methods in a real world, automatically generated KG.

\subsection{Dataset}
TinkerBell~\mycite{tinkerbell2017Tac} is a KG construction system  that achieved top performance in TAC-KGP2017 evaluation.\footnote{Tinkerbell constructed a KG from \texttt{LDC2017E25} that contains $30K$ English documents. Half of them are from online forums and the other half from Reuters and NYT. I focused on the $77,845$ entities from English documents appearing in $344,735$ sentences. $25,149$ entities were also linked to DBpedia.} I used it as my automatic KG. For each entity $e$ in TinkerBell I create $\cM_e$ by concatenating all sentences that mention $e$ and remove the top $100$ most frequent features in the corpus from $\cM_e$ to clean stop words. Tinkerbell was constructed from the TAC KGP 2017 evaluation source corpus, \texttt{LDC2017E25}, that contains 30K English documents and 60K Spanish and Chinese documents. \footnote{\url{tac.nist.gov/2017/KGP/data.html}} Half of the English documents come from online discussion forums and the other half from news sources, e.g., Reuters or the New York Times. My experiments only use the 77,845 EDL entities within TinkerBell that are assigned the type \texttt{Person}. I use these links to create a map from DBPedia categories to entities in TinkerBell, say $M$. Each entity in TinkerBell is associated with spans of characters that mention that entity. I tokenize and sentence segment the documents in LDC2017E25 and associate sentences to each entity corresponding to mentions. In the end, I get 344,735 sentences associated with the 77K entities. The median number of sentences associated with an entity is $1$, and the maximum number of sentences is $4638$ for the \textit{Barack Obama} entity.\footnote{The Mean is 4.43, the standard deviation is 29.19, the minimum number of sentences for an entity is 1, the maximum number of sentences is 4638, and the median is 1 (44,317 entities).} This is a good example of how automatic KGs differ from manually curated KGs. In TinkerBell most of the entities appear in only a single sentence so only a single fact may be known about them. In contrast, KGs like FreeBase and DBpedia have a more uniform coverage of facts for entities present in them. Another difference is that relational information such as ancestry relations between entities are noisier in an automatically generated KB than in DBpedia which relies on manually curated information present in Wikipedia.

\subsection{Implementation Details}
I prune the vocabulary by removing any tokens that occur less than $5$ times across all entities. I end up with, $\rF {=} 105448, \rV = 61311$, $\rDoc =  24661$, and  $\rVp = 19476$. I used BM25 implemented in Gensim~\mycite{rehurek_lrec} and I implemented BS myself. I choose $\lambda=0.5$, out of $0$, $0.5$, or $1$, after visual inspection. I used \wTv and \setX codebases released by the authors.\footnote{\url{https://bitbucket.org/yoavgo/word2vecf}, \url{github.com/mickeystroller/SetExpan}} For \nvge, I set $d{=}50$, $\sigma{=}1$. The generative network $\gennn$ does not have hidden layers and the inference network $\infnn$ has $1$ hidden layer of size $500$ with a $\tanh$ non-linearity and two output layers for the mean $\mu_x$ and log of the diagonal of the variance $\Sigma_x$. I use a diagonal $\Sigma_x$. \footnote{Training \nvge on $1$ Tesla K80 using the Adam optimizer with learning rate $5e^{{-}5}$ and minibatch size $64$ took $12$ hours.} For \wTv, I used $d=100$ to use the same number of parameters per entity as in \nvge. I trained with default hyperparameters for 100 iterations. I used \setX with the default hyperparameters as well except that I limited the number of maximum iterations to $3$ since I only needed top $4$ entities for my experiments.

\subsection{Experimental Design}
Prior work typically evaluates ESE on a small number of queries, constituting the most frequent entities, e.g., \myciteauthor{ghahramani2006bayesian} reported results for $10$ queries with highly cited authors and \myciteauthor{shen2017setexpan} used 20 test queries created of 2000 most frequent entities in Wikipedia. However, in automatic KGs, most entities are mentioned only a few times. For example, 60\% of the entities in TinkerBell are mentioned once. I am primarily interested in unbiased evaluation over such entities; therefore I stratified the evaluation queries into three types.

The 1st type contains entities mentioned in only $1$ sentence, the 2nd contains entities appearing in $2-10$ sentences, and the 3rd contains entities mentioned in $11-100$ sentences. I also stratified queries based on whether they had $3$, or $5$ entities. For each query type, I randomly generate 80 queries by first sampling 80 Wikipedia categories and then sampling entities from those categories that were also part of the TinkerBell KG. This results in $480$ queries. See \tabref{tab:qex} for examples.

For each query, I showed the names and first paragraphs from the Wikipedia abstracts of the query's entities, to help the AMT workers disambiguate entities unfamiliar to them. Then I showed the workers the top $4$ entities returned by each system. Each resultant entity was shown with up to 3 \textit{justification} sentences. Figure~\ref{fig:hit_system} illustrates the AMT interface.
\begin{figure}[htbp]
  \centering
  \includegraphics[width=\columnwidth,trim=0 10 0 0,clip]{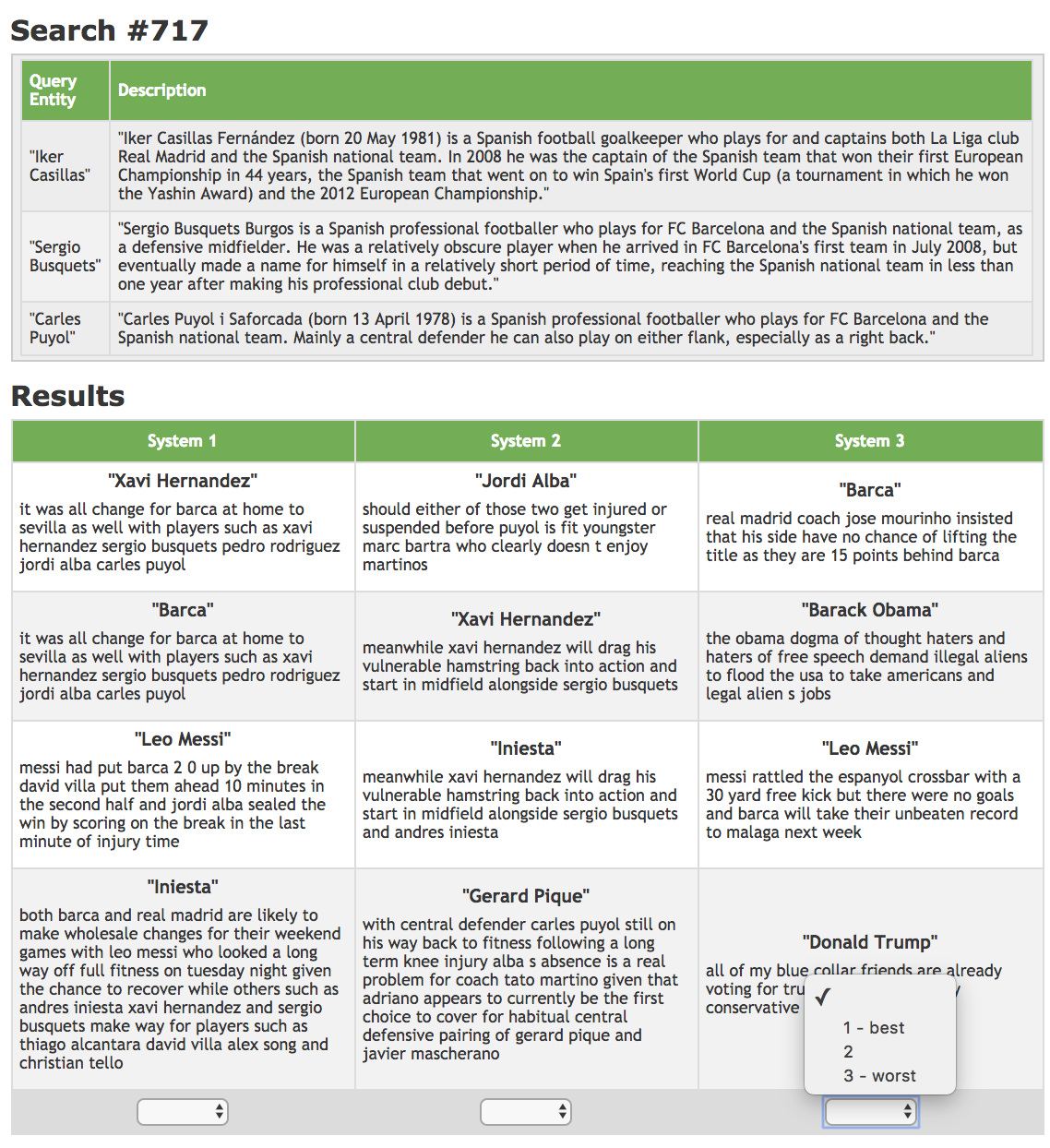}
  \caption[Example Mturk HIT for NVSE evaluation.]{Example of task shown to a crowd-source worker on Amazon Mechanical Turk.}
  \label{fig:hit_system}
\end{figure}

Since \setX and \wTv do not return justifications, I used \nvge to extract justifications for their results. I asked workers to rank the systems between $1$, the best system, to $3$, the worst; and I allowed for ties. The annotators found it difficult to compare results from 5 systems at a time, so I split my evaluation into two groups. Group 1 compared \nvge to BS and BM25, and group 2 compared \nvge to \setX and \wTv. I randomized the placement of the lists so that the workers could not figure out which system created which list.

\begin{table}[t]
  \centering
  \setlength{\tabcolsep}{4pt}
  \begin{tabular}{| p{5cm} | p{8cm} |}\hline
    \textbf{Category} & \textbf{Entities}\\[-1pt]\hline
    (1 Sent./Ent.) American Jazz Singers   & Paula West, Natalie Cole, Chaka Khan \\\hline
     (2-10 Sent.) Australian Major Golfers & Marc Leishman, David Graham, James Nitties\\\hline
     (11-100 Sent.) The Apprentice (U.S) Contestants& Maria, Rod Blagojevich, Dennis Rodman, Joan Rivers, Piers Morgan\\\hline
  \end{tabular}
  \caption[Examples of queries for NVSE evaluation.]{Examples of randomly created queries}
  \label{tab:qex}
\end{table}

\begin{table}[t]%[htbp]
  \setlength{\tabcolsep}{1.75pt}
  \centering
\begin{tabular}{|c@{\hskip 1pt}|r|ccc||ccc|}\hline
  {\small Ents$.$ In}  & {\small Sents. }  &          \multicolumn{3}{c||}{Group 1}      &          \multicolumn{3}{c|}{Group 2}       \\[-2pt]\cline{3-8}
  {\small Query} & {\small Per Ent.} & {\small \nvge}    & {\small BM25}    & {\small BS}  & {\small \nvge}    & {\small SetEx}   & {\small W2Vecf} \\\hline
            & 1          & 27       & \fb{38} & 15  & \fb{51}  & 14      & 15    \\[-1pt]
3           & 2-10       & \fb{29}  & 25      & 26  & \fb{49}  & 13      & 18    \\[-1pt]
            & 11-100     & \fb{35}  & 23      & 22  & \fb{44}  & 10      & 26    \\[-1pt]\hline
            & 1          & \fb{38}  & 25      & 17  & \fb{58}  & 19      & 3     \\[-1pt]
5           & 2-10       & \fb{40}  & 27      & 13  & \fb{53}  & 19      & 8     \\[-1pt]
            & 11-100     & 24       & \fb{33} & 24  & \fb{52}  & 11      & 17    \\[-1pt]\hline\hline
            & Total      & \fb{193} & 171     & 117 & \fb{307} & 86      & 87    \\[0pt]\hline
\end{tabular}
\caption[NVSE evaluation results (abbridged)]{The number of times a system was ranked $1^{st}$ over 80 queries compared to other systems in the same group. Ties were allowed so some rows may not sum to 80. Bold highlights the system with the most $1^{\text{st}}$ in its group. Extended results with second and third place rankings of the system are shown in \tabref{tab:old-sys-23}.}
\label{tab:old-sys}
\end{table}

\subsection{Results}
\tabref{tab:old-sys} %and \ref{tab:new-sys}
shows the number of times the annotators ranked each system as the best out of the $80$ queries. Over all queries, \nvge returned better results compared to the $4$ baselines systems. It performed best with 5 entities in the query where each entity was only mentioned up to $10$ times in the corpus. This shows that \nvge can discern better quality topics from multiple entities with sparse data. Extended results showing second and third place rankings of the systems are given in \tabref{tab:old-sys-23} which show that in cases that when NVSE does not rank first, it is typically chosen as the second-ranking system.
%% The largest gains between queries with $3$ or $5$ entities is when the entities appear in just $1$ sentence. One possible explanation is that it is easiest to determine $z$ when there are many entities in the query that appear only once in the text. This finding suggests that \nvge is robust to low data settings, i.e. when entities have little evidence in the text. %I observe that in one instance \nvge does not receive the most 1st place ratings, but does not do poorly there because it receives the most 2nd place ratings by a landslide. %and on average with

Table~\ref{tab:old-sys-23} shows the second and third place rankings of the systems and extends the results shown in \tabref{tab:old-sys}.
\begin{table}[htbp]%[htbp]
  \setlength{\tabcolsep}{1.75pt}
  \centering
    \begin{adjustwidth}{-1cm}{}
\begin{tabular}{|c@{\hskip 1pt}|r|ccc|ccc||ccc|ccc|}\hline
  {\small Ents$.$ In}  & {\small Sents. }  &          \multicolumn{3}{c|}{Group 1}      &          \multicolumn{3}{c||}{Group 2}   & \multicolumn{3}{c|}{Group 1}      &          \multicolumn{3}{c|}{Group 2}          \\[-2pt]\cline{3-14}
  {\small Query} & {\small Per Ent.} & {\small \nvge}    & {\small BM25}    & {\small BS}  & {\small \nvge}    & {\small SetEx}   & {\small W2Vecf} & {\small \nvge}    & {\small BM25}    & {\small BS}  & {\small \nvge}    & {\small SetEx}   & {\small W2Vecf}\\\cline{3-14}
  & 1      & 36 & 28 & 16 & 20 & 21 & 39 &17 & 14 & 49 & 9  & 45 & 26 \\[-1pt]
3 & 2-10   & 22 & 36 & 22 & 26 & 22 & 32 &29 & 19 & 32 & 5  & 45 & 30 \\[-1pt]
  & 11-100 & 24 & 26 & 30 & 23 & 22 & 34 &21 & 31 & 28 & 12 & 48 & 20 \\[-1pt]\hline
  & 1      & 28 & 37 & 15 & 20 & 47 & 13 &14 & 18 & 48 & 2  & 14 & 64 \\[-1pt]
5 & 2-10   & 22 & 27 & 31 & 21 & 50 & 9  &18 & 26 & 36 & 6  & 10 & 63 \\[-1pt]
  & 11-100 & 20 & 27 & 32 & 17 & 29 & 34 &36 & 20 & 24 & 11 & 40 & 29 \\[-1pt]\hline
\end{tabular}
\end{adjustwidth}
\caption[NVSE evaluation results (Longform)]{The number of times a system was ranked $2^{nd}$ (left subtable) and $3^{rd}$ (right subtable) over 80 queries. }
\label{tab:old-sys-23}
\end{table}

%I also note that
The IR method BM25 was the strongest baseline, outperforming BS and \setX, and even \nvge in two settings. I believe that this is because of the low-resource conditions of my evaluation where ad-hoc IR methods can have an advantage. Another reason why BM25 worked very well in my evaluation was the lack of auxiliary signals such as entity inter-relations and entity links and because all the entities were of type \textit{person}. This makes my task different from the entity list completion (ELC) task~\cite{balog2009overview} and a bit simpler for methods that focus heavily on lexical overlap. Another difference between the ESE task and the ELC task was that in the ELC task a descriptive prompt describing the query was also given to the users while evaluating the relevance of the returned results whereas no such prompt was given in the ESE task. I also found that sometimes BM25 was rated highly because it returned results that were highly relevant to a single query entity instead of being topically similar to all entities. For example, on the query associated with ``The Apprentice Contestants'' BM25's results solely focused on Dennis Rodman, but \nvge tried to infer a common topic amongst entities and returned generic celebrities which annotators did not prefer.

On entities with little data, \wTv and \setX  perform poorly. \wTv requires large amounts of data for learning useful representations~\mycite{AltszylerSS16} which explains why it performs poorly in my evaluation. The \setX algorithm directly uses context features extracted from the mentions of an entity and returns entities with the same context features. This approach can overfit with low data. Even though \setX uses an ensembling method to reduce the variance of the algorithm, I believe using context-features causes overfitting when an entity appears in only a few sentences. Lastly, I believe that BS suffers because its impoverished generative model has neither non-linearities nor low-dimensional topics for modeling correlations amongst tokens.

\begin{table*}[t!]%[htbp]
  \setlength{\tabcolsep}{2pt}
    \begin{adjustwidth}{-1.5cm}{}
\begin{tabular}{|cc|cc|cc|cc|cc|cc|}\hline
  \multicolumn{2}{|c|}{column $3$} & \multicolumn{2}{|c|}{column $14$} & \multicolumn{2}{|c|}{column $20$} & \multicolumn{2}{|c|}{column $33$} & \multicolumn{2}{|c|}{column $37$} \\ \hline
 merger & procurement& husband & iii & best & very & game & tackle & wild  & lighting\\
 industry & securities & sister & house & its & most & drill & fuzzy & holly & costumes\\
 premiers & AP-doc1 & \ffont{she} & labor & good & end & offensive & 21 & exhibit & fashion\\
 NYT-doc2 & analyst & \ffont{her} & king & some & do & coach & doc & martin's & nightclub\\
 protection & founders & daughter & church & only & such & artur & doc3 &  thriller & theatrical \\ \hline
 \multicolumn{2}{|c|}{\textit{business finance}} & \multicolumn{2}{|c|}{\textit{family royalty}} & \multicolumn{2}{|c|}{\textit{qualifier words}} & \multicolumn{2}{|c|}{\textit{football sports}} & \multicolumn{2}{|c|}{\textit{entertainment movie}} \\ \hline
\end{tabular}
  \end{adjustwidth}
\caption[Concepts corresponding to NVSE embedding components.]{ The first row contains top $10$ features most similar to $z_{j}$. The bottom row contains labels agreed upon by the authors; I loosely refer to the group where $j=20$ as ``qualifiers''. Underscored words signify that the feature came from $\cVp$.}
\label{tab:interpret_lat_concepts}
\end{table*}

\begin{table*}[htbp]
  %\linespread{0.9}\selectfont\centering
\begin{adjustwidth}{-1.5cm}{}
\begin{tabular}{|p{2.7cm}|p{2.67cm}|p{3.2cm}|p{3.2cm}|p{3.2cm}|}\hline
{ Abu Bakr Baghdadi (1)} & { Osama Bin Laden (1)} & { O.B. Laden (1.5)  A.B. Baghdadi (1)} & {  O.B. Laden (0.5)  A.B. Baghdadi (2)} & {  O.B. Laden (-0.2)  A.B. Baghdadi (1)} \\[2pt]\hline
{ qaida, iraq, abu, baghdadi, bakr, al, leader, attacks} & { bin, laden, osama, al, cia, pakistani, afridi, qaida} & { qaida, al, u, pakistani, cia, qaeda, government, killed} &  { qaida, al, leader, attacks, u, killed, officials, islamic} & { bakr, baghdadi, abu, iraq, al, sectarian, nuri, kurdish} \\\hline
\end{tabular}
\end{adjustwidth}
\caption[Example of Query Rationales output from NVSE]{The top row represents a query with weights in parentheses and the bottom row lists corresponding query rationales.} %Example of how query rationals change
\label{tab:query-rational-weights}
\end{table*}

\section{Analyzing Interpretability}
I now attempt to understand the similarity relations encoded in \nvge's internal concept representations to understand what it is learning. I also provide examples of how query rationales and query weights can help users fine-tune their queries.

\subsection{Understanding the concept space} To gain some insight into the distribution over concepts inferred by \nvge I determined the top $10$ words activated by individual dimension of $z$ by computing $\gennn(e_j)$ where $e_j$ is a one-hot vector in $\mathbb{R}^{50}$. % I let $k=10$ and I create a new vector $z_{j} = [0, \ldots, 0, z[j], 0, \ldots, 0]$. I determine the $10$ features that are closest to $z_{j}$ and manually determine a common theme from those features.
\tabref{tab:interpret_lat_concepts} shows the top $10$ words for selected components of $z$. I can easily recognize that dimensions $3,33$ and $37$ of $z$ represent finance, sports, and entertainment. Even though I did not constrain $z$ to be component-wise interpretable, this structure naturally emerged after training.
%Interpretability has two parts: how does the model understand the query, %why does the model return the results it did based upon how it understood the query. %Query rational addresses the former while results justifications address the latter. % To demonstrate its capabilities, I link to an anonymized video demonstration/figure. %For justification I will present a table with the best sentences for examples, %and these can also determined by hand.

\subsection{Weights \& Query Rationale}
\tabref{tab:query-rational-weights} depicts how the \textit{query rationale} returned by \nvge changes in response to entity weights. In the first column the query is \{Abu Bakr Baghdadi\} and the query rationale tells us that \nvge focuses on \textit{iraq}, \textit{baghdadi} etc. The second column shows a different query \{Osama Bin Laden\} and the query rationales changes accordingly to \textit{pakistani} and \textit{osama}. The third and fourth column show rationales when the weights on ``Laden'' and ``Baghdadi'' are varied. When more weight is put on ``Laden'' then the query rationales contain more features that are associated to him, and when more weight is put on ``Baghdadi'', then features such as ``islamic'' which is a token from his organization are returned. The last column shows an interesting configuration where a user is effectively asking for results that are similar to ``Baghdadi'' but dissimilar to ``Laden'' and the feature for \textit{kurdish} gets activated. Since the system returns results in under $100$ms, the user can fine-tune her query in real-time with the help of these query rationales.

I give one more example of the utility of negative weights: When $\cQ = \{\text{Brady}\}$, \nvge's rationale is [\textit{brady, game, patriots, left, knee, field, tackle}], indicating that \nvge associated the ``Brady'' entity with Tom Brady who is a member of the Patriots football team. When I added ``Wes Welker'' to $\cQ$ with a negative weight, the query rationale changed to [\textit{brady, game, left, tackle, knee, back, field}]. Since Wes is a Patriots receiver who received a negative weight in the query, \nvge deactivated the \textit{patriots} feature and activated the \textit{tackle} feature, opposite to a \textit{receiver}.
% the system ignores that aspect in the query.
%\item Example 1:
%\begin{itemize}
%\item "Brady" - brady, game, patriots, left, knee, field, tackle
%\item "Brady (1.0)", "Wes Welker" (-.5) - brady, game, left, tackle, knee, back, field
%\item "Brady (1.0)", "Wes Welker" (-.5), "Patriot" (-.5) - brady, game, tackle, knee, left, back, defensive
%(this last example the query rationale is all defensive things.
%\end{itemize}
%For example,
%\appref{app:weights_experiment} includes more examples and I provide a recorded demonstration at {\small \url{https://youtu.be/sGO_wvuPIzM}}.

%% \subsection{Intrinsic Quantitative Analysis of \nvge}
%% \nvge is a generative model and I can use the ELBO to upper bound its perplexity.
%% I randomly chose a subset $\cS \subset \cD$, of size $100$ and for each entity I estimated the ELBO using $100$ samples from ${q_\phi(z | x)}$. Through this procedure I upper bounded the perplexity of the model to be $64$ nats.
%% %% \begin{equation*}
%% %% \exp\left(- \frac{1}{100} \sum_{x \in \cS} \frac{\hat{l}_x}{\sum_{j=1}^\rV f_x[j]}\right),
%% %% %\text{Perplexity} \approx \exp\left(- \frac{1}{100} \sum_{x \in \cS} \frac{\hat{l}_x}{\sum_{j=1}^\rV f_x[j]}\right),
%% %% \end{equation*} where $\hat{l}_x$ estimates the ELBO for entity $x$.
%% The low perplexity of my model suggests that my model fits the data well.
%% Finally \figref{fig:sng} plots the singular values of the decoder matrix.
%% \begin{figure}[htbp]
%%   \centering
%%   \includegraphics[width=0.8\linewidth]{decoder_svd_really_cropped.pdf}
%%   \caption{Singular values of the decoder matrix in $\gennn$}
%%   \label{fig:sng}
%% \end{figure}

\section{Conclusion}
This chapter is based on the following journal publication:

\begin{itemize}
\item [] Rastogi, Pushpendre, Adam Poliak, Vince Lyzinski, and Benjamin Van Durme (2018). ``Neural variational entity set expansion for automatically populated knowledge graphs''. In: Information Retrieval Journal.
\end{itemize}

\textbf{The main ideas and scientific contributions are:}

\begin{itemize}
\item The first to learn Deep Representation of entities grounded in natural text for the task of Set Expansion.
\item The first to propose an efficient way of combining the posterior outputs of a VAE's inference network from multiple observations by summing the natural parameters.
\item One of the first methods for extracting query rationales for entity recommendations given a set expansion query.
\end{itemize}

I introduced \nvge as a step towards making advances in entity set expansion useful to real-world settings. \nvge is a novel unsupervised approach based on the VAE framework that discovers related entities from noisy knowledge graphs. \nvge ranks entities in a KG using an efficient and fast scoring function \eqref{eq:score_nver}, ranking 80K entities on a commodity laptop in $100$ milliseconds.

My experiments demonstrated that \nvge could be applied in real-world settings where automatically generated KGs are noisy. \nvge outperformed state of the art \er systems and other strong baselines on a real-world KG. I also presented a flexible approach to interpret \er methods and justify their recommendations.

In future work, I will extend my work by improving my model using more powerful auto-encoders such as the Ladder VAE~\mycite{sonderby2016ladder}; secondly I will experiment with the use of side information such as links from a KG through the use of Graph Convolutional Networks~\mycite{kipf2016semi}. Third, I will like to quantitatively measure how query rationales and justifications help users in accomplishing their search task. Finally, I will incorporate confidence scores from the KG in my model. Although there may be future work to improve my ESE method, I believe that \nvge serves as a significant step towards utilizing KGs and semantics for information retrieval and understanding in real-world settings.

%auto-ignore
\chapter{Knowledge Base Embeddings under Logical Constraints}
\label{cha:kbe}
Knowledge bases are large repositories of information about the entities in the real world and the relations between them. They can be thought of as large graphs marking the relations between real-world entities as the edges between its vertices. In the previous chapters, I presented algorithms for learning representation of words and entities from unlabeled, unstructured textual corpora. In this chapter, I shift focus from embedding the components of unstructured text to representing the structured information present in knowledge bases. To that end, I follow a two-pronged approach.\footnote{Previous versions of this work were published in \cite{rastogi2017elkb} and \cite{rastogi2017predicting}}

First, I scrutinize an existing method for embedding knowledge bases and demonstrate its shortcomings in accurately representing asymmetric-transitive relations both theoretically and empirically. I study the effect of the transitivity of a relation on the performance of the RESCAL algorithm by~\cite{nickel2011three}, and I demonstrate via a theorem and empirical results that RESCAL is inappropriate for representing transitive-asymmetric relations in a KB.

Second, I present new objectives and training algorithms for encouraging \textit{logical consistency} in the predictions by a knowledge base completion algorithm by incorporating logical constraints into the learning of entity and relation representations during the training of a Knowledge Base Completion (KBC) system. Enforcing logical consistency in the predictions of a KBC system guarantees that the predictions comply with logical rules such as symmetry, implication and generalized transitivity. My method encodes logical rules about entities and relations as convex constraints in the embedding space to enforce the condition that the score of a logically entailed fact must never be less than the minimum score of an antecedent fact. Such constraints provide a weak guarantee that the predictions made by a KBC model will match the output of a logical knowledge base for many types of logical inferences.

%% TODO
%% 2. Go slowly about than just jumping into, a discussion of why this is
%%    the right way to do things, the whole idea of representing words
%%    and entities goes back a long time, use logical space, why is the
%%    logical KB important. Explain how the BERT, ELMO type of methods
%%    don't
%%    New contextual representations method
%%    Will my method work or not work, you wanna talk about how the
%%    method is relevant an year from now.
%% 5. KB embeddings into the logical space - there may be something
%%    ill-suited. We are back to very limited
%%    Put a little bit of the smarts into the embedings, but your model
%%    should
%%    What are the advatanges
%%    Talk about what you are giving up. you can't expect ppl to
%%    What are disdvantages boaut what you are doing, in 5 years what
%%    will be, better models for learning representations.
%%    Here we made a decision, so other ppl who care about the same thing
%%    that we care about.

\section{Introduction} \label{sec:introduction} Large-scale and highly accurate knowledge bases~(KB) such as Freebase and YAGO2 ~\cite{hoffart2013yago2} have been recognized as essential for high performance on natural language processing tasks such as Relation extraction~\cite{dalton2014entity}, Question Answering~\cite{yao2014information}, and Entity Recognition in informal domains~\cite{ritter2011named}. Because of this importance of large scale KBs and since the recall of even Freebase, one of the largest open source KBs, is low\footnote{It was reported by~\citeauthor{dong2014knowledge} in October 2013, that 71\% of people in Freebase had no known place of birth and that 75\% had no known nationality.} a large number of researchers have presented models for knowledge base completion (KBC). Knowledge Base Completion (KBC), or link prediction, is the task of inferring missing edges in an existing knowledge graph.

A popular strategy for KBC is to \textit{embed} the entities and relations in low dimensional continuous vector spaces and to then use the learned \textit{embeddings} for link prediction. In other words, continuous real-valued vectors and matrices are automatically learned that can represent the entities and edges in a knowledge base, and at the time of inference, these real-valued representations are used to predict whether a particular edge exists between two entities. This general strategy can be implemented in many different ways, and I refer the reader to the survey by~\newcite{nickel2016review} for more details. Even though the strategy of embedding the elements of a graph is popular for knowledge base completion, theoretical studies of such methods are scarce. More specifically, although many methods have been evaluated empirically on select datasets for KBC, much less attention has been paid to understanding the relationship between the logical properties encoded by a given KB and the KBC method being evaluated.

In this chapter, I demonstrate \textit{theoretically}, and \textit{experimentally}, the adverse effect that asymmetric, transitive relations can have on a KBC method that relies on a single vector embedding of a KB entity. Transitive-asymmetric relations such as the \texttt{type of} relation in Freebase~\cite{bollacker2008freebase} and, the \textit{hyponym} relation in WordNet~\cite{miller1995wordnet} are ubiquitous in KBs and therefore very important~\cite{guha2015towards}. For my theoretical result, I analyze a widely cited KBC algorithm called RESCAL~\cite{nickel2011three,toutanova2015representing} and I prove that on large KBs that contain a large proportion of asymmetric, transitive relations, methods such as RESCAL will wrongly predict the existence of edges that are the reverse of edges in the training data. I also present a way to mitigate this problem, by using role sensitive embeddings for entities and I empirically verify that my proposed solution improves performance. Through my experiments, I also discover a drawback in the prevalent evaluation methodology, of randomly sampling unseen edges, for testing KBC models and show that random sampling can overlook errors on special types of edges.

A number of state of the art methods for Knowledge Base Completion (KBC) utilize a representation learning framework and learn distributed vector representations, i.e., \textit{embeddings}, of the entities and relations in a Knowledge Base (KB). Although such models make correct predictions on a sizable portion of the data, they cannot guarantee to follow logical rules and to make inferences consistent with those rules. For example, there is no way to guarantee in existing KBC methods that if an embeddings based KB predicts the fact that \emph{Alice murdered Bob} \twrs{(Murdered,(Alice, Bob))} then it will also predict that {\emph{Alice Killed Bob}}, even though it is very simple to enforce this in a traditional logical inference system by specifying the rule that \texttt{Murdered} implies \texttt{Killed}. Consider another example, if a KB knows that \twrs{Alice Is A Woman} and that \twrs{Bob Is Son Of Alice}, but the KBC method cannot guarantee to infer that \twrs{Alice Is Mother Of Bob} then such a method will have limited use in real applications.

In the second half of this chapter, I present a novel method for directly encoding logical rules via convex constraints on the embeddings. Such methods for directly ``shaping'' the feasible subspaces of embeddings based on logical properties of relations have not been deeply explored before, and I will show through my experiments that such a method can improve the performance of an existing KBC system. I validate my method via experiments on a knowledge graph derived from WordNet.

\section{Related Work}
\label{sec:related-work}
Due to the large body of work that has been done for the task of KBC, it is not possible to cover all of the related work on KBC in this section. Instead, I refer the reader to the survey~\cite{nickel2016review} for an overview of the empirical work that has been done in the area of KBC and link prediction. Similarly, The problem of enforcing consistency between the predictions made by a machine learning system and a first-order logic system, which is what my work attempts to do, has a large history of research but I will only be able to review recent work on learning representations of entities and relations of a knowledge graph and refer the reader to reviews of neural-symbolic systems~\cite{garcez2002neural-symbolic,hammer2007perspectives} for more references.

\subsection{Methods for Knowledge Base Completion}
\label{sec:meth-knowl-base}

Since I focus on the analysis of RESCAL, my work is most closely related to the paper~\cite{nickel2014reducing}. This chapter proves an important theorem that shows that the dimensionality required by the RESCAL model\footnote{Actually their theorem provides a lower bound for a more general model than RESCAL which automatically applies to RESCAL.} for exactly representing a weighted adjacency matrix of a knowledge graph must be greater than the number of strongly connected components in the graph. In my setting where I consider data sets that contain only transitive-asymmetric relations, the number of strongly connected components in the graph equal the number of vertices in the graph. Therefore their theorem proves that the dimensionality required for exactly representing a dataset such as WordNet using an algorithm such as RESCAL must be greater than the number of entities in the knowledge graph. In contrast to this result, my analysis gives an explicit example of a type of query for which the RESCAL algorithm will make wrong inferences.

My analysis trivially extends to a few other factorization based algorithms, e.g. the Holographic embedding algorithm by~\cite{nickel2016holographic}. The holographic embedding method can be rewritten as a constrained form of RESCAL with a ``holographically constrained'' matrix $M$. Figure~\ref{fig:holo} shows an example of a $3\times 3$ holographically constrained matrix with the constraint that elements with the same color must hold the same value. Since such a matrix is asymmetric by construction, my theorem proves that there will exist vectors $a,b$, and $c$ for which $M$ will violate transitivity.
\begin{figure}[htbp]
  \centering
  \includegraphics[width=.4\linewidth]{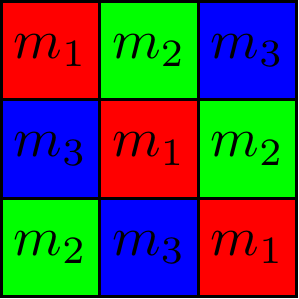}
  \caption[Illustration of Holographic Constraints.]{An illustration of a ``holographically constrained'' matrix.}
  \label{fig:holo}
\end{figure}

Recently~\newcite{bouchard2015approximate} argued that the phenomenon of transitivity of relations between vertices in a knowledge graph could be modeled with high accuracy if the knowledge graph is modeled as a thresholded version of a latent low-rank real matrix, and the vertex embeddings are learned as a low-rank factorization of that latent matrix. Based on this argument they claimed that factorizing a knowledge graph with a squared loss was less appropriate in comparison to factorizing it with a hinge loss or logistic loss. In this work, I provide an argument based on the symmetry of transitive matrices to show that the method of RESCAL which minimizes the squared reconstruction error must fail to capture phenomenon like transitivity in large knowledge bases. In this way, my results complement the work by~\citeauthor{bouchard2015approximate}.

\subsubsection{Logically Constrained Representations for KBC}
\label{sec:embedd-logic-constr}
\newcite{grefenstette2013towards} presented a novel model for simulating propositional logic with the help of tensors; however, their model relied on high-dimensional boolean embeddings of the entities and relations, and it only guaranteed adherence to the $\implication$  rule out of the ones presented in this chapter. \newcite{rocktaschel2014low,rocktaschel2015injecting} generalized Grefenstette's work learning embeddings of entities and relations that were real-valued and low dimensional and their learning mechanism could accommodate arbitrary first-order logic formulae into the parameter learning objective by propositionalizing the formulae. Their method has two drawbacks in comparison to my proposal --- 1. The process of propositionalization can be very expensive, especially for rules like $\proptrans$ and $\typeimpl$ that quantify over tuples of entities, and 2. Their method of \textit{differentiation through logic} does not guarantee that the learned embeddings will always be able to predict unseen relations that are logically entailed given the rules and the training data.

\newcite{bowman2014learning,bowman2015recursive} presented a neural network-based method for predicting the existence of natural logic relations between two entities. Their approach too had the drawback that it could not guarantee the inference of logically entailed facts.

\newcite{wang2015knowledge} presented an interesting approach for ``batch-mode'' knowledge base completion. They proposed to perform KBC in two steps -- First, they learned a distributional model of the entities and relations in a KB to predict the likelihood of unobserved facts. In the second step, they optimized a global objective with logical constraints using an ILP solver. Their approach is very different from ours since I present an online method for performing knowledge base completion and I directly translate the logical rules into constraints on the parameters instead of relying on a black box ILP solver.

\newcite{guo2015semantically} presented a method based on LLE~\cite{roweis2000nonlinear} for incorporating side information in the form of semantic categories of entities, but their method is not capable of incorporating the range of logical rules that I can. \newcite{demeester2016lifted} and \newcite{vendrov2016order-embeddings} proposed an approach to constrain the learnt embeddings in a way that is identical to the method prescribed by us in Subsection~\ref{ssec:constraints}. Our work generalizes their approach in two ways --- Firstly, I generalize their proposed constraints by using the language of convex geometry, and secondly, I propose constraints for many more logical rules and score functions than either of the previous two papers.

\newcite{hu2016harnessing} presented an adversarial setup with a \textit{teacher} neural network co-operating with a \textit{student} neural network to regularize the predictions of the student network to follow logical rules; however, their method amounts to propositionalization of the logical rules. Their method is more general than ours since it can be used to train neural networks however again it lacks guarantees during inference. \newcite{wang2016learning} presented a novel method of factorizing the adjacency matrix of a proof graph of a probabilistic logic language to learn embeddings of first-order logic formulas. My method is conceptually simpler than theirs and requires fewer training stages. Finally, \newcite{guo2016jointly} proposed an alternative method for embedding rules and entities based on t-norm fuzzy logics which was very similar to \newcite{rocktaschel2015injecting}'s approach.

\section{Theoretical Analysis of RESCAL}
\label{sec:main-result}
\paratextbf{Notation:} A KB contains \textit{(subject, relation, object)} triples. Each triple encodes the fact that a \textit{subject} entity is related to an \textit{object} through a particular \textit{relation}.  Let $\mathcal{V}$ and $\mathcal{R}$ denote the set of entities and relationships. {I use $\mathcal{V}$ to denote entities to evoke the notion that an entity corresponds to a vertex in the knowledge graph.} I assume that $\mathcal{R}$ includes a type for the \textit{null relation} or \textit{no relation}. Let $\mathrm{V} = |\mathcal{V}|$ and $\mathrm{R}=|\mathcal{R}|$ denote the number of entities and relations. I use $v$ and $r$ to denote a generic entity and relation respectively. The shorthand $[n]$ denotes $\{x | 1 \le x \le n, x \in \mathbb{N}\}$. Let $\mathcal{E}$ denote the entire collection of facts and let $e$ denote a generic element of $\mathcal{E}$. Each instance of $e$ is an edge in the knowledge graph. I refer to the subject, object and relation of $e$ as $e^{sub}, e^{obj} \in \mathcal{V}$ and $e^{rel} \in \mathcal{R}$ respectively. $\mathrm{E} = |\mathcal{E}|$ is the number of known triples.

\paratextbf{RESCAL:} The RESCAL model associates each entity $v$ with the vector $a_v \in \mathbb{R}^d$ and it represents the relation $r$ through the matrix $M_r \in \mathbb{R}^{d \times d}$. Let $v$ and $v'$ denote two entities whose relationship is unknown, and let $s(v, r, v') = a_v^{{T}} M_r a_{v'}$, then the RESCAL model predicts the relation between $v$ and $v'$ to be: $ \hat{r} = \operatorname*{argmax}_{r \in \mathcal{R}} s(v, r, v')$. Note that in general if the matrix $M_r$ is asymmetric then the score function $s$ would also be asymmetric, \ie $s(v, r, v') \ne s(v', r, v)$. Let $\Theta = \{a_v | v \in \mathcal{V}\} \union \{M_r | r \in \mathcal{R}\}$.
% Therefore even though the same embedding is used to represent an entity
% regardless of whether it is the first or the second entity in a relation,
% the RESCAL model could still handle asymmetry of relations
% if the matrix $M_r$ was asymmetric.

\paratextbf{Transitive Relations and RESCAL:} In addition to relational information about the binary connections between entities, many KBs contain information about the relations themselves. For example, consider the toy knowledge base depicted in \figref{fig:toykb}. Based on the information that \texttt{Fluffy} \textit{is-a} \texttt{Dog} and that a \texttt{Dog} \textit{is-a} \texttt{Animal} and that \textit{{is-a}} is a transitive relations I can infer missing relations such as \texttt{Fluffy} \textit{is-a} \texttt{Animal}.

Let us now analyze what happens when I encode a transitive, asymmetric relation. Consider the situation where the set $\mathcal{R}$ only contains two relations $\{r_0, r_1\}$. $r_1$ denotes the presence of the \textit{is-a} relation and $r_0$ denotes the absence of that relation. The embedding based model can only follow the chain of transitive relations and infer missing edges using existing information in the graph if for all triples of vertices $v, v', v''$ in $\mathcal{V}$ for which I have observed \textit{($v$, is-a, $v'$)} and \textit{($v'$, is-a, $v''$)} the following holds true:

\begin{align*}
s(v, r_1, v') {>} s(v, r_0, v') \texttt{ and } s(v', r_1, v'') {>} s(v', r_0, v'') & \implies s(v, r_1, v'') {>} s(v, r_0, v'')\\
\text{I.e. } a_v^{{T}} (M_{r_1} - M_{r_0}) a_{v'} > 0 \texttt{ and } a_{v'}^{{T}} (M_{r_1} - M_{r_0}) a_{v''} > 0 &\implies a_{v}^{{T}} (M_{r_1} - M_{r_0}) a_{v''} > 0 \numberthis \label{eq:mtrans}
\end{align*}

I now define a \textit{transitive matrix} and state a theorem that I prove in \secref{sec:proof}.
\begin{definition} A matrix $M \in \mathbb{R}^{d \times d}$ is transitive if $a^{{T}} M b > 0$ and $b^{{T}} M c > 0$ implies $a^{{T}} M c > 0$. \end{definition}
\begin{theorem}\label{thmMain} Every transitive matrix is symmetric. \end{theorem}
If I enforce the constraint in Equation~\ref{eq:mtrans} to hold for all possible vectors and not just a finite number of vectors then $M_{r_1} - M_{r_0}$ is a transitive matrix. By \thref{thmMain}, $M_{r_1} - M_{r_0}$ must be symmetric. This further implies that if $s(v, r_1, v') > s(v, r_0, v')$ then $s(v', r_1, v) > s(v', r_0, v)$. In terms of the toy KB shown in \figref{fig:toykb}; if the RESCAL model predicts that \texttt{Fluffy} \textit{is-a} \texttt{Animal} then it will also predict that \texttt{Animal} \textit{is-a} \texttt{Fluffy}.

\paratextbf{Augmenting RESCAL to Encode Transitive Relations:} The analysis above points to a simple way for improving RESCAL's performance on asymmetric, transitive relations. The reason that the original method fails to satisfactorily encode transitive asymmetric relations is because if the score $s(v, r_1, v')$ is high then $s(v', r_1, v)$ will also be high. I can avoid this situation by using two different embeddings for all the entities and compute the score of a relation through those role specific embeddings; i.e. I can use the embeddings $a_v^1, a_v^2$ to represent vertex $v$ and let $s(v, r_1, v') = a_v^1 M_{r_1} a_{v'}^2$ and $s(v', r_1, v) = a_{v'}^1 M_{r_1} a_{v}^2$. This idea of using role-specific embeddings has been known for a long time starting from~\cite{Tucker1966}. \footnote{Recently \cite{yoon2016translation} used this idea of using role-specific embeddings to preserve the properties of symmetry and transitivity in \textit{translation based} knowledge base embeddings.} The specific method that I have just explained is generally known to KBC researchers as the Tucker2 decomposition~\cite{sameertowards2015}. To encode more than one relations, only the matrix $M_r$ needs to change, but the entity embeddings can be shared across all relations.

\begin{figure}[htbp]
\centering
\includegraphics[]{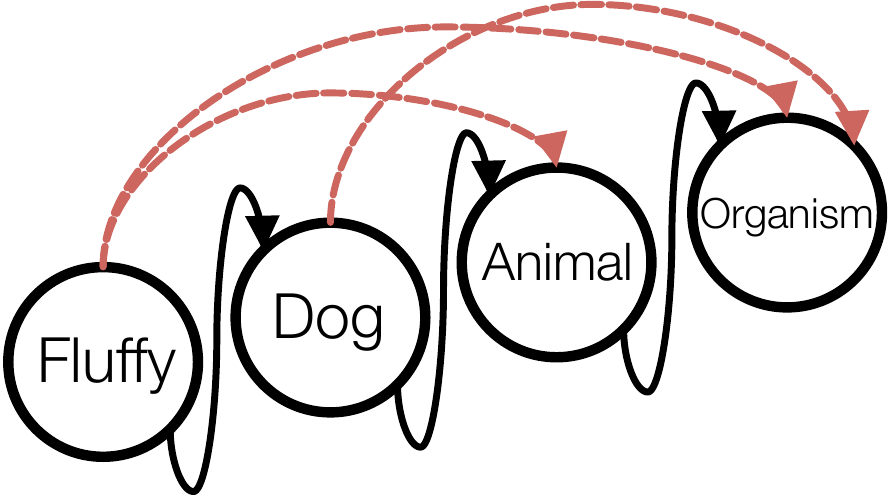}
\caption[A toy KB with one transitive relation.]{  A toy knowledge base containing only \textit{is-a} relations. The dashed edges indicate unobserved relations that can be recovered using the observed edges and the fact that \textit{is-a} is a transitive relation.}
\label{fig:toykb}
\end{figure}

\subsection{Proof of Theorem~\ref{thmMain}}
\label{sec:proof}
I now present my novel proof of \thref{thmMain} beginning with a lemma. \footnote{\thref{thmMain} was first proven by \cite{212808}(unpublished). My proof is more elementary and direct.} \begin{lemma}\label{lem2}   Every transitive matrix is PSD. \end{lemma}
\begin{proof}
  Consider the triplet of vectors $c := x, b := Mc, a := Mb$. Then   $a^T(Mb) = ||Mb||^2 \ge 0$ and $b^T(Mc) = ||b||^2 \ge 0$ and $a^TMc = b^TMb$.   Three cases are possible,   either $b = 0$, or $Mb = 0$, or both $Mb \ne 0$ and $b   \ne 0$. In the third case transitivity applies and I conclude that $b^TMb >   0$. In all cases $b^TMb \ge 0$ which implies $M$ is PSD.
\end{proof}
The next lemma proves that if $M$ is transitive then $x^TMy$ and $x^TM^Ty$ must have the same sign.
\begin{lemma}\label{lem3}   If $\exists x, y\ x^TMy > 0 $ but $x^TM^Ty < 0$ then $M$ is not transitive. \end{lemma}
\begin{proof}   Let $x,y$ be two vectors that satisfy $x^TMy > 0$ and $x^TM^Ty < 0$. Since   $x^TM^Ty = y^TMx$ therefore $y^T M (-x) > 0$. If I assume $M$ is transitive,   then $x^T M (-x) > 0$ by transitivity, but \lemref{lem2} shows such an $x$ cannot   exist. \end{proof}

\lemref{lem1} is a general statement about all matrices which states that if the two bilinear forms have the same sign for all inputs then they have to be scalar multiples of each other. I omit its proof due to space constraint.
\begin{lemma}\label{lem1}   Let $M_1, M_2 \in \mathbb{R}^{d \times d} \setminus \left\{ 0 \right\}$. If   $x^TM_1y {>} 0$ ${\implies}$ $x^TM_2y {>} 0$ then $M_1 = \lambda M_2$ for some $\lambda   > 0$. \end{lemma}
Finally I use \lemref{lem3}--\ref{lem1} to prove \thref{thmMain}. \begin{proof}   Let $M$ be a transitive matrix and let $x,y$ be two vectors such that   $x^{T} M y > 0$. By transitivity of $M$ and \lemref{lem3} $x^TM^Ty > 0$. Therefore   by \lemref{lem1} I get $M = \lambda M^T$ for some $\lambda > 0$.   Clearly $\lambda = 1$, this concludes the proof that $M$ is symmetric. \end{proof}

\subsection{Experimental Results}
\label{sec:experiment}
My theoretical result in \secref{sec:main-result} was derived under the assumption that the constraint~\ref{eq:mtrans} held over all vectors in $\mathbb{R}^d$ instead of just the finite number of vector triples used to encode the KB triples. It is intuitive that as the number of entities inside a KB increases my assumption will become an increasingly better approximation of reality. Therefore my theory predicts that the performance of the RESCAL model will degrade as the number of entities inside the KB increases and the dimensionality of the embeddings remains constant. I perform experiments to test this prediction of my theory.

\subsubsection{Experiments On Simulated Data}
\label{sec:simulation-study}
I tested the applicability of my analysis by the following experiments:  I started with a complete, balanced, rooted, directed binary tree $\mathcal{T}$, with edges directed \textit{from} the root \textit{to} its children. I then augmented $\mathcal{T}$ as follows: For every tuple of distinct vertices $v$, $v'$ I added a new edge to $\mathcal{T}$ if there already existed a directed path starting at $v$ and ending at $v'$ in $\mathcal{T}$. I stopped when I could not add any more edges without creating multi-edges.  For the rest of the chapter, I denote this resulting set of ordered pairs of vertices as $\mathcal{E}$ and those pairs of vertices that are not in $\mathcal{E}$ as $\mathcal{E}^c$.
%, For example, $\mathcal{E}$ contains an edge % from the root vertex to every other vertex and $\mathcal{E}^c$ contains an edge % from every vertex to the root vertex.
For a tree of depth 11, $\mathrm{V} = 2047, \mathrm{E}=18,434$ and $|\mathcal{E}^c| = 4,171,775$. See \figref{fig:ds} for an example of $\mathcal{E}, \mathcal{E}^c$.

\begin{figure}[htbp]
  \centering
  \includegraphics[width=0.5\linewidth]{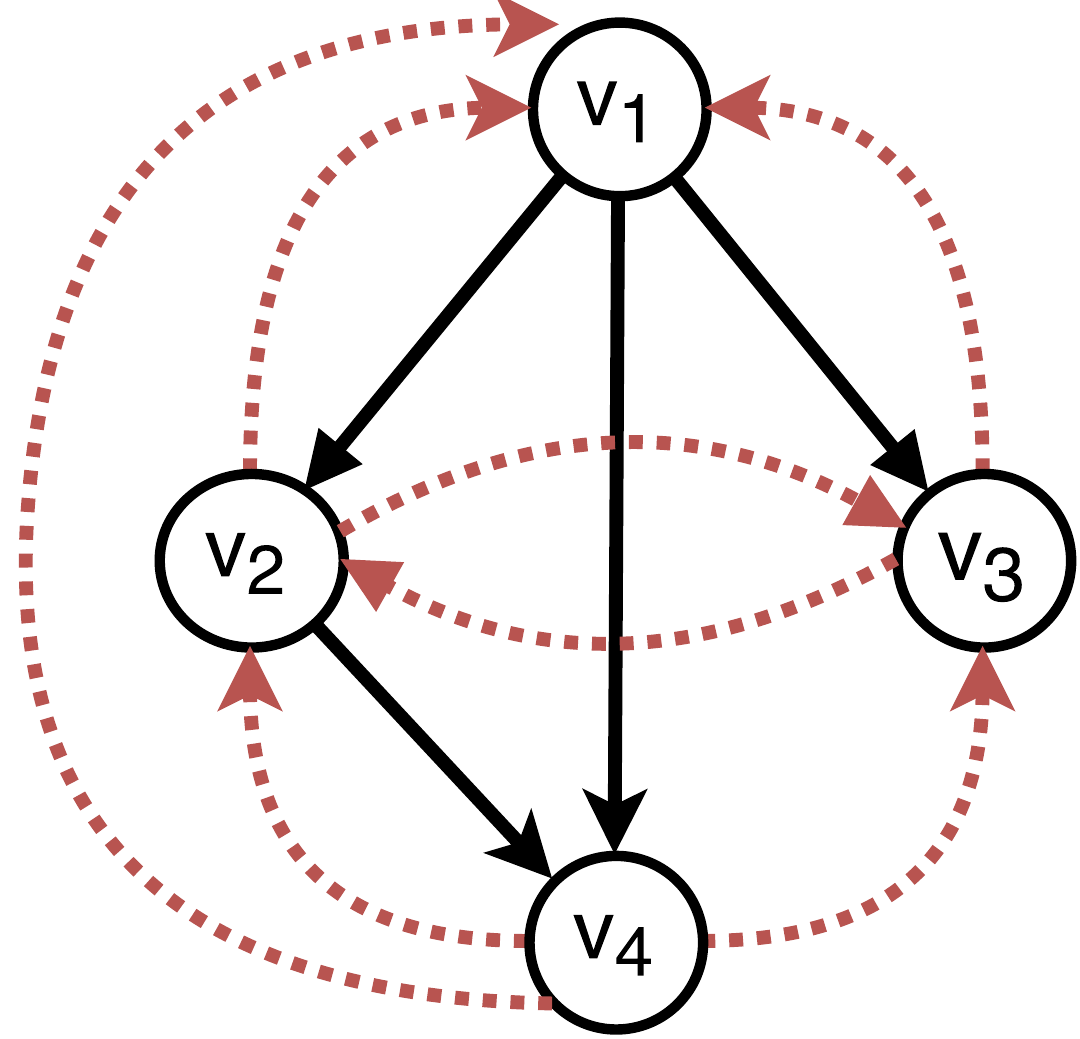}
  \caption[Illustration of $\mathcal{E}$ and $\mathcal{E}^{rev}$]{Assume that the black edges constitute $\mathcal{E}$ and the red dotted denote $\mathcal{E}^c$, then $\mathcal{E}^{rev}$ contains the edges $(v_4, v_1), (v_4, v_2), (v_2, v_1)$, and $(v_3, v_1)$.}
  \label{fig:ds}
\end{figure}

I trained the RESCAL model under two settings: In the first setting, called \textit{FullSet}, I used entire $\mathcal{E}$ and $\mathcal{E}^c$ for training. In the second setting, called \textit{SubSet}, I randomly sample $\mathcal{E}^c$ and select only $\mathrm{E} = |\mathcal{E}|$ edges from $\mathcal{E}^c$. All the edges in $\mathcal{E}$ including all the edges in the original tree are always used during both \textit{FullSet} and \textit{SubSet}. For both the settings of \textit{FullSet} and \textit{SubSet} I trained RESCAL 5 times and evaluated the models' predictions on $\mathcal{E}$, $\mathcal{E}^c$ and $\mathcal{E}^{(rev)}$. $\mathcal{E}, \mathcal{E}^c$ have already been defined, and $\mathcal{E}^{(rev)}$ is the set of reversed ordered pairs in $\mathcal{E}$.  \bigie $\mathcal{E}^{rev} = \{(u,v) | (v,u) \in \mathcal{E}\}$

For every edge in these three subsets, I evaluated the model's performance under $0-1$ loss. Specifically, to evaluate the performance of RESCAL on an edge $(v, v') \in \mathcal{E}$ I checked whether the model assigns a higher score to $(v, r_1, v')$ than $(v, r_0, v')$ and rewarded the model by $1$ point if it made the right prediction and $0$ otherwise. As before, $r_1$ and $r_0$ denote the presence and absence of relationship respectively.

Note that low performance on $\mathcal{E}^{rev}$ and high performance on $\mathcal{E}$ will indicate exactly the type of failure predicted from my analysis. I vary the dimensionality of the embedding $d$, and the number of entities $\mathrm{V}$, since they influence the performance of the model, and present the results in \tabref{tab:fullset}--\ref{tab:subset}. The rightmost column of \tabref{tab:subset} is the most direct empirical evidence of my theoretical analysis. The performance of RESCAL embeddings is substantially lower on $\mathcal{E}^{rev}$ in comparison to $\mathcal{E}, \mathcal{E}^c$. The last row with $d=400$, however, shows a very sharp drop in the accuracy on $\mathcal{E}^c$ while the performance of $\mathcal{E}^{rev}$ increases slightly. I believe that this happens because of higher overfitting to the forward edges as the number of parameters increases.

\begin{table}[htbp]
\addtolength{\tabcolsep}{-2.2pt}
\begin{subtable}[t]{\linewidth}
\centering
\begin{tabular}{r | ccc | ccc | ccc}
\multicolumn{10}{c}{{\textit{\small FullSet}}}\BStrut \\\hline%\cline{4-8}
\TStrut $d$ & \multicolumn{3}{c|}{$\mathrm{V}=2047$} & \multicolumn{3}{c|}{$4095$} & \multicolumn{3}{c}{$8191$} \\\hline
& $_\mathcal{E}$ & $_{\mathcal{E}^c}$ & $_{\mathcal{E}^{rev}}$ & $_\mathcal{E}$ & $_{\mathcal{E}^c}$ & $_{\mathcal{E}^{rev}}$& $_\mathcal{E}$ & $_{\mathcal{E}^c}$ & $_{\mathcal{E}^{rev}}$\\
$50$    & 66                & 100    & 100 & 60 & 100 & 100 & 54 & 100 & 100 \\
$100$   & 76                & 100    & 100 & 69 & 100 & 100 & 63 & 100 & 100 \\
$200$   & 86                & 100    & 100 & 79 & 100 & 100 & 72 & 100 & 100 \\
$400$   & 94                & 100    & 100 & 88 & 100 & 100 & 81 & 100 & 100 \\
\end{tabular}
\caption{ Accuracy in percentage of RESCAL with all the edges as training     data (denoted as \textit{FullSet}) on $\mathcal{E}, \mathcal{E}^c,     \mathcal{E}^{rev}$.\label{tab:fullset}}
\end{subtable}
\begin{subtable}[t]{\linewidth}
\centering
\begin{tabular}{r| ccc | ccc | ccc}
\multicolumn{10}{c}{{\textit{\small SubSet}}}\BStrut \\\hline%\cline{4-8}
\TStrut$d$ & \multicolumn{3}{c|}{$\mathrm{V}=2047$} & \multicolumn{3}{c|}{$4095$} & \multicolumn{3}{c}{$8191$} \\\hline
& $_\mathcal{E}$ & $_{\mathcal{E}^c}$ & $_{\mathcal{E}^{rev}}$ & $_\mathcal{E}$ & $_{\mathcal{E}^c}$ & $_{\mathcal{E}^{rev}}$& $_\mathcal{E}$ & $_{\mathcal{E}^c}$ & $_{\mathcal{E}^{rev}}$\\
$50$  & 100 & 93 & 52         & 100 & 91 & 48 & 100&  89 & 44 \\
$100$ & 100 & 78 & 58         & 100 & 92 & 56 & 100&  89 & 52 \\
$200$ & 100 & 60 & 72         & 100 & 71 & 61 & 100&  90 & 59 \\
$400$ &  100 & 54 & 67         & 100 & 57 & 70 & 100&  65 & 62 \\
\end{tabular}
\caption{ Accuracy in percentage of RESCAL trained  with all positive     edges and subsampled negative edges as training data (together called
\textit{SubSet}). \label{tab:subset}}
\end{subtable}
\caption[RESCAL accuracy in different training scenarios.]{$\mathrm{V}$ denotes the number of nodes in the tree. $d$ denotes the number of dimensions.}
\end{table}
\subsubsection{Experiments On WordNet}
\label{sec:wordnet-experiments}
% I tested my analysis on real data by performing experiments on the % WordNet dataset.
WordNet is a KB that contains vertices called \textit{synsets} that are arranged in a tree-like hierarchy under the \textit{hyponymy} relation. The hyponym of a synset is another synset that contains elements that have a more specific meaning. For example, the \textit{dog} synset\footnote{A synset must be   qualified by the word sense and the part of speech. So a valid   synset  called \textit{dog.n.01}. For simplicity   I skip this detail in my explanation but my implementation distinguishes   between the synset \textit{dog.n.01} and \textit{dog.n.02}.} is a hyponym of the \textit{animal} synset and an \textit{animal} is a hyponym of \textit{living\_thing} therefore a \textit{dog} is a hyponym of \textit{living\_thing}. I extracted all the hyponyms of the \textit{living\_thing.n.01} synset as the vertices of $\mathcal{T}$ and I used the transitive closure of the direct hyponym relationship between two synsets as the edges of $\mathcal{T}$. Quantitatively, the \textit{living\_thing} synset contained $16,255$ hyponyms, and $16,489$ edges. After performing the transitive closure $\mathrm{E}$ became $128,241$.

I performed two experiments with the WordNet graphs, using the same \textit{FullSet} and \textit{SubSet} protocols described earlier. The results are in the left half of \tabref{tab:wordnet}. I see that even though the accuracy on $\mathcal{E}$ and $\mathcal{E}^c$ is high, the performance on $\mathcal{E}^{rev}$ is much lower. This trend is in line with my theoretical prediction that the RESCAL model will fail on ``reverse relations'' as the KB's size increases.

\begin{table}[htbp]
  \centering
  \addtolength{\tabcolsep}{-2.2pt}
  \begin{tabular}{r | ccc | ccc || ccc}
    $d$   & \multicolumn{3}{c|}{\textit{FullSet}} & \multicolumn{3}{c||}{\textit{SubSet}} & \multicolumn{3}{c}{\textit{SubSet}} \\\hline
    & $_\mathcal{E}$ & $_{\mathcal{E}^c}$ & $_{\mathcal{E}^{rev}}$ & $_\mathcal{E}$ & $_{\mathcal{E}^c}$ & $_{\mathcal{E}^{rev}}$ & $_\mathcal{E}$ & $_{\mathcal{E}^c}$ & $_{\mathcal{E}^{rev}}$\\
    $50$  & 71 & 100 & 100       & 100 & 93 & 58   & 100 & 55 & 65     \\
    $100$ & 79 & 100 & 100       & 100 & 94 & 60   & 100 & 56 & 56     \\
    $200$ & 84 & 100 & 100       & 100 & 93 & 63   & 100 & 56 & 75     \\
    $400$ & 89 & 100 & 100       & 100 & 68 & 69   & 100 & 97 & 91     \\
  \end{tabular}
\caption[Accuracy on WordNet for Logically constrained embeddings.]{ Results from experiments on WordNet.   I used the subtree rooted at the \textit{living\_things} synset from the WordNet hierarchy. $d$ indicates the dimensionality of the embeddings used and the triple of numbers under \textit{FullSet} and \textit{SubSet} indicates the accuracy of RESCAL on $\mathcal{E}, \mathcal{E}^c, \mathcal{E}^{rev}$. $\mathrm{V}$ is $16,255$ for all columns. The right half shows results from experiments on WordNet with role dependent embeddings for entities.}
\label{tab:wordnet}
\end{table}

Finally, I present the results of augmenting RESCAL with role-specific embeddings in the right half of \tabref{tab:wordnet}. The results show that using role-specific embeddings increases the performance over the performance of the RESCAL algorithm and with a high dimensionality of embeddings it is possible to encode both the forward and the reverse relations in the embeddings. Please note that I do not claim that my proposed augmentation for RESCAL will empirically be any better than the much more recently proposed methods such as ARE~\cite{nickel2014reducing}, or Poincar\'{e} embeddings~\cite{nickel2017poincar}. I leave a careful empirical comparison of these techniques for future work.

\subsection{Discussion}
\label{sec:discussion}
The information present in large scale knowledge bases has helped in moving information retrieval beyond retrieval of documents to more specific entities and objects. And in order to further improve coverage of knowledge bases, it is important to research knowledge base completion methods. Since many knowledge bases contain information about real-world artifacts that obey hierarchical relations and logical properties, it is important to keep such properties in mind while designing knowledge base completion algorithms. In this chapter, I demonstrate a close connection between the logical properties of relations such as asymmetry, and transitivity, and the performance of KBC algorithms used to predict those relations. Specifically, I theoretically analyzed a popular KBC algorithm named RESCAL, and my analysis showed that the performance of that model in encoding transitive and asymmetric relations must degrade as the size of the KB increases. My experimental results in \tabref{tab:fullset},\ref{tab:subset} and~\ref{tab:wordnet} confirmed my theoretical hypothesis, and most strikingly I observed that the accuracy of RESCAL on $\mathcal{E}^{rev}$ was substantially lower than its performance on either $\mathcal{E}$ or $\mathcal{E}^c$, even though $\mathcal{E}^{rev}$ is a subset of $\mathcal{E}^c$.
% This discrepancy is most pronounced in configurations % where $\mathrm{V}$ is high, and $d$ is low.

In \tabref{tab:example}, I visualize the errors made by RESCAL by listing a few edges in $\mathcal{E}^{rev}$ that were wrongly predicted as true edges. These examples show that the trained RESCAL model can predict that \textit{fruit tree} is a hyponym of \textit{mango} or that every \textit{accountant} is a \textit{bean   counter}. Such wrong predictions can be harmful. Based on my analysis, I advocated for role-specific embeddings as a way of alleviating this shortcoming of RESCAL, and I empirically showed its efficacy in~\tabref{tab:wordnet}.

My results also highlight a problem with the commonly employed KBC evaluation protocol of randomly dividing the edge set of a graph into train and test sets for measuring knowledge base completion accuracy. For example with $d=50$ the average accuracy on both $\mathcal{E}$ and $\mathcal{E}^c$ is quite high but on $\mathcal{E}^{rev}$ accuracy is low even though $\mathcal{E}^{rev}$ is a subset of $\mathcal{E}^c$. Such a failure will stay undetected with existing evaluation methods.

\begin{table}[htbp]
  \centering\small
  \addtolength{\tabcolsep}{-2.2pt}
  \begin{tabular}{l | l}\hline
    \textit{\normalsize Argument 1} & \textit{\normalsize Argument 2}\\\hline
    draftsman.n.02   & cartoonist        \\
    fruit tree       & mango             \\
    taster           & wine taster       \\
    accountant       & bean counter      \\
    scholar.n.03     & rhodes scholar    \\
  \end{tabular}
  \caption[Some examples of errors by RESCAL]{ Examples of wrong predictions for the hyponym relations by the     RESCAL model with $d=400$ when trained under the \textit{SubSet} setting. The default synset is \textit{n.01}. i.e. the default synset in this table is the sense 1 for nouns.}
  \label{tab:example}
\end{table}

\section{Training Relation Embeddings under Logical Constraints}
\label{sec:method}
Let a knowledge base be defined as a tuple $(\cF, \cL)$,  with $\cF$ a set of statements, and a set of first order logic rules $\cL$. Every element $f \in \cF$ is itself a nested tuple $(r, (e, e'))$ which states that the entities $e$ and $e'$ are connected via the relation $r$. Let $\cE$ and $\cR$ be the set of all entities and relations respectively. Let $\cT$ be the set of all entity tuples that appear in $\cF$, and let $\cU$ denote the universe of all possible facts, i.e. $\cT = \{t \mid (r, t) \in \cF\}$, and $\cU = \{(r, (e, e')) \mid \ r \in \cR,\ e,e' \in \cE\}$. Note that $\rT \le \rF$.\footnote{   Per standard convention I denote the size of a set using the corresponding   roman symbol. E.g. $\rE$ is the size of $\cE$.} Finally, $\cFc = \cU \setminus \cF$ is the set of unknown facts. The goal of a KBC system is to rank the elements of $\cFc$ so that facts that are correct receive a smaller rank than incorrect facts.

\textbf{Embedding Model:} I assume that every relation $r \in \cR$ and entity $e \in \cE$ can be represented using real valued vectors $\bldr\in \mathbb{R}^d$ and $\blde\in \mathbb{R}^{\tilde{d}}$; $d$ and $\tilde{d}$ may have different values. The vector representation of each tuple $t$ is computed from its constituent entities via a composition function $\comp: \mathbb{R}^{\std} {\times} \mathbb{R}^{\std} {\to} \mathbb{R}^d$, i.e. $\bldt = \comp(\blde, \blde')$. For example $\comp$ may denote vector concatenation, in which case $\bldt = [\blde^T, {\blde'}^T]^T$. I will use the semicolon symbol $;$ as an infix operator to denote vector concatenation, i.e. $(\x; \y) = [\x^T, \y^T]^T$. Finally, $\x \ge \y$ denotes that the vector $\x$ is elementwise larger than $\y$ and $B(\x, r)$ denotes the $L_2$ ball centered at $\x$ with radius $r$.

\textbf{Score Function:} A majority of the existing work on embedding based KBC measures the \textit{correctness} of a fact via a scoring function, $\score: \cR {\times} \cE {\times} \cE {\to} \mathbb{R}$, with the property that when $\score(f) > \score(f')$, fact $f$ is more likely to be correct than $f'$. The two major classes of score functions are:
\begin{align}
  \score(f) &= \ip{\bldr}{\bldt}\label{eq:scoreip}\\
  \score(f) &= -\norm{\bldr- \bldt}\label{eq:scoredist}
\end{align}
In Equations~(\ref{eq:scoreip}--\ref{eq:scoredist}), $\bldr$ and $\bldt$ are vector representations of $r$ and $t=(e, e')$, respectively, that are constituents of $f=(r, (e, e'))$. For brevity, I will omit this expansion from here onwards.

\textbf{Unconstrained Objectives for Learning Score Function} \newcite{rendle2009bpr} proposed the Bayesian Personalized Ranking (BPR) objective as a way of tuning recommendation systems when a user can only observe positive training data, such as correct facts, but the facts that are absent may be either correct or incorrect. in this chapter I will focus on the BPR objective since this objective has been used for learning the parameters of a KBC system by various researchers~\cite{rendle2009bpr,demeester2016lifted,riedel2013relation}. \newcite{wang2016learning} experimentally showed that the BPR objective outperforms other objectives such as Hinge Loss and Log Loss.

BPR posits that the training data is a single joint sample of $\rU (\rU-1)$ bernoulli random variables $\{b_{ff'} \mid f \in \cU, f' \in \cU, f' \ne f\}$. $b_{ff'}$ equals $1$ when $f$ is in $\cF$ and $f'$ is in $\cFc$ and $0$ otherwise. $b_{ff'}$ is parameterized by its probability $p_{ff'}$ and all $b_{ff'}$ are conditionally independent given the probabilities $p_{ff'}$. The probability values must obey the reasonable condition that $p_{ff'} = 1 - p_{f'f}$. A natural way to satisfy this condition is to parameterize $p_{ff'}$ as $\sigma(\score(f) - \score(f'))$ where $\sigma$ is the sigmoid function.\footnote{The sigmoid function, $\sigma(x) =   \frac{1}{1+\exp(-x)}$, has the useful properties that   $\sigma(x) + \sigma(-x) = 1$ and   $\frac{d \sigma(x)}{dx} = \sigma(x)\sigma(-x)$.}
% Note that each partition of $\cU$ into $\cF$ and $\cFc$ creates a unique valuation of $\{b_{ff'}\}$ % but not every valuation of $\{b_{ff'}\}$ corresponds to a valid split of $\cU$.
The BPR estimator is simply the $L_2$ regularized MLE estimator of this probabilistic model, with regularization strength $\alpha$. \tabref{tab:obj} lists instances of the BPR objective that arise with different score functions.
\begin{table}[htbp]
    \begin{adjustwidth}{-1cm}{}
  \begin{tabular}{c |  c | l}\hline
    Model & $\score$ & \ \ {Minimization Objective ($J$)} \\\hline
    \tworow{A $\mid t=(\blde;\blde')$}{\AR $\mid t=\blde\otimes \blde'$}  &
    \eqbref{eq:scoreip} &
    $\begin{aligned}
      - \sum_{{(f,\fc) \in \cF {\times} \cFc}}
      \ \log(\sigma(\ip{\bldr}{\bldt} - \ip{\rc}{\tc}))
      + \alpha (\sum_{r \in \cR}\norm{\bldr}+ \sum_{t \in \cE} \norm{\blde})
    \end{aligned}$\\[10pt]
    \hline
    \begin{tabular}{c}
      $B$\\
      $t=(\mathbf{e}; \mathbf{e}'; \mathbf{e}^T\mathbf{e'} )$
    \end{tabular} &
    \eqbref{eq:scoreip}  &
    \begin{tabular}{l}\\
      $-\sum_{{\text{  }(f,\fc) \in \cF {\times} \cFc}} \log\sigma\Bigg(
      \begin{aligned}
        &\ip{\bldr_1}{\blde} + \ip{\bldr_2}{\blde'} + \ip{\blde}{\blde'} \\[-12pt]
        &- \ip{\rc_1}{\ec} - \ip{\rc_2}{\ec'} - \ip{\ec}{\ec'}
      \end{aligned}
      \Bigg)$ \\[12pt]
      $+ \alpha (\sum_{r \in \cR}\norm{\bldr}+ \sum_{t \in \cE} \norm{\blde})$\\
    \end{tabular}\\[10pt]
    \hline
    \tworow{C $\mid t=(\blde;\blde')$}{\CT $\mid t=\blde- \blde'$} &
    \eqbref{eq:scoredist} &
    \begin{tabular}{l}\\
    $-  \sum_{{(f,\fc) \in \cF {\times} \cFc}} \log(\sigma(\norm{\rc - \tc} - \norm{\bldr- \bldt}))$\\
    $+ \alpha (\sum_{r \in \cR}\norm{\bldr}+ \sum_{t \in \cE} \norm{\blde})$\\
    \end{tabular}
    \\[10pt]
    \hline
    \begin{tabular}{c}
      $D$\\
    $t=(\mathbf{e}; \mathbf{e}' ||\mathbf{e}-\mathbf{e}'||)$
    \end{tabular} &
    \eqbref{eq:scoredist} &
    \begin{tabular}{l}\\
      $-\underset{{(f,\fc) \in \cF {\times} \cFc}}{\displaystyle\sum}
       \ \log\sigma\left( \begin{aligned}\norm{\rc_1 - \ec} + \norm{\rc_2 - \ec'} + \norm{\ec - \ec'}\\[-12pt]
         - \norm{\bldr_1 - \blde} - \norm{\bldr_2 - \blde'} - \norm{\blde- \blde'}\end{aligned}\right)$\\[12pt]
       $+ \alpha \left(\sum_{r \in \cR}\norm{\bldr}+ \sum_{t \in \cE} \norm{\blde}\right)$
    \end{tabular}
  \end{tabular}
%% \end{adjustbox}
  \end{adjustwidth}
  \caption[Instantiations of the BPR objective.]{     Instances of the BPR objective corresponding to different choices of     composition and score functions. For example, if     $\comp(\blde, \blde') = (\blde;\blde')$ and \eqnref{eq:scoreip} is used as the score     function then I need to minimize the function in the first row with     respect to $\bldr,\blde$. In the first and third row, $\bldr= (\bldr_1;\bldr_2)$, in the     second row $\bldr= (\bldr_1;\bldr_2;1)$ and in last row $\bldr= (\bldr_1;\bldr_2;0)$. The     symbol $\otimes$ refers to the vector outer product operator that takes two     vectors of size $\std$ and produces a vector of size $\std^2$. Since this     scoring function is equivalent to the score function of RESCAL I call the     model \AR. Similarly the scoring function for $\CT$ is the same as TransE ~\cite{bordes2013translating}     .}
  \label{tab:obj}
\end{table}

\textbf{Logical Consistency of Embeddings through Constraints}
My general scheme for incorporating logical relations into embeddings is to ensure that during the learning of the vector representation of entities and relations, the score of a consequent fact will be greater than the score of any of its antecedents. In other words if $f_1, \ldots, f_{n-1} \implies f_n$ then $\score(f_n) \ge \min_{i \in [1, n-1]}(\score(f_i))$. If this does not hold, then it will be possible for my KB to assign a low score to a logically entailed fact even though all of its antecedents have a high score.

I analyzed common logical rules found in large scale KBs, and for different combinations of a logical rule and scoring function, I devised inequalities that the score function should satisfy. I translated those inequalities into constraints that restrict the entity and relation representations in a KB.

%% By completeness, I mean that the constraints should
%% ensure that the score of a consequent fact is larger than the minimum of the
%% scores of the antecedent facts. This constraint minimizes the number of false
%% negatives that the system can produce and therefore ensures ``completeness'' of
%% my set of predictions. The second constraint of efficiency simply means that
%% the constraints should be efficient to enforce during the parameter training
%% phase.  Specifically, this means that it should be possible to satisfy the
%% constraint by efficiently projecting onto the feasible set for that constraint.
I use the projected subgradient descent algorithm for learning the parameters of my KBS system. Algorithm~\ref{alg:tellmain} shows a specific instance, for Model A and batch size $1$, of my parameter learning algorithm with a general set of rules $\cL$. I now show how to construct convex constraints from logical rules.

\subsection{Constraints for Logical Consistency: Relational Implication}
\label{ssec:constraints}
I now present the constraints for guaranteeing that the predictions from embeddings based KBC systems are consistent with logical rules starting with implication rules. An implication rule of the form, \implication$(r, r')$, specifies that if a fact $f=(r, t)$ is correct, then $(r', t)$ must also be correct. For example, the rule \implication(\twr{Husband Of, Spouse Of}) enforces that if my KB predicts the fact, \twr{(Husband Of, (Don, Mel))}, then it will also predict \twr{(Spouse Of, (Don, Mel))}. As explained above I can enforce such a rule by ensuring that $\score(r', t) \ge \score(r, t)\ \forall t \in \cT$. \footnote{I abuse notation in saying that   $\score(r, (x, y)) = \score(r, x, y)$.}

When I use the inner product score function~(\ref{eq:scoreip}) then this inequality can be enforced by ensuring that $\ip{\bldr'-\bldr}{\bldt} \ge 0$ for all $t \in \cT$.
%% The naive approach of adding 1 constraint for each entity
%% tuple to the training objective has the drawback that the addition of $\rT$
%% constraints can tremendously slow down training~\cite{demeester2016lifted}.
%% Instead of adding a large number of constraints, if
I constrain $\bldt$ to lie in a subset of $\mathbb{R}^d$, say $\T$, then the implication rule can be enforced by constraining $\bldr'- \bldr$ to lie in the dual cone of $\T$, denoted $\T^*$. A very convenient special case arises when I chose $\T$ to be a ``{self dual cone}'' for which $\T = \T^*$. The set of positive real vectors $\rp$ is one example of such a self-dual cone. \footnote{Other self-dual cones, distinct from $\rp$ also exist such as the   Lorentz cone   $\{x \in \mathbb{R}^d\mid x_{d} \ge \sqrt{\sum_{i=1}^{d-1} x_i^2}\}$.   I refer the reader to \cite{gruber2007convex} for more details on the   geometry of closed convex cones and their polar and dual sets.}

When I use the $L_2$ distance score function~(\ref{eq:scoredist}) then the restriction on the score function translates into the following constraints on the vector representations: $\norm{\bldr- \bldt} - \norm{\bldr'- \bldt} \ge 0$ $\implies \ip{\bldr- \bldr'}{\bldr'+\bldr-2\bldt} \ge 0$ $\implies \ip{\bldr- \bldr'}{\bldr'+\bldr}/2 \ge \supp{\bldr- \bldr'}$. Here $\supp{\x}$ is the value of the support function of $\T$ at $\x$ which is defined as $\supp{\x} = \sup_{\bldt \in \T}\ip{\x}{\bldt}$. It is necessary and sufficient for the feasibility of this constraint that the $h_{\T}$ function should be finite for at least one value of $\x = \bldr- \bldr'$. Once I have fixed $\bldr-\bldr'$ then $\bldr+ \bldr'$ can be easily chosen from the halfspace $\hfs{\bldr'- \bldr}{2\supp{\bldr- \bldr'}}$. Note that if $h_{\T}$ is difficult to compute then implementing this constraint will also be difficult, therefore I must chose $\T$ wisely.

One example of a good choice of $\T$ is $\rp$. $\supp[\rp]{\bldr- \bldr'}$ is finite and zero \ifoif $\bldr- \bldr'\le 0$ therefore, the value of $\bldr+\bldr'$ must lie in the halfspace $\ip{\bldr- \bldr'}{\bldr'+\bldr}\ge 0$. Unfortunately, the problem of finding $\bldr$ and $\bldr'$ vectors that satisfy this constraint is non-convex and it is not possible to project on to this set given a pair of vectors that violate the constraints. I remedy this situation by adding an additional constraint that $\bldr+\bldr'$ must also lie in the negative orthant, i.e. $\bldr+ \bldr'\le \Z$. \tabref{tab:impl} presents all the derived constraints. Unfortunately, since the \CT model defines $\bldt = \blde- \blde'$, therefore it is not possible to set $\T = \rp$. In the case of the \CT model if I constrain $\blde$ to lie in $B(\Z, \rho)$ then $\bldt$ must lie in $B(\Z, {2}\rho)$ and $\supp{\bldr-\bldr'} = 2\rho(\bldr-\bldr') \implies \frac{\ip{\bldr- \bldr'}{\bldr'+\bldr}}{||{\bldr-     \bldr'}||} \ge 4\rho$. One way to make this constraint amenable to efficient projection is to enforce that $\bldr+ \bldr'= 4\rho (\bldr- \bldr')$ and $\norm{\bldr- \bldr'} \ge 1 \implies \norm{\bldr'} \ge |{2\rho - .5}|$. This constraint becomes trivial if $\rho = 0.25$

\subsubsection{Reverse Relational Implication and Symmetry}
\label{sssec:rev-impl}
A reverse relational implication rule denoted by $\rrelimpl(r,r')$ specifies that if $(r, (x, y))$ is correct, then $(r', (y, x))$ is also correct for all $(x,y) \in \cT$. This rule can be enforced through the inequality that $\score(r', y, x) \ge \score(r, x, y)$. Depending on the model let $\bldr= (\bldr_1; \bldr_2)$ or $(\bldr_1; \bldr_2; 1 \text{ or } 0)$ as shown in \tabref{tab:obj}, and similarly decompose $\bldr'$.  I will omit this detail in later sections. Under models A and B, this inequality translates to the following constraint $\ip{\y}{\bldr_1'} + \ip{\x}{\bldr'_2} \ge \ip{\x}{\bldr_1} + \ip{\y}{\bldr_2}$ and under models C and D, the necessary constraints are $\ip{\bldr_1 - \bldr_2'}{\bldr_1 + \bldr_2' - 2\x} \ge \ip{\bldr_1'-\bldr_2}{\bldr_1' + \bldr_2 - 2\y}$. Stronger versions of these constraints, which are more efficient to enforce, are shown in \tabref{tab:impl}.

A symmetry rule denoted as $\symm(r)$ specifies that if the fact $(r, (e, e'))$ is known to be correct then $(r, (e', e))$ is also correct. I can only comply with this logical rule in an embedding base KB by ensuring that $\score(r, e, e')$ $=$ $\score(r, e', e)$. Under all 4 score models this rule can be enforced only by ensuring that $\bldr_1 = \bldr_2$.

\subsubsection{Entailment}
\label{sssec:entailment}
A type A entailment logical rule denoted as \entail$_A(r,e,r',e')$ specifies that $(r, (e, x))$ implies $(r', (e', x))$ for all $x$ in $\cE$.\footnote{I use the   term, entailment, in the sense of entailment of properties. Note that this is different from implication.} A Type B entailment rule, \entail$_B(r,e,r',e')$ specifies that $(r, (x, e))$ implies $(r', (x, e'))$. $r$ and $r'$ may denote the same relation. For example, the rule \entail$_B$(\twr{Is A,Man,Is A,Mortal})  can be used to enforce that if \twr{(Is A, (Socrates, Man))} then the KB must also predict that \twr{(Is A, (Socrates, Mortal))}. The final constraints required to implement a type B entailment rule are shown in \tabref{tab:entailment}.\footnote{Details: To implement a type B entailment rule I need to ensure that $\score(r', x, e') \ge$ $\score(r, x, e)$ for all $x \in \cE$. Under model A this inequality translates to, $\ip{\bldr_1' - \bldr_1}{\x} \ge \ip{\bldr_2}{\blde} - \ip{\bldr_2'}{\blde'}$. Model B requires $\ip{\bldr_1' - \bldr_1 + \blde'- \blde}{\x} \ge \ip{\bldr_2}{\blde} - \ip{\bldr_2'}{\blde'}$. Model C requires $\ip{\bldr_1 - \bldr_1'}{\bldr_1 + \bldr_1' - 2\x} \ge \ip{\bldr_2' -\blde'+ \bldr_2 -   \blde}{\bldr_2' - \blde'- \bldr_2 + \blde}$, and finally the constraints over model D's score functions are $\ip{\bldr_1 - \bldr_1'}{\bldr_1 + \bldr_1'} + \ip{\blde- \blde'}{\blde+ \blde'} - \ip{\bldr_2' - \blde'- \bldr_2 + \blde}{\bldr_2' - \blde'+ \bldr_2 - \blde} \ge 2 \ip{\bldr_1 - \bldr_1' -   \blde- \blde'}{\x}$.}

\subsubsection{Property Transitivity}
\label{sssec:proptrans}
A property transitivity rule denoted \proptrans$(r, r', e', r'', e'')$ specifies that if $(r, (x, y))$ and $(r', (y, e'))$ are true then $(r'', (x, e''))$ is also true. For example, the rule \proptrans(\twr{Partner, Convicted, Criminal, Suspected, Criminal}) can be used to incorporate the common sense rule that if an entity is the partner of a convicted criminal then it will be suspected of being a criminal into the embeddings based KB. Note that the score of the hypothesis fact $(r'', (x, e''))$ should be high if the antecedent facts have high score for any possible entity $y$. A natural way in which I can incorporate such a rule into score based KBC models is by ensuring that $\score(r'', x, e'') \ge \max_{y \in \cE} \min(\score(r, x, y), \score(r', y, e'))$. In order to derive efficient constraints that can enforce this inequality I now strengthen the constraint imposed on the score function by replacing the $\min$ function in the lower bound to a convex combination of the scores, i.e. let $\lambda \in (0,1)$, I enforce the inequality that $\score(r'', x, e'') \ge \max_{y \in \cE} \lambda \score(r, x, y)  + (1-\lambda) \score(r', y, e')$ .

Since a convex combination of two values is greater than their minimum, this stronger inequality translates to the following constraint for model A: $\ip{\blde''}{\bldr_2''} - (1-\lambda)\ip{\blde'}{\bldr_2'} + \ip{\x}{\bldr_1''-\lambda \bldr_1} \ge \ip{\y}{\lambda \bldr_2 + (1-\lambda)\bldr_1'}$. Let $\ab = \frac{\bldr_1'' - \lambda \bldr_1 + \blde''}{\lambda}$, $\bb = -\frac{(1-\lambda)(\bldr_1' + \blde') + \lambda \bldr_1}{\lambda}$, $\cb= \frac{\ip{\bldr_2''}{\blde''} - (1-\lambda)\ip{\bldr_2'}{\blde'}}{\lambda}$, and let $\bE$ contain the set $\{\blde\mid e \in \cE\}$. For Model B, the above inequality on the score function leads to the the constraint: $\forall \x, \y \in \bE,\ \ip{\x}{\y} \le \ip{\x}{\ab}  + \ip{\y}{\bb} + \cb$. Remember that my goal is to devise a set $\bE$, and constraints on relation embeddings such that it is efficient to project onto it and for which the above inequality can be guaranteed. The following proposition shows how to construct such a set:
\begin{proposition}\label{prop:trans}
  Let $\x, \y$ be members of $\rp \cap B(\ab, ||\ab||)$ and $\ab \ge \Z$   then $\ip{\x}{\y} \le \ip{\x + \y}{\ab}$.
\end{proposition}
The above proposition shows that if $\ab=\bb$ and $\cb \ge 0$ then by setting $\bE = \rp \cap B(\ab, ||\ab||)$ I can satisfy the above constraints.
%% Unfortunately, I were not able to drive suitable constraints for models C and %% D. See \tabref{tab:entailment}

\begin{table}[htbp]
  \centering
  \begin{tabular}{c | r | l }
    Rule & Model & {Constraints} \\\hline\hline
    \multirow{2}{*}{$\implication(r, r')$}& A, \AR , B &  $\bldr\le \bldr'$ \\
    & C, D &  $\bldr\le \bldr'\le -\bldr$\\\hline\hline
    \multirow{3}{*}{$\rrelimpl(r,r')$}& A,B & $\bldr_2' \ge \bldr_1$, $\bldr_1' \ge \bldr_2$\\
    & C, D & $\bldr_1 \le \bldr_2' \le -\bldr_1$, $\bldr_2 \le \bldr_1' \le -\bldr_2$.\\
    & R & $\textit{matrix}(\bldr')\ge \textit{matrix}(\bldr)$
    \end{tabular}
  \caption[Sufficient constraints for enforcing implication.]{Constraints sufficient for enforcing $\implication(r, r')$ and $\rrelimpl(r,r')$
    The constraint $\blde\ge \Z \forall e \in \cE$ applies for all models.
  $\textit{matrix}$ is the inverse of the operation that converts a matrix to a vector by concatenating its columns. I.e. $\textit{matrix}(\bldr)$ denotes the matrix
    form of the vector $\bldr$.}
  \label{tab:impl}
\end{table}

\begin{algorithm}[H]
  \floatname{algorithm}{Alg.}
  \caption[Projected SGD Algorithm]{Projected SGD for Model A, Batch Size=1}\label{alg:tellmain}
  \begin{algorithmic}%[1]
    \State Given: $\cF, \cFc, \cL$. Hyperparameters: $\alpha, \eta, \rS$.
    \For{each fact $f \in \cF$}
    % \State Mark $f$ as seen.
    \For{$\rS$ steps}
    \State Sample $\fc = (\bar{t}, \bar{r})$ from $\cFc$
    % \State Set $f = f'$ with $p=0.5$,
    % \State $\qquad$or sample $f$ from seen facts
    \State Let $\vvv = \sigma(\ip{\rc}{\tc} - \ip{\bldr}{\bldt})$
    \State \LeftComment{Fix $\blde$ and optimize ${J}$}
    \State {$\grad{\bldr}{\Jr} = {-\bldt}\vvv, \ \grad{\rc}{\Jr} = {\tc}\vvv$}
    \State $(\bldr; \rc) {\leftarrow}$\textit{proj}
           $_{\cL}\left((\bldr; \rc) - \eta ((\grad{\bldr}{\Jr}; \grad{\rc}{\Jr}) + 2\alpha (\bldr; \rc))\right)$
    \State \LeftComment{Fix $\bldr$ and optimize ${J}$}
    \State {$\grad{\bldt}{\Je} = {-\bldr}\vvv, \ \grad{\tc}{\Je} = {\rc_1}\vvv$}
    \State $(\bldt; \tc) {\leftarrow} $\proj$_{\cL}\left((\bldt; \tc) - \eta ((\grad{\bldt}{\Jr}; \grad{\tc}{\Jr}) + 2\alpha (\bldt; \tc))\right)$
    \EndFor
    \EndFor
  \end{algorithmic}
\end{algorithm}

\subsubsection{Type Implication}
\label{sssec:type-implication}
A type implication rule, denoted as $\typeimpl(r, e, r')$, specifies that if the fact $(r, (x, y))$ is correct then $(r', (x, e))$ is also correct $\forall (x, y) \in \cT$. In other words, this rule enforces that positional arguments of a relation possess certain properties. For example, the rule \typeimpl(\twr{Husband of, Man, Gender}) can enforce that if my KB predicts the fact that \twr{(Husband of, (Don, Mel))} then it also predicts that \twr{(Gender, (Don, Man))}.

Under model A the $\typeimpl(r, e, r')$ rule translates to the following inequality for the parameters $\ip{\x}{\bldr'_1} - \ip{\x}{\bldr_1} \ge \ip{\y}{\bldr_2} -\ip{\blde}{\bldr'_2} \forall (x, y) \in \cT$. Let $\ab = \blde+ \bldr_1' - \bldr_1$, $\bb=-\bldr_2$ and $\cb=\ip{\bldr_2'}{\blde}$. Under model B, the restriction on the score function translates to: $\ip{\x}{\y} \le \ip{\ab}{\x} + \ip{\bb}{\x} + \cb$. The analysis for this case again relies on Propostion~\ref{prop:trans} and the analysis for models C and D is yet out of reach. See \tabref{tab:entailment}.

\begin{table}[htbp]
  \centering
  \begin{tabular}{c | l}
    {Model} & \multicolumn{1}{c}{Constraints} \\\hline
    A& $\bldr_1' \ge \bldr_1$, $\ip{\bldr_2}{\blde} \le \ip{\bldr_2'}{\blde'}$\\[1ex]\hline
    B& $\bldr_1' \ge \bldr_1 + \blde- \blde'$, $\ip{\bldr_2}{\blde} \le \ip{\bldr_2'}{\blde'}$\\[1ex]\hline
    C& $\begin{aligned}
      &\bldr_1 \le \bldr_1' \le -\bldr_1, \bldr_2 - \blde\le \bldr_2' - \blde',\\
      &\bldr_2'  + \bldr_2 \le \blde'+ \blde
    \end{aligned}$ \\[2ex]\hline
    D& $\begin{aligned}
      &\bldr_1 - \bldr_1' \le \blde'- \blde,\; \blde\ge \blde', \bldr_1 \le \bldr_1' \le -\bldr_1,\\
      &\bldr_2 - \blde\le \bldr_2' - \blde',\; \bldr_2'  + \bldr_2 \le \blde'+ \blde\\
    \end{aligned}$
  \end{tabular}
  \caption[Sufficient constraints for entailment.]{Sufficient constraints for $\entail_B(r, e, r', e')$. The constraint $\blde\ge \Z \forall
  e \in \cE$ applies for all models.}
  \label{tab:entailment}
\end{table}

\begin{table}[htbp]
  \centering
  \begin{tabular}{c | r | l}
    Rule & Model & \multicolumn{1}{c}{Constraints}\\\hline
    \multirow{5}{*}{$\proptrans$} & A & $\begin{aligned}\bldr_1'' \ge \lambda \bldr_1, \lambda \bldr_2 + (1-\lambda)\bldr'_1 \le \Z\\ \ip{\blde''}{\bldr_2''} \ge (1-\lambda)\ip{\blde'}{\bldr_2'}\end{aligned}$\\[1ex]\cline{2-3}
& B & $\begin{aligned}\bldr_1'' + \blde'' + (1-\lambda)(\bldr_1' + \blde')=0,\\
                   {\ip{\bldr_2''}{\blde''} \ge (1-\lambda)\ip{\bldr_2'}{\blde'}}, \text{and}\\
                   \ab \ge \Z, \forall x \in \cE,\ \x \in \rp \cap B(\ab, ||\ab||)
    \end{aligned}$\\\hline\hline
    \multirow{3}{*}{$\typeimpl$}
    & A & $\bldr'_1 \ge \bldr_1$, $\ip{\blde}{\bldr'_2} \ge 0$ and $\bldr_2 \le \Z$ \\\cline{2-3}
    & B & $\begin{aligned}
      \blde + \bldr_1' = \bldr_1 - \bldr_2,\ \ip{\bldr_2'}{\blde} > 0, \text{and}\\
      -\bldr_2 \ge \Z, \forall x \in \cE\, \x \in \rp \cap B(-\bldr_2, ||\bldr_2||)\end{aligned}$\\
  \end{tabular}
  \caption[Sufficient constraints for property-transitivity.]{Constraints for enforcing $\proptrans(r, r', e', r'', e'')$ and $\typeimpl(r, e, r')$.
    $\ab = \frac{\bldr_1'' - \lambda \bldr_1 + \blde''}{\lambda}$}
  \label{tab:proptrans}
\end{table}

\subsection{Evaluating Logical Deduction and KBC on WordNet}
\label{sec:experiment-elkb}
My method for training embeddings based KBC systems allows for a very interesting application for solving logical puzzles using an embedding based KBC system without using an external logical-symbolic subsystem. I perform a controlled experiment where I compare the performance of an embedding based KBC system trained with the constraints versus a system that has been trained without those constraints.

\textbf{Data} Consider the logical deduction problem shown in~\tabref{tab:prob1}. This is a simplified version of a logical puzzle presented in \newcite{russell1995modern}. In this puzzle, $\twr{Nono}$ is a country that possesses a ${WMD}$ and $\twr{Benedict}$ has traded with $\twr{Nono}$. The KB has to deduce whether $\twr{Benedict}$ is a criminal based on just two input facts and 3 rules. The total number of facts is $5^2 \times 4 = 100$.

\begin{table}[htbp]
  \centering
  \begin{tabular}{|l|}\hline
    \multicolumn{1}{|c|}{\textbf{Rules}}              \\\hline
    \implication(\twr{Trade With, Transact With})     \\
    \entail$_B$(\twr{Possess,{WMD},Considered,Enemy}) \\
    \tworowb[l]{\proptrans(\twr{Transact With, Enemy,}}%
    {\,\,\twr{Considered, Criminal, Considered})}     \\\hline
    \multicolumn{1}{|c|}{\textbf{Facts}}              \\\hline
    $(\twr{Possess,(Nono,WMD)})$                      \\
    $(\twr{Trade With,(Benedict,Nono)})$              \\\hline
    \multicolumn{1}{|c|}{\textbf{Query ?}}            \\\hline
    (\twr{Considered,(Benedict, Criminal)})          \\\hline
  \end{tabular}
  \caption[Example of a logical deduction problem]{\label{tab:prob1} A Logical Deduction Problem. Based on the rules and facts a KB should infer that \texttt{Benedict} is a \texttt{Criminal}.}
\end{table}

\begin{table}[htbp]
  \centering
  \begin{tabular}{l | ccc | ccc }
          & \multicolumn{3}{c|}{\underline{\;Baseline\;}} & \multicolumn{3}{c}{\underline{\;ELKB\;}} \\
        Model & P@10 & MRR & MAP & P@10 & MRR & MAP   \\\hline
    \;\;A     &0.00 &0.02& 0.01  & \textbf{0.20}\udg&\textbf{0.44}\udg& \textbf{0.83}\udg \\
    \;\;B     &0.00 &0.03& 0.03  & \textbf{0.17}\udg & \textbf{0.26}\udg& \textbf{0.35}\udg \\
  \end{tabular}
  \caption[Results for the ELKB method]{Table of Results. The baseline of \AR is equivalent to the RESCAL
    method.% and \CT is equivalent to TransE
    Bold marks that the average
    performance is higher. $\dagger$ implies that the difference is significant
    with two-tailed p-value $\le 0.005$ as measured by a matched pair t-test. \label{tab:results1}}
\end{table}

\textbf{Evaluation} I train two versions of two KBC systems, Models A and B, with batch size$=1$, $\alpha=0.001, \eta=0.1, S=200, d=50$, and $\std=25$ using Algorithm~\ref{alg:tellmain}. Both KBs were trained in one pass using the two training facts. The only difference was that the baseline system did not constrain the embeddings to obey logically derived geometric constraints. After training, I queried the KBs for the scores of all possible facts. I ranked all the facts based on their scores, excluding the training facts, and marked all facts that could be logically entailed from the two training facts as correct results and the rest of them as incorrect. I performed 10 runs, and in each run, I computed the MRR, P@10, MAP for the two models.
Finally, I averaged these quantities over 10 runs.%% \footnote{Since I am testing whether a
  %% logically entailed fact will be given higher score by a KBC system without
  %% explicitly generating entailed facts using a 3rd party system therefore we
  %% do not compare to other methods that additively augment the training objective
  %% with logically entailed facts.}

\textbf{Results}  \tabref{tab:results1} shows that my method was able to rank logically entailed facts with much higher precision and recall than the baseline systems. This validates my intuition that logical rules can be usefully incorporated into the parameter learning mechanism of a KBC system via simple geometric constraints even for low dimensional embeddings. The reason for the large improvement in performance by the ELKB system in comparison to the baseline is that the ELKB model makes the score of entailed facts higher than the score of non-entailed facts because of the constraints during learning. E.g. the scores of entailed facts such as \twr{(Considered,(Nono, Enemy))}, and \twr{(TransactWith,(Benedict, Nono))} are forced to be high in comparison to non-entailed facts such as \twr{(Trade With,(Benedict,WMD))}. In comparison, the baseline method does not have this systematic advantage, and its scores remain unchanged.

\subsection{Evaluating Link Prediction on WordNet}
In the link prediction task, the KBC system is given incomplete facts, with either a missing head entity or tail entity, i.e. given either $(r, (\_, e'))$, or $(r, (e, \_))$ the system has to predict $e$ or $e'$ respectively. I evaluated the utility of proposed constraints by comparing the performance of model A and model B trained with and without the constraints. I now present the results of my experiments on the WN18 knowledge graph,\footnote{I found that the performance of models $C, D$ and $\AR$ was too low therefore I do not report their results.} derived from WordNet, and released by~\newcite{bordes2013translating}, which is a popular testbed for KBC algorithms \cite{wang2014knowledge,lin2015learning,toutanova2015representing,yang2015embedding}.

\textbf{Data} The WN18 dataset comes with standard train, development and test splits. These three splits of the data contain 141442, 5000, and  5000 facts respectively. The total number of relations in the WN18 dataset is 18, and the number of entities is 40,943. Recently \newcite{guo2016jointly} publicly released a list of logical rules\footnote{\url{aclweb.org/anthology/attachments/D/D16/D16-1019.Attachment.zip}}
%get rid of this is I need the space which I directly incorporated into my framework. All of their rules were $\rrelimpl$ rules.%% \footnote{The complete list of 14

For all the models I fixed batch size$=10$, $\alpha=0.001, \eta=0.125, S=200, \std=100$, for model $\CT$, $d=100$ and otherwise $d=200$. Following existing work, I measured the MRR, HITS@3 and Hits@10 metrics and reported their average over the two tasks of head entity prediction and tail entity prediction. Instead of training in a single pass, I trained my models for 50 epochs on the WN18 dataset and chose the best parameters using early stopping on the validation set. In other words, I used the parameters from that epoch which performed the best on the validation set in terms of the HITS@10 metric. Finally, I combined the predictions of the best performing T model and the best performing system based on model B. In order to combine the two ranking systems I trained a logistic regression classifier using the default settings in vowpal webbit\footnote{\url{https://github.com/JohnLangford/vowpal_wabbit}} to first predict whether model T or model B will produce a better ranking and then output that system's ranking over entities for evaluation. My logistic regression classifier had 73\% accuracy on the training data and 70\% accuracy over the test data. By using this third system, I were able to create a single ranking system that performed better than model \CT which is very similar to the TransE model.\footnote{The main differences between model \CT and TransE are that TransE used hinge loss versus the BPR objective. TransE does not regularize the relation embeddings and forces the entity embeddings to lie on the unit sphere, instead in model \CT I add a quadratic regularization term to regularize the embeddings.}
\begin{table}
  \centering
  \begin{tabular}{c | c | l ll}
    Model & Project & MRR    & HITS@3   & HITS@10   \\\hline
    A     & No      & 0.0152 & 0.016 & 0.0330 \\
    A     & Yes     & 0.0238 & 0.03  & 0.0514 \\
    B     & No      & 0.0677 & 0.072 & 0.137  \\
    B     & Yes     & 0.241  & 0.283 & {0.50}  \\
    T    & No      & {0.311}  & {0.412} & {0.66}\\
    B$^{\scriptstyle\text{project}}$+T  & --      & \textbf{0.367}  & \textbf{0.475} & \textbf{0.708}
  \end{tabular}
  \caption[Comparative Results for ELKB models]{\label{tab:comp} MRR, HITS@3 and HITS@10 of the
constrained and unconstrained versions of models A, B and unconstrained T.
B+T reports the results of combining models B and model T.}
\end{table}

\textbf{Results} \tabref{tab:comp} shows that both the constrained and unconstrained versions of model A perform quite poorly.  This is to be expected since model A scores a triplet $(r, (e, e'))$ as $\ip{\bldr}{\blde} + \ip{\bldr}{\blde'}$. Regardless of $e'$, the ranking produced by the model will remain the same. Therefore model $A$ is unsuitable for this task, for similar reason model C is also an unsuitable model. However, the drastic improvement in the performance of model $B$ when it is trained according to the constraints corresponding to the $\rrelimpl$ rules demonstrates the utility of my proposed constraints. After adding the constraints, the MRR increased almost 3 times and the value of HITS@10 by 4 times from $0.137$. Recall that at test/inference time the constraints do not play any role, so the only role of the constraints is as a form of regularization on the parameters of the model.

\subsection{Discussion}
\label{sec:disc}
It is instructive to look at a few examples of the predictions that model B makes and to compare them to the predictions made by model T. \tabref{tab:diff} compares the top 5 predictions of constrained model B with the predictions from  model T for the input $(\texttt{hypernym}, (\_, \texttt{flowering shrub NN\_1}))$. The true answer is \texttt{poinciana\_gilliesii\_NN\_1} so the model \CT achieves a reciprocal rank of 1 for this example, but constrained model B is not able to rank the right answer within the top 5 answers. This list of answers shows that model B ranks those answers higher that are similar to \texttt{flowering shrub NN\_1}, but it is not able to properly use the relation information. However, by properly combining the models, I can improve the performance of the overall system.
\begin{table}[htbp]
  \centering
  \begin{tabular}{| c | c |}\hline
    B & T \\\hline
\text{flowering\_shrub\_NN\_1}      &    \text{poinciana\_gilliesii\_NN\_1}   \\
\text{genus\_caesalpinia\_NN\_1}    &    \text{mysore\_thorn\_NN\_1}          \\
\text{shrub\_NN\_1}                 &    \text{flowering\_shrub\_NN\_1}       \\
\text{tree\_NN\_1}                  &    \text{pernambuco\_wood\_NN\_1}       \\
\text{rosid\_dicot\_genus\_NN\_1}   &    \text{caesalpinia\_bonducella\_NN\_1}\\\hline
\end{tabular}
\caption[Error analysis of logically constrained model]{\label{tab:diff} A comparison of the top 5 predictions of constrained model B with the predictions from  model T for the input $(\texttt{hypernym}, (\_, \texttt{flowering shrub NN\_1}))$. }
\end{table}

\section{Conclusion}
\label{sec:conclusion}

I have presented a novel method for incorporating logical constraints into an embedding based knowledge base by constraining the parameters of a KB. My experiments on a small logical deduction problem, and on WordNet, indicate that my ideas of imposing geometric constraints on embeddings for enforcing logical rules are sound and that they can improve the generalization of models that are hard to train otherwise. Although the KBC models A, B, C and D do not perform as well as existing models trained without constraints such as TransE,  I show that they can be used as part of a combination of systems to improve upon existing methods.

%auto-ignore
% Yesterday the spin was that I have written a hypothesis and that the two experiments give me conclusive answers regarding which method is better for which application. And that I have better error analysis in the case of entity linking regarding what the shortcomings of VAE.
\chapter{Comparative Experiments}
\label{cha:more}

In this chapter, I present experiments on the downstream tasks of Coreferent Mention Retrieval and Entity Linking using the MVLSA and Variational Autoencoder representation learning methods.

\textit{The Coreferent Mention Retrieval (CMR) task}, is an information retrieval task, in which the system receives a query sentence mentioning an entity, and the goal is to retrieve sentences containing coreferent mentions of that entity. A user may use a CMR system to find more mentions of an entity when performing an exploratory task over a corpus containing information about entities. The CMR task is a special case of the well-studied problem of {Cross-Document Coreference Resolution}~\cite{bagga1998entity,mayfield2009cross} -- in which the system has to cluster all mentions of all entities  -- in which, unlike Cross-Doc Coref. Rather than operating on the entire mention graph, the system uses retrieval techniques to limit its focus to an implicit subgraph anchored by the given query mention.

Recently, \cite{sankepally2018test} introduced the CMR task and introduced a new dataset for this task. In this chapter, I compare a number of unsupervised representation learning methods to learn representations of mentions. I compare LSA, MVLSA, and Variational Auto Encoder based approaches and demonstrate that these features can improve the performance of a strong information retrieval system.

\textit{The Entity Linking task} is the task of automatically annotating spans of words in natural language texts that mention an entity with the coreferent entity. Entity linking is also sometimes called entity disambiguation to distinguish it from the task of jointly detecting the span of words mentioning an entity and the linking the span to an entity. In this chapter, I focus exclusively on the task of linking a given span in an unstructured plain text document to an entity in a Knowledge Base.

The ACE corpus~\cite{ace05} contains a wide range of documents from varying genres such as newswire and online newsgroups. The entire corpus is annotated with the mention boundaries, coreference information between mentions, the semantic type of each mention, and finally, the entity links for each mention which were added by \cite{bentivogli2010rte6}. The semantic type of an entity can be from one of seven classes: \texttt{Person, Organization, GPE, Location, Facility, Weapon}, and \texttt{Vehicle}. The entity links from mentions to a canonical Wikipedia URL are absent for the \texttt{Weapon}  and \texttt{Vehicle} classes. I compare LSA, MVLSA and Variational Auto Encoder based approaches for entity linking and I observe how these unsupervised features learned can improve entity linking performance.
\subsection{Hypothesis}
\label{sec:more-hyp}
Based on the experiments and discussion in chapters~\ref{cha:nvse} and \ref{cha:mvlsa} we hypothesize that the NVSE method which is based on the Variational autoencoder framework will outperform the spectral method based MVLSA method for the CMR task. One of the reasons for this is that the MVLSA method requires access to a large number of disparate views which are not readily available for the CMR task.

Due to similar reason, we believe that the VAE based method will be better than MVLSA for the task of entity linking. In fact, we skip testing the MVLSA method in a head-to-head comparison with NVSE on the entity linking task and focus on comparing the VAE based method to a state-of-the-art entity embedding method based on max-margin learning.
\section{Coreferent Mention Retrieval}
\label{sec:cont-ment-retr}

The CMR dataset released by \cite{sankepally2018test} was constructed using the TAC-KBP2014 Entity Discovery and Linking dataset~\cite{ji2014overview} which is available from the LDC as \texttt{LDC2014E54, LDC2014E13}. I will refer to this collection as TAC14. This dataset contains newswire documents, annotated mentions of entities in those documents, and some entity links between those mentions and canonical URLs of entities in Wikipedia. \cite{sankepally2018test} used $84$ mentions as input queries from TAC14. A subset of the documents from TAC14 collections was chosen for retrieval as follows: First, the earliest and latest dates for the documents from which the query mentions were selected were determined. Then, those documents whose dates did not fall between these dates were filtered out. This reduced the size of the retrieval set from $1$ million to $117,132$ documents. Given a query mention -- out of the $84$ mentions -- the goal was to find all co-referent mentions for that query out of $117,132$ documents. Since the TAC14 collection has fairly sparse mention, annotations, therefore, \citeauthor{sankepally2018test} used the Amazon Mechanical Turk platform to obtain additional relevance judgments for a set of candidate mentions. A total of $4,172$ relevant mentions were collected in this way. %% See \citeauthor{sankepally2018test} for a detailed explanation of the data collection protocol on Mechanical Turk and the inter-annotator agreement rates.

\cite{sankepally2018test} measured the performance of a system --  at the level of individual sentences -- using the Mean Inferred Average Precision metric~\cite{yilmaz2006estimating}. The average of infAP over all queries is Mean infAP, and analogously the average of inferred AP is mean Inferred AP. Inferred Average Precision is a refined version of the Average Precision metric that accounts for the fact that some of the results that are returned by a system may actually be relevant but may have been skipped by judges during manual annotation. Inferred Average Precision metric also assumes that the top-K results returned by a system have perfect recall. I briefly explain how inferred Average Precision~(infAP) is computed for a query. Let us first recall how Average Precision (AP) is computed. Assume that a retrieval system returns a ranked list of documents, for each document in the ranked list I evaluate whether the document is relevant or not. This gives us a sequence of binary values $(e_i)_{i=1}^{N}$ where $e_i$ is the binary relevance of the $\ath{i}$ result. The average precision is computed as

\[
\sum_{i=1}^{N} \frac{e_i}{N} \frac{\sum_{j=1}^i e_j}{i}
\]

Note that ${\sum_{j=1}^i e_j}/{i}$ is the precision at the $\ath{i}$  position. \cite{yilmaz2006estimating} showed that the above metric could be considered as an expectation of the following random experiment.

\begin{algorithm}
  \begin{algorithmic}[1]
    \State \textbf{Input:} A list, $L$, of binary relevance values, And a map, $M$, of relevant documents to their rank in $L$ or if a relevant document is not present in $L$ then the default value is $L$.
    \State \label{alg:ap:2}Select a relevant document at random from the keys of $M$ and let its associated rank be $R$.
    \State \label{alg:ap:3}Select a rank $r$ uniformly at random from the set $\{1, \ldots, R\}$.
    \State \label{alg:ap:4}Output the binary relevance of the document at rank $r$.
  \end{algorithmic}
  \caption{The Average Precision Random Experiment.\label{alg:cmr-datas-eval}}
\end{algorithm}

Steps~\ref{alg:ap:3} and \ref{alg:ap:4} effectively compute the precision at a relevant document, and Step~\ref{alg:ap:2} has the effect of weighted averaging over all documents, weighted by the relevance of a document. The inferred Average Precision metric changes steps~\ref{alg:ap:3} and \ref{alg:ap:4} to compute precision at rank $R$  more robustly when all the relevant mentions are not available in the retrieved set.

\subsection{Experiments}
\label{sec:cmr-experiments}

I performed mention retrieval experiments on the CMR dataset in a query-by-example setting. I first learned representations of un-linked mention spans of named entities using methods such as MVLSA and Variational Autoencoders, and then I used these features to re-rank the top-100 results returned by the \textsc{Lucene}~\cite{mccandless2010lucene} Information Retrieval software. In the following sections, I explain the pre-processing steps performed before doing Lucene retrieval and the features used to learn mention representations, and finally how the mention representations were used for the final mention retrieval.

\subsubsection{Mention Featurization}
\label{sssec:cmr-mf}
A mention is a span of words inside a sentence that refers to an entity.\footnote{Such Spans are also called named-spans, but I will use the term \textit{mentions} throughout.} Such spans are already marked in the TAC14 collection~\cite{ji2014overview}. Given all the mentions in the corpus, my first step is to obtain the raw features for each mention-span in the corpus. I followed the best-performing pre-processing method from ~\cite{sankepally2018test} for the following steps.

I created two sets of features for each mention.\footnote{These feature sets are also called fields in the Information Retrieval literature.} The first field of features -- called the \textit{mention field} -- uses the mention words, and the second field -- called the \textit{document field} -- is built upon the background document text for representing candidate mentions. Each field contains binary and real-valued features. The binary features are constructed by counting the occurrence of the mention string, the mention type, trigrams of mention string, and an acronym of the mention. The acronym feature was created by concatenating the first alphabetic character of each mention word. All English stop words such as ``the'', ``an'', ``of'' were removed before constructing the features.\footnote{The \textsc{Lucene} StandardAnalyzer filters the lowercased and normalized output of a grammar-based tokenizer which implements the word-break rules from the Unicode Text Segmentation algorithm~\cite{davis2018unicode} using a list of English stop words. The normalizer removes the 's at the end of words and dots in acronyms.} The real-valued features were constructed from the mention words, words from the surrounding sentence and top-scoring words from surrounding document and words in the coreference chain of the mention.\footnote{The coreference chain was extracted using the Stanford CoreNLP coreference resolution system.} All real-valued features in both the fields were weighted using the BM25 weights\footnote{See \S~\ref{ssec:bm25} for details about BM25.} but the binary features were not weighted using BM25.

After obtaining the features, I index them using the \textsc{Lucene} library. At query time I first retrieve the top-100 mentions from Lucene along with their scores as computed by Lucene. For representing the query, I only used the mention words in the query and projected these query features using \textsc{Lucene}'s multi-field Query Parser. I.e., I duplicated the mention words across both the mention field and the document field. The \textit{document field} was given a weight of $w$ and the \textit{mention field} was given a weight of $1.0$.

For example, the query \textit{Keith Wiggans} is represented as weighted bag-of-words features spread across two fields with associated weights as shown below:\\
\verb#mention_keith:1.0, document_keith:0.1, mention_wiggans:1.0, document_wiggans:0.1#

\subsubsection{Learning Entity Embeddings}
\label{sssec:lee}
As mentioned earlier I experimented with MVLSA and the Variational Autoencoder method described in \charef{cha:nvse}. I conducted experiments on a CMR dataset by \cite{sankepally2018test} and evaluate the inferred Average Precision metric.
The results of the evaluation are shown in \tabref{tab:cmrres} and we can conclude the following from the results:
\subsection{Results and Discussion}\label{sec:cmr-results}
MVLSA m is the intermediate dimensionality, k is the final dimensionality.
\begin{table}[htbp]
  \centering
  \begin{tabular}{p{3cm} p{2cm} l p{1cm} p{1cm}}
  \multicolumn{2}{c}{\textbf{Method}} & \textbf{Hyper-parameters} & \textbf{InfAP} & \textbf{Ensemble}\\\hline
  \multirow{2}{3cm}{\cite{sankepally2018test}}&
   & $w=0.01$                         & 54.26 &\\
  && $w=0.01$                         & 54.53 &\\\hline
  \multirow{3}{3cm}{LSA of Concatenated Views}&
   & dim$=20$                         & 34.08 &50.11\\
  && dim$=300$                        & 40.08 &50.32\\
  && dim$=500$                        & 38.08 &49.39\\\hline
  \multirow{4}{3cm}{Multiview LSA}&
   & $m=500$,$k=500$                  & 37.08 &-\\
  && $m=500$,$k=1000$                 & 36.65 &-\\
  && $m=500$,$k=300$                  & 36.51 &-\\
  && $m=256$,$k=300$                  & 35.08 &48.95\\\hline
\multirow{14}{3cm}{Variational Autoencoder}
  &\multirow{3}{2cm}{Multilabel Decoder}
   & $lbv=0,encdim1=150,\beta=1$      & 36.35 &-\\
  && $lbv=1,encdim1=150,\beta=1$      & 39.55 &-\\
  && $lbv=1,encdim1=300,\beta=1$      & 40.99 &-\\\cline{2-4}
  &\multirow{11}{2cm}{Multinomial Decoder}
 & $\nepoch=5,\beta=1$,               & 50.25 &55.13\\
&& $\nepoch=5,\beta=1$, +$lbv=0$      & 46.29 &-\\
&& $\nepoch=5,\beta=1$, +deduplicate  & 50.25 &-\\
&& $\nepoch=0,\beta=1$                & 50.74 &54.05\\
&& $\nepoch=40,\beta=1$               & 50.29 &55.04\\
&& $\nepoch=5,\beta=1$, +smooth=$0.5$ & \textbf{50.91} &55.27\\
&& $\nepoch=2,\beta=0.0,smooth=0.5$   & 49.67 &\textbf{55.80}\\
&& $\nepoch=2,\beta=0.2,smooth=0.5$   & 49.31 &55.42\\
&& $\nepoch=2,\beta=0.4,smooth=0.5$   & 49.62 &55.53\\
&& $\nepoch=2,\beta=0.6,smooth=0.5$   & 48.83 &55.14\\
&& $\nepoch=2,\beta=0.8,smooth=0.5$   & 50.34 &54.69\\
%% 7L1TmnE00ED300        & 49.92 \\
%% 7L1TmnE00ED600        & 50.52 \\
%% 7L1TmnE00ED600VD600   & 50.15 \\
%% 7L1TmnMedrun          & 49.46 & 54.87\\
%% 7L1TmnE00ED300Mean        54.61
%% 7L1TmnE00ED600Mean        54.67
%% 7L1TmnE00ED600VD600Mean   54.63
%% Co7L1TmnMean              55.00
%% Co7L1TmnLrSlowerMean      52.60
%% Co7L1TmnXActLinMean       54.55
%% Co7L1TmnLVActLinMean      55.18
%% Co7L1TmnAmsMean           54.46
%% Co7L1TmnBeta1HighMean     55.27
%% Co7L1TmnLVActLinBeta1HighKLAN4           43.33
%% Co7L1TmnLVActLinBeta1HighKLAN4Mean       53.66
%% Co7L1TmnLVActLinBeta1HighED300KLAN4      41.86
%% Co7L1TmnLVActLinBeta1HighED300KLAN4Mean  52.69
%%               Epoch0 Epoch1 Epoch2
%% DA1     35.70  37.18  37.88  35.70
%% DA1Mean 51.70  52.80  50.54  51.70
%% Basically using Dawen's architecture did not work at all.
%% Co7L1TmnLLMNOM             29.20
%% Co7L1TmnLLMNOM (negative)  38.61
%% Co7L1TmnLLMNOM (pwmean)    38.63
%% DA1LLMNOM                  38.63
\end{tabular}
\caption[Contextual mention retrieval results]{Results of Different Unsupervised Representation Learning Algorithms on the Contextual Mention Retrieval Task.}
\label{tab:cmrres}
\end{table}
%%TODO Requires much more than this.  If this is the concluding chapter where the methods are being put head to head, then you need to have some real error analysis: what mistakes does each approach make?  Are they the same kinds of mistakes or different?  What might be improved in the future?
\subsubsection{Performance Comparison between MVLSA and VAE} The highest performance of the Variational Autoencoder is $50.91$ infAP points individually which is far superior to the best performance of achieved by MVLSA of $37.08$. Clearly the non-linearity of the encoder and optimizing the ELBO objective is helping the VAE extract more useful features from the raw bag-of-words representation.

\subsubsection{The influence of Decoder Type on VAE}
The decoder type choice for the VAE exerts significant influence on the performance of the VAE. The best performance with the multilabel decoder for the VAE is $40.99$ which increases by $10$ absolute points to $50.91$ in the case of the multinomial decoder. This shows us that the multinomial decoder is the better choice for encoding high-dimensional sparse bag-of-words representations of text documents.

\subsubsection{The effectiveness of VAE training}
An interesting observation is the relatively small improvement in the individual performance of the VAE from $50.74\%$ infAP to $50.91\%$ infAP after the multinomial decoder is trained for $5$ epochs. However, the improvement in the ensemble-performance is larger from $54.05\%$ to $55.80\%$ which is almost a $2\%$ absolute improvement in inferred average precision.

\section{Named Entity Disambiguation}
\label{sec:named-entity-disamb}
The goal of Named Entity Disambiguation (NED) is to link a detected name-mention in a text document to an entity in a knowledge graph (KG). See \citet{hoffart11} for an overview of the NED task and survey of approaches circa 2011. See \cite{bollacker2008freebase,dong2014knowledge} for an introduction to knowledge graphs.

More recently, unsupervised representation-learning approaches -- such as Word2Vec~\cite{mikolov2013distributed}, Paragraph Vectors~\cite{le2014distributed}, and BERT~\cite{devlin18} -- have become popular for language processing tasks. Simultaneously systems that learn low-dimensional and dense \textit{Entity Embeddings} were proposed for the NED task by \citet{he13,yamada16,fang16,zwic16,yamada17} and \citet{ganea17}. In light of the greatly increased research into entity-embeddings, and because of the sustained interest in solving the NED task, the investigation -- presented in this paper -- of the effects of different entity embedding methods on NED accuracy will be useful.

One of the current best NED systems is the document-level joint-inference neural model proposed by \citealp{ganea17}. This model roughly operates along the following three steps:

\textbf{In Step One}: A max-margin objective is optimized -- independently for each entity -- to learn an entity embedding from its description page and the tokens in a fixed size window surrounding its mentions. The optimization procedure itself is inspired by Word2Vec and relies on negative sampling. The entity embeddings are learned separately, and then they are frozen.

\textbf{In Step Two}: For each mention, a list of top-$k$ candidate entities are re-scored using a \textit{local} neural network which receives the mention its context and the entity embeddings as input.

\textbf{Finally, in Step Three}: a document level joint-inference procedure\footnote{Specifically loopy-belief propagation over a full-connected factor graph with pairwise potentials is used for joint inference.} is used to determine which entity is referred by which mention.

Although \citealp{ganea17} quantified the downstream impact on the NED accuracy of using a \textit{global} joint-inference model in comparison to a \textit{local} NED model, the downstream impact of the entity embedding method in step one on the final NED accuracy remains unclear. Therefore, in this paper, we quantify the effect of different pre-training objectives and different types of input contexts on NED accuracy.

\subsection{Entity Embedding: Methods and Data}
\label{sec:mee}
Let $\mathcal{E}$ denote the set of entities. Abstractly an entity embedding is a map, say $\e: \mathcal{E} \to \R^d$. $d$ is typically chosen to between $100$ to $500$ based on cross-validation with the most common choice being $300$. Qualitatively, an entity's embedding should be discriminative amongst homonym entities such as a \textit{(river) bank} and a \textit{(financial) bank}, and it should be similar to the combined representation of content words that co-occur with its mentions.

\subsubsection{Data Sources}
\label{sec:data-sources}
Mainly two sources of data have been used in previous work for learning $\e$:
\textbf{The first data-source} is the text in the \textit{canonical} page that describes an entity. For example, the Wikipedia page for \textsc{Amarchy} defines the concept of anarchy. It mentions that \textsc{Anarchy} is a type of \textit{political philosophy} which \textit{rejects hierarchy}.
Clearly, embedding the content words that appear on the canonical page of an entity into a low-dimensional dense feature vector can help us describe an entity succinctly. We denote this type of data in general as $\mathcal{U}_1$.

\textbf{The second data-source} -- which may not exist in some cases -- consists of tokens surrounding the mentions of an entity. For example, the Wikipedia page for \textsc{Alexander Grothendieck}, the famous algebraic geometer, mentions that \textit{Grothendieck was born in Berlin to [anarchist](Anarchism) parents $\ldots$}\footnote{[description](link) is the markdown hyperlink syntax.}. By optimizing the entity embedding for \textsc{Anarchy} to be similar to the representations of the tokens surrounding the word \textit{anarchist} in this sentence, we can disambiguate the political philosophy of Grothendieck's parents. We denote this data-source as $\mathcal{U}_2$. Take note that this data-source is not truly unlabeled, and it may not exist in many practical applications. Especially in \textit{cold-start} situations~\cite{tac2015cold,rastogi2017vertex} large manually hyperlinked entity corpora do not exist. Therefore entity-embedding models that can operate with or without entity mentions are desirable. If $\mathcal{U}_2$ is available then the total data is $\mathcal{U} = \mathcal{U}_1 \cup \mathcal{U}_2$ otherwise $\mathcal{U} = \mathcal{U}_1$.

In addition, it is useful to quantify the benefit of entity mentions on well-studied NED task such as \textit{wikification}~\cite{mihal07,ratinov11} where large quantities of hyperlinked mentions are readily available, so that we can quantify how much may we gain by annotating this information.

\subsubsection{Methods}
\label{sec:methods}

Let $\mathcal{W}$ denote the word vocabulary of our entity embedding system. Contemporary approaches for learning entity embeddings~\cite{yamada17,ganea17} start with a pre-trained word embedding map $\w : \mathcal{W} \rightarrow \R^d$.\footnote{Typically the dimensionality of the word embedding map $\w$ and the entity embedding map $\e$ is kept the same to reduce the number of free hyper-parameters and to map words and entities to the same vector space. We follow the same convention in this work.} For example, if the word embeddings of \textit{philosophy} and \textit{theory} are similar and the embeddings for \textit{rejects} and \textit{shuns} are similar then an NED system could correctly disambiguate a \textit{political theory that shuns hierarchy}. It is folk-knowledge that using a pre-trained embedding such as Word2Vec pre-trained vectors~\cite{mikolov2013distributed} improves performance, but the downstream impact on NED -- to the best of our knowledge -- has not been reported in previous work.

\paragraph{Max-Margin Entity Embedding}
\label{sssec:mm}
The word embedding map $\w$ and unlabeled data $\mathcal{U}_{\{1,2\}}$ can be used to learn $\e$ in many ways. However, Word2Vec inspired objectives have dominated other approaches in the context of NED research \cite{yamada16,ganea17,yamada18}. In this paper, we focus on a state-of-the-art model for entity disambiguation proposed by \citealp{ganea17} which minimizes a max-margin loss for learning entity embeddings.

\citealp{ganea17} motivate their objective as follows: Let $\mathcal{U}_e$ denote the portion of training data containing all the words that co-occur with $e$. Let $p(w | e)$ denote a conditional-multinomial distribution of words that occur in $\mathcal{U}_e$. $p(w|e)$ is estimated from empirical counts ${\#(w, e)} / {\sum_{w\in \mathcal{U}_e} \#(w, e)}$.  Next, let $q(w)$ be an unconditional probability distribution. \citealp{ganea17} define $q(w) = {p}(w)^{\alpha}$ for some $\alpha \in (0,1)$ where ${p}(w) \propto \#(w)$ in $\mathcal{U}$. Let $w^+, w^-$ be two random variables sampled from $p(w|e)$ and $q(w)$ respectively. Careful readers will have noticed that the setup so far is the same as \cite{mikolov2013distributed}. Indeed, this is why we said that these models are ``Word2Vec inspired''.

The novelty of \citealp{ganea17} is that they optimize the max-margin objective in \eqref{eq:ganea} instead of the logistic-type loss defined by \citealp{mikolov2013distributed} to infer the optimal embedding for entity $e$ with a margin hyper-parameter $\gamma>0$:
\begin{align}
& \e(e) =  \argmin_{\z: \|\z\|=1, \z \in \R^d} \mathbb{E}_{w^+|e} \; \mathbb{E}_{w^-} \left[ h\left( \z;w^+,w^- \right) \right]  \nonumber\\
& h(\z;w,v) = \max(0, \gamma - \z^T (\w_{w}-  \w_{v})) \label{eq:ganea}
\end{align}
Since the word-embeddings are kept fixed therefore the above objective is convex.

\paragraph{Variational Entity Embedding}
\label{sssec:vee}
There are many ways of constructing the entity embedding map $\e$ from $\w$ and $\mathcal{U}$. \citealp{ganea17} motivated their learning algorithm using generative assumptions but optimized a \textit{contrastive} max-margin objective in the actual learning process. At this step, a question naturally arises that how important is the max-margin algorithm for learning and what other methods could be motivated from the same generative assumption.

In order to answer these questions, we propose to use the Variational-Autoencoder Framework (VAE)~\cite{kingma2014} for learning entity embeddings and comparing their downstream performance to the max-margin entity embeddings by \citealp{ganea17}. Under the standard VAE framework, the generative model is:
\begin{align*}
  \z &\sim \pi = \mathcal{N}(\mathbf{0}, \mathbf{I}), \text{ and } \x_e \sim p(w|\z) = \NN^g_{\btheta}(\z)
\end{align*}
Here $\pi$ denotes the standard gaussian prior distribution on the latent variable $\z$, the output of $\NN^g$ are the mean-parameters of a conditional multinomial distribution. $\x_e$ denotes  the bag-of-words representation of  $\mathcal{U}_e$.\footnote{Recall that in the previous section, we defined $\mathcal{U}_e$ as the training words that co-occur with $e$.} The parameters $\btheta$ of $\NN^g$ are learnt by defining two inference-networks $\NN^{i,m}_\bphi, \NN^{i,v}_\bphi$ that map $\x_e$ to a gaussian posterior over the latent variable $\z$. $\NN^{i,m}_\bphi, \NN^{i,v}_\bphi$ compute the mean and variance of the posterior respectively. The parameters $\btheta, \bphi$ are learnt jointly by minimizing the Evidence Lower Bound Objective (ELBO):
\begin{align}
  &\argmin_{\btheta,\bphi}\sum_{e \in \mathcal{U}}E_{\NN^i_\bphi(\z|\x_e)}[\log \NN^g_\btheta(\x_e | \z)] \nonumber\\
  &\qquad - \beta \KL(\NN^i_\bphi(\z|\x_e) || \pi).\label{eq:elbo}
\end{align}
Here $\beta \ge 0$ is a hyper-parameter that can be tuned via cross-validation. A larger value of $\beta$ leads to disentangled representations~\cite{higgins17} but a lower value of $\beta$ can improve the model fit by decreasing the KL penalty. Another interpretation of $\beta = 0$ is that instead of using $\pi$ as the prior of $\z$ we are using a dynamic prior equal to $\NN^i_\bphi(\z|\x_e)$.
After the training we define $\e(e)$ as the posterior mean of $\z$ given $\x_e$, \ie
\begin{align}
  \e(e) = \NN^{i,m}_\bphi(\x_e).\label{eq:vae}
\end{align}
%% Although previous approaches have been dominated by
%% Therefore a question naturally arises
%% The max-margin objective really important? Or are there other methods that could be equally useful.
%% We constructed entity embeddings using a variational auto-encoder as follows:

\paragraph{Null Objective Entity Embedding}
\label{sssec:rp}
Recall that we defined $\x_e$ as the bag-of-words vector representation of $\mathcal{U}_e$. The VAE method learns a neural network to map $\x_e$ to $\e(e)$ as shown in \eqref{eq:vae}. We want to know how beneficial is the VAE objective itself in training a discriminative neural network. Is it the case that $\x_e$ are themselves linearly separable and a random low dimensional embedding will work just as well or better than the VAE?
In order to answer these questions, we conduct an experiment where we use \eqref{eq:vae} with a randomly initialized inference network as our entity embedding.

We call this the \textit{Null Objective} Entity Embedding because we do not optimize objective \eqref{eq:elbo} for training $\NN^{i,m}_\bphi$.
\subsection{Related Works}
\label{sec:related-works}
Recently \cite{kar18} also released a local model; however, the code they have released was incomplete and
Yamada et al. also released entity embeddings
In order to limit the scope of this study, I restricted myself to the global model proposed by and we use the variational autoencoder to closely mimic the generative assumptions in the original paper but without the max-margin training. Our model is most closely related to the NVDM model~\cite{miao2016neural}.

There are potentially many ways of either bag-of-words methods such as Paragraph Vectors~\cite{le2014distributed}, Simple Embedding~\cite{arora17} or even pre-trained sentence encoders such as ELMO~\cite{peters18} and BERT~\cite{devlin18} to construct $\e$. Most recently \cite{yamada18} proposed a method to train entity embeddings; however they did not show a downstream evaluation of NED accuracy. Therefore, it is not known at this time how effective these embeddings will be for NED. We leave this evaluation for future work.

\subsection{Experiments and Results}
\label{sec:experiments}
To quantify the effect of the embedding objective on NED accuracy we trained
the best performing \textit{global} NED model from \citealp{ganea17} with pre-trained entity embeddings learnt using different objectives. For each objective, we varied whether the entity embeddings had access to the mentioned context $\mathcal{U}_2$. Following prior work, we evaluated the micro F1 score of the trained models on four of the most popular NED datasets. The four datasets are AIDA-CoNLL~\cite{hoffart11}, MSNBC~\cite{cuc07}, AQUAINT (AQUA.)~\cite{milne08} and ACE2004 (ACE04)~\cite{ratinov11}. We used the preprocessed versions of these datasets released by \newcite{ganea17,guo14}. The F1 scores and pertinent statistics for the datasets are shown in \tabref{tab:mf}.

The Max-Margin Embeddings were trained on a February 2014 dump of Wikipedia as follows: Each entity vector was randomly initialized from a mean-$0$, variance-$1$, $300$-dimensional normal distribution. Values larger than $10$ were clipped. First, the embeddings were trained to convergence on the content words in the canonical description page for an entity. In each iteration, $20$ samples of $w^{+}$ and $5$ samples of $w^-$ were used to minimize \eqref{eq:ganea}. The optimizer was Adagrad with a learning rate of $0.3$. Hyperparameters $\alpha=0.6$ and $\gamma=0.1$ were the recommended values from \citealp{ganea17}. If $\mathcal{U}_2$ was included, then we used the tokens from a window of size $20$ as positive examples as well.

%% VAE_ENCDIM1=300 VAE_DIM=300 VAEEPOCHS=100 LOAD_WORD2VEC=1 MENTION_CONTEXT_WINDOW=10
%% LOAD_WORD2VEC=0 MENTION_CONTEXT_WINDOW=-1
%% pp_max_features=50000 PREPROC=bm25 EMBED_METHOD=vae
%% VAEEPOCHS=10 VAE_LBV=1 VAE_ENCDIM1=150 VAE_DECODERTYPE=multinomial VAE_DIM=300
%% VAE_BATCH_SIZE=1023
%% VAEOPTIMIZER=adamcustom adam_lr=1e-3 adam_epsilon=1e-8 adam_amsgrad=0 adam_beta_1=0.9
%% x_activation=relu mu_activation=linear logvar_activation=sigmoid KLANNEAL_ULIM=0.0
%% Namespace(epochs=$VAEEPOCHS,lowerbound_encoder_variance=$VAE_LBV,encdim1=$VAE_ENCDIM1,decodertype='$VAE_DECODERTYPE',vaedim=$VAE_DIM,shuffle=1,batch_size=$VAE_BATCH_SIZE,optimizer='$VAEOPTIMIZER',trainmode='${TRAINMODE-prod}',sparseinput=${sparseinput-True},floatx='float32',validation_split=0.01,adam_lr=$adam_lr,adam_epsilon=$adam_epsilon,adam_amsgrad=$adam_amsgrad,adam_beta_1=$adam_beta_1,x_activation='$x_activation',mu_activation='$mu_activation',logvar_activation='$logvar_activation',klanneal_ulim=$KLANNEAL_ULIM,archetype='${archetype-myarch}',l2_regularization=${l2_regularization-0},load_word2vec=$LOAD_WORD2VEC,mcwindow=$MENTION_CONTEXT_WINDOW)
The $300$-dimensional VAE embeddings were trained using a fixed vocabulary of $50,000$ words.

For learning the NED model, we used ADAM with a learning rate of 1e-4 until the validation accuracy exceeded a threshold. Afterward, we reduced the learning rate to 1e-5. The validation threshold was selected for each configuration by first training the system with a threshold of $100\%$ for three trials and then using the median of the highest validation F1 scores. The rest of the training details of the global NED model were identical to \citealp{ganea17}.
\begin{table}[htbp]
  \centering
  \begin{tabular}{c| c | p{1.05cm} | p{1.2cm} | p{1.05cm} | p{1cm}}\toprule
    {Objective}& $\mathcal{U}_2$ & {AIDA (4485) (98.2\%)} & {MSNBC (656) (98.5\%)} & {AQUA. (727) (94.2\%)} & {ACE04 (257) (90.6\%)} \\\hline
    \multirow{2}{1cm}{Max Margin} & \xmark & 89.5  & 92.7          & 84.9         & 88.1 \\
                                  & \cmark & 92.2  & 93.7          & 88.5         & 88.5 \\\hline
    \multirow{2}{*}{VAE}          & \xmark & 83.9  & 93.3          & 84.6         & 87.7 \\
                                  & \cmark & 85.4  & 94.7          & 86.9         & 86.5 \\\hline
    \multirow{2}{*}{Null}         & \xmark & 84.1  & 92.6          & 81.3         & 88.1 \\
                                  & \cmark & 83.0  & 93.2          & 86.4         & 88.1 \\\bottomrule
  \end{tabular}
  \caption{Micro F1 results for the same NED model but with entity embeddings pre-trained with different objectives. \cmark and \xmark\ in the second column indicates whether the learning algorithm had access to the mention context or not, respectively. The top two numbers in the dataset columns indicate the test-set size and the recall of the top-$30$ candidate entities.}
  \label{tab:mf}
\end{table}

\subsubsection{Entity Relatedness:} In addition to downstream NED evaluation we also evaluated our embeddings on an intrinsic task of entity relatedness prediction. The entity relatedness test set of~\cite{cecca13} measures how well the geometry of the entity embeddings captures manually annotated similarity relations between entities. It contains 3319 and 3673 \textit{entity-relatedness queries} for the test and validation sets. Each query consists of one prompt entity and up to one hundred candidate entities. Each candidate has a gold label indicating whether it is related to the prompt entity or not. The cosine similarity of the entity embeddings is used to rank the candidate entities, and the goal is to rank related entities before the unrelated entities. Same as previous work we report Normalized Discounted Cumulative Gain (NDCG) at three different ranks and the Mean Average Precision (MAP) score. \tabref{tab:er} shows the performance of different systems on this test set.
\begin{table}[htbp]
  \centering
  \begin{tabular}{c| c | c | c | c | c}\toprule
    {\footnotesize Objective}                        & $\mathcal{U}_2$ & \NDCG1 & \NDCG5 & {\NDCG10} & \MAP  \\\hline
    \multirow{2}{1cm}{Max Margin} & \xmark & 0.616  & 0.590  & 0.616     & 0.549 \\
                                  & \cmark & 0.646  & 0.611  & 0.639     & 0.576 \\\hline
    \multirow{2}{*}{VAE}          & \xmark & 0.592  & 0.559  & 0.578     & 0.514 \\
                                  & \cmark & 0.615  & 0.571  & 0.596     & 0.542 \\\hline
    \multirow{2}{*}{Null}         & \xmark & 0.577  & 0.537  & 0.563     & 0.503 \\
                                  & \cmark & 0.615  & 0.571  & 0.596     & 0.542 \\\bottomrule
  \end{tabular}
  \caption{Entity Relatedness Evaluation Results.}
  \label{tab:er}
\end{table}

%% \subsubsection{Error Analysis}
%% TODO: There is no error analysis, from particular tasks.

\section{Conclusion}
\label{sec:concl}
We started with the hypothesis that the entity embeddings based on the variational autoencoder will outperform the MVLSA embeddings and our experiments on the CMR task validated this hypothesis. Moreover, within the different types of decoder types, we found that the multinomial decoder had a significantly higher performance than the multilabel decoder.

We then focused attention on the multinomial VAE generative model and compared its performance to a state-of-the-art contrastive method for learning entity embeddings for the task of named entity disambiguation. We also evaluated the contribution to the NED accuracy by the mention context $\mathcal{U}_2$? The results  in \tabref{tab:mf} show that although the performance improvement varies with each system and on each dataset, for the best Max-Margin entity embeddings the improvement in performance is substantial.

Based on the experiments we can conclude that although MVLSA and GCCA, in general, are useful methods for learning entity embeddings their utility is limited in situations where multiple views of data are not readily available. On the other hand, variational methods for learning embeddings are promising in single view situations and can fare well in comparison to other more discriminative methods for learning entity embeddings for NLP tasks.
%% How important is max-margin training of entity embeddings and can a conceptually more straight-forward approach of using VAEs outperform it

%% generalizing and what are some final conclusions
%% that you can make. What are the fundamental lessons learnt, when
%% would GCCA be applicable ? when will NVSE be applicable, what kind
%% of mistakes will they make, where they will be applicable. Draw
%% connections to previous chapters,
%% we are seeing something, try to draw a connections, the broad
%% strokes are the contribution of the thesis.

%auto-ignore
\chapter{Concluding Remarks}
\label{cha:conc}
In this thesis, I developed novel algorithms for learning representations at the level of words and entities mentioned in linguistic corpora. \chapref{cha:mvlsa} focused on multi-view learning of word embeddings using the novel spectral method called MVLSA, and \chapref{cha:nvse} presented the NVSE model -- which builds upone the Variational Auto-Encoder framework -- for learning embeddings of entities. Through automatic evaluations for MVLSA and human evaluations for NVSE I showed that these methods can outperform existing state of the art methods in their respective domains.

Then in \chapref{cha:kbe} I explored ways of geometrically regularizing the learning of entity embeddings in a knowledge base to force the learnt embeddings to comply with logical constraints. I showed that on toy tasks at least it is possible to perform interesting logical inference using the proposed regularization methods. \chapref{cha:more} applied the proposed MVLSA, NVSE algorithms to more practical, downstream tasks of Contextual Mention Retrieval (CMR) and Named Entity dismabiguation (NED). We found that an ensemble of the features learnt through the variational-autoencoder approach with pre-existing bag-of-words features can improve the performance of a state-of-the-art CMR system. However on the task of Named Entity Disambiguation we did not find any benefit of our representations over the state of the art entity embeddings.

Future directions for some of the work in this thesis, especially regarding the NVSE algorithm, involves the use of pretrained contextual sequence encoders such as ELMO~\cite{matthew2018ctx} and BERT~\cite{devlin2018bert}. The methods for extracting features for entities proposed in this thesis should be compatible with these sequence encoding methods. Another potential future direction will be to use the constraint based methods developed in this thesis for embedding entities in knowledge graphs and applying them to more sophisticated models of entities such as the Box-Lattice Measures for probabilistic embeddings~\cite{luke2018prob,li2018smoothing} and \cite{subramanian2018new}. In order to do so, it will be desirable to improve the scalability of the alternative projection stochastic gradient algorithm proposed in \chapref{cha:kbe}.
%TODO Future directions: Scalability, Transfer Learning, Non linear extensions, robustness, prior knowledge. TODO: There may be new things that we might be able to do now, we can
%% enable entity centric search, what are the new things that I am enabling.
%% Now discussion of future work: what do you think are the next big things people should consider on this line?

% https://www.library.jhu.edu/library-services/electronic-theses-dissertations/formatting-guidelines-checklist/
% For the remainder of the ETD, including text, illustrations, appendices, and bibliography, use Arabic numerals (1, 2, 3, etc.). The numbering begins with one (1) and runs consecutively to the end of the manuscript. Do not use suffixes to the Arabic numerals, such as 12a.
\appendix
%auto-ignore
\chapter[Appendix]{}%Appendix
\label{cha:appendix}

\section{Bayesian Sets}
\label{sec:bs}
The Bayesian Sets algorithm ranks the elements in $\XmD$ according to the ratio of  two probabilities:
\[
\textit{score}(x) = \frac{p(x|\cD)}{p(x)} = \frac{E_{p(z | \cD)}[p(x|z)]}{E_{\pi(z)}[p(x|z)]}
\]
Instead of assuming the commonly used Beta-Binomial distribution I assume that $p(x | z)$  is a product of independent Poisson distributions with Gamma conjugate priors. I.e.
$p(x | z) = \prod_k \frac{z_k^{x_k}}{x_k\!}$. The conjugate prior on $z$ is a product of Gamma distributions, $$p(z  | \alpha, \beta) = \prod_k \frac{{\beta_k}^{\alpha_k}}{\Gamma(\alpha_k)} {z_k}^{\alpha_k - 1} \exp(-\beta_k z_k)$$. Let $f(x_k, \alpha_k, \beta_k) =$
$${\left(\frac{{x_k+\alpha_k-1}}{x_k} \right)} (1-\frac{1}{1+\beta_k})^{\alpha_k} (\frac{1}{1+\beta_k})^{x_k}.$$
The Bayesian Sest score under these conditions is
\[
\textit{score}(x) = \prod_k {\frac{f(x_k, \alphatil_k, \betatil_k)}{f(x_k, \alpha_k, \beta_k)}}
\]
Where $\alphatil_k = \alpha_k+\sum_{x \in \cD}x_k$ and $\betatil_k = \beta_k + \rD$. Note that if $\alphatil_k = \alpha_k$ then ${\frac{f(x_k, \alphatil_k, \betatil_k)}{f(x_k, \alpha_k, \beta_k)}} = (\frac{1+\beta_k}{1+\beta_k+D})^{x_k}$ which means that features that occur in $x$ that did not occur in $\cD$ are penalized based on the number of times the feature appeared. Therefore, the Gamma-Poisson distribution is a good approximation only when quantitative differences in the number of times a feature appears are important.

Finally I assume that the components of $x$ were sampled from conditionally independent gaussian distributions with unknown mean and precisions. I.e. $p(x | \mu, \tau) = $
\[\prod_k \sqrt{\frac{\tau}{2\pi}}\exp(-(x_k-\mu_k)^2 \tau_k)\] and $p(\mu, \tau | \rho, \lambda, \alpha, \beta) =$
\[ \prod_k \frac{{\beta_k}^{\alpha_k} \sqrt{\lambda_k}}{\Gamma(\alpha_k) \sqrt{2\pi}}{\tau_k}^{\alpha_k - \frac{1}{2}} \exp(-\beta_k \tau_k) \exp(- \frac{\lambda_k \tau_k (\mu_k - \rho_k)^2}{2}).\]
In the following formulaes I omit the susbscript $k$ for convenience.
\begin{align*}
  \xbar    &= \frac{1}{\rD}\sum_{x \in \cD} x\\
  \rhot    &= \frac{\lambda \rho + \rD \xbar}{\lambda + \rD} \\
  \lambdat &= {\lambda + \rD}\\
  \alphatil  &= \alpha + \rD/2\\
  \betatil   &= \beta + \frac{1}{2}\sum_{x \in \cD} (x - \xbar)^2 + \frac{\rD \lambda}{\rD  + \lambda}\frac{(\xbar - \rhot)^2}{2}
\end{align*}

The Bayesian Sets score is the ratio of two $t$ distribution values:
\[
\textit{score}(x) = \prod_k \frac{t_{2\alphatil_k}(x_k \mid \rhot_k, \frac{\betatil_k (\lambdat_k + 1)}{\alphatil_k \lambdat_k})}{t_{2\alpha_k}(x_k \mid \rho_k, \frac{\beta_k (\lambda_k + 1)}{\alpha_k \lambda_k})}
\]

Now the value of $t_\nu(x |a, b)$ where $a$ is the location parameter and $b$ is the scale parameter is:
\[
t_\nu(x |a, b) = \frac{\Gamma(\frac{\nu+1}{2})}{\sqrt{b\nu\pi}\Gamma(\frac{\nu}{2})} \left( 1 + \frac{(x-a)^2}{b\nu}\right)^{-\frac{\nu+1}{2}}
\]

In order to use this distribution with count data, it is %very
important to use some variance stabilizing transform, and then perform mean and variance normalization to preprocess all the count features. In this way I can set the priors $\rhot_k$ to be $0$ and $\lambda_k$ can be set uniformly to some small number such as $2$ and $alpha_k, \beta_k$ can be chosen to be $2, 1$ respectively.

\subsection{Binarizing feature counts}
\label{app:bs-binarize}
BS binarizes the feature vector $f_x$ as $f'_x$ via thresholding: %, following the protocol: %ir protocol:
 \begin{align*}
   f'_x&[j] = \bI[f_x[j] > \mu[j] + \lambda \sigma[j]]\\
  \mu[j] &= \frac{{\sum_{x \in \cX} f_x[j]}}{\rX},
  \sigma^2[j] {=} {\frac{{ \sum_{x \in \cX} (f_x[j] - \mu[j])^2}}{\rX}},
 \end{align*}
where $\lambda \in \mathbb{R}$ is a hyperparameter.
I tried three values of $\lambda$ -- $\{0, 0.5, 1\}$ -- and set it to $0.5$ based on preliminary experiments. BS's scoring function becomes
\begin{subequations}
\begin{align}
 \underset{BS}\score(\cQ{,}x) &= {\sum_{j=1}^\rF} {\left(\log \frac{\talpha[j] \beta[j]}{\alpha[j]\tbeta[j]}\right)} f'_x[j] \\ \label{eq:bs_score}
  \talpha[j] &= \alpha[j] +  \sum_{x \in \cQ} f'_x[j]\\
  \tbeta[j] &= \beta[j] + \rQ - \sum_{x \in \cQ} f'_x[j].
\end{align}
\end{subequations}

\section{Ranking methods}
\label{app:rankings}
A standard function for computing the distance between distributions is the KL-divergence.
Another possibility to compute the distance between distributions is to compute the symmetric version of the KL-divergence.
Another standard method for computing the similarity between two probability distributions is to compute the probability product kernel (PPK) between two distributions~\mycite{jebara2004probability}; i.e.
\[\ip{ q_\phi(z|\cQ)}{q_\phi(z|x)} = \int_z q_\phi(z|\cQ)q_\phi(z|x) dz\]
In the special case that $q_\phi(z|\cQ)$ and $q_\phi(z|x)$ have the special deep-gaussian form then the KL divergence as well as the inner product can be computed in closed form.
%\footnote{
  KL Divergence between two distributions normal distributions $p_1, p_2$ with parameters $(\mu_1, \Sigma_1)$ and $(\mu_2, \Sigma_2)$ is:
\[
 KL(p_1 || p_2) = \frac{1}{2}\left(\text{tr}(\Sigma_2^{-1}\Sigma_1) + (\mu_1 - \mu_2)^\top \Sigma_2^{-1}(\mu_1 - \mu_2) - d + \log \frac{\det(\Sigma_2)}{\det(\Sigma_1)}\right).
\] and PPK is
\[
 \exp(\frac{-(\mu_1 - \mu_2)^T (\Sigma_1 + \Sigma_2)^{-1} (\mu_1 - \mu_2)}{2} - \log\det((\Sigma_1 + \Sigma_2)))
\]
In the further special case that $\mu_2 = \mathbf{0}, \Sigma_2 = \mathbf{I}$ then the KL divergence simplifies to:
\[
 KL(p_1 || p_2) = \frac{1}{2}\left(\text{tr}(\Sigma_1) + \mu_1^T \mu_1  - d - \log {\det(\Sigma_1)}\right).
\]

However, I propose here a simple way to compute the distance between two normal distributions. If $\mu_1, \Sigma_1$ and $\mu_2, \Sigma_2$ are the mean and variance of two normal
distributions, $p_1, p_2$ then I use the following distance
\[d(p_1, p_2)  = ||{\mu_1 \Sigma_1^{-1} - \mu_2 \Sigma_2^{-1}}||^2 = ||\xi_1 - \xi_2||^2
\]

This metric can be implemented as a single matrix multiplication while KL divergence and PPK cannot.
Intuitively this distance gives higher weightage to those dimensions where the variance of the either the distributions is lower. In preliminary experiments I  found this distance to be superior to KL divergence and PPL and I have used this distance function in all of my experiments. I believe that the regularization from the gaussian prior that encourages the posterior distributions to be close to the origin make shift invariance unnecessary.

\printbibliography[title={References}]

\end{document}